\documentclass[english]{article}
\usepackage[T1]{fontenc}
\usepackage[latin9]{inputenc}
\usepackage{geometry}
\usepackage{babel}
\usepackage{array}
\usepackage{longtable}
\usepackage{float}
\usepackage{url}
\usepackage{multirow}
\usepackage{amsmath}
\usepackage{amsthm}
\usepackage{amssymb}
\usepackage{graphicx}
\usepackage[unicode=true,pdfusetitle,
 bookmarks=true,bookmarksnumbered=false,bookmarksopen=false,
 breaklinks=false,pdfborder={0 0 1},backref=false,colorlinks=false]
 {hyperref}
 \newcommand{\citep}{\cite}
 \newcommand{\citet}{\cite}

\makeatletter

\providecommand{\tabularnewline}{\\}
\floatstyle{ruled}
\newfloat{algorithm}{tbp}{loa}
\providecommand{\algorithmname}{Algorithm}
\floatname{algorithm}{\protect\algorithmname}

\theoremstyle{plain}
\newtheorem{thm}{\protect\theoremname}[section]
\theoremstyle{plain}
\newtheorem{lem}[thm]{\protect\lemmaname}
\ifx\proof\undefined
\newenvironment{proof}[1][\protect\proofname]{\par
	\normalfont\topsep6\p@\@plus6\p@\relax
	\trivlist
	\itemindent\parindent
	\item[\hskip\labelsep\scshape #1]\ignorespaces
}{%
	\endtrivlist\@endpefalse
}
\providecommand{\proofname}{Proof}
\fi

\makeatother

\providecommand{\lemmaname}{Lemma}
\providecommand{\theoremname}{Theorem}

\global\long\def\A{\mathbf{A}}%
\global\long\def\D{\mathbf{D}}%
\global\long\def\avx{\overline{x}}%
\global\long\def\tr{\mathrm{Tr}}%
\global\long\def\R{\mathbb{R}}%
\global\long\def\E{\mathbb{E}}%
\global\long\def\V{\mathrm{Var}}%
\global\long\def\dom{\mathcal{X}}%
\global\long\def\sm{\mathbf{B}}%
\global\long\def\diag{\text{\text{diag}}}%
\global\long\def\err{\text{\ensuremath{\mathrm{Err}}}}%

\begin{document}
\title{Adaptive and Universal Algorithms for Variational Inequalities with
Optimal Convergence}
\author{Alina Ene\thanks{Department of Computer Science, Boston University. ${\tt aene@bu.edu}$}
\and Huy L. Nguyen\thanks{Khoury College of Computer and Information Science, Northeastern University.
${\tt hu.nguyen@northeastern.edu}$}}
\date{(version 3)\thanks{The first version of this paper appeared on arXiv on October 15, 2020. The current version is a minor revision of version 2.}}
\maketitle
\begin{abstract}
We develop new adaptive algorithms for variational inequalities with
monotone operators, which capture many problems of interest, notably
convex optimization and convex-concave saddle point problems. Our
algorithms automatically adapt to unknown problem parameters such
as the smoothness and the norm of the operator, and the variance of
the stochastic evaluation oracle. We show that our algorithms are
universal and simultaneously achieve the optimal convergence rates
in the non-smooth, smooth, and stochastic settings. The convergence
guarantees of our algorithms improve over existing adaptive methods
by a $\Omega(\sqrt{\ln T})$ factor, matching the optimal non-adaptive
algorithms. Additionally, prior works require that the optimization
domain is bounded. In this work, we remove this restriction and give
algorithms for unbounded domains that are adaptive and universal.
Our general proof techniques can be used for many variants of the
algorithm using one or two operator evaluations per iteration. The
classical methods based on the ExtraGradient/MirrorProx algorithm
require two operator evaluations per iteration, which is the dominant
factor in the running time in many settings.
\end{abstract}
\allowdisplaybreaks

\section{Introduction}

\label{sec:intro}

Variational inequalities with monotone operators are a general framework
for solving problems with convex structure including convex minimization,
convex-concave saddle point problems, and finding convex Nash equilibrium
\citep{Nemirovski04,JuditskyNT11}. Given a convex domain $\dom\subseteq\mathbb{R}^{d}$
and a monotone mapping $F:\dom\to\mathbb{R}^{d}$,
\[
\left\langle F(x)-F(y),x-y\right\rangle \ge0\ \forall x,y\in\dom
\]
we are interested in finding an approximation to a solution $x^{*}$
such that\footnote{Such a solution is called a strong solution. Following previous work,
we design algorithms that converge to a weak solution. If $F$ is
monotone and continuous, a weak solution is a strong solution and
vice-versa. We defer the formal definitions to Section \ref{sec:prelim}.}
\[
\left\langle F(x^{*}),x^{*}-x\right\rangle \le0\ \forall x\in\dom
\]
More recently, algorithms developed in this framework are also applied
to non-convex problems including optimizing generative adversarial
networks (GANs) \citep{DISZ18,YadavS0JG18,CGFL19,GidelBVVL19,MertikopoulosLZFCP19}.
In this context, due to the large scale of the problems, several important
issues are brought to the fore. First, the algorithms typically require
careful settings of the step sizes based on the parameters of the
problems such as smoothness, especially for high dimensional problems
where the smoothness varies for different coordinates. Second, classical
methods based on the extra gradient algorithm \citep{Korpelevich76}
or the more general mirror prox algorithm \citep{Nemirovski04} requires
two gradient computations per iteration, which is the dominant factor
in the running time, making them twice as slow as typical gradient
descent methods. To rectify the first issue, several works have been
developed to design adaptive algorithms that automatically adapt to
the smoothness of the problem \citep{BachL19,ene2020adaptive}. These
works build upon the impressive body of works that brought about adaptive
algorithms for convex optimization methods (see e.g. \citep{McMahanS10,duchi2011adaptive,kingma2014adam}).
A different line of work focused on reducing the number of gradient
computation to one per iteration \citep{Popov80,GidelBVVL19,HIMM19,ChambolleP11,Malitsky15,CuiS16,DISZ18,MOP20}.
It is worth noting that in practice, especially in the context of
training GANs, these methods are almost always used in a heuristic
fashion along with adaptive techniques such as Adam \citep{kingma2014adam}.

In this work, we develop new algorithms achieving the best of both
worlds: our algorithms automatically adapt to the smoothness of the
problem and require only one gradient computation per iteration. We
include two variants of the core algorithm, one variant adapts to
a single shared smoothness parameter for all coordinates and the other
variant adapts simultaneously to different smoothness parameters for
different coordinates. Our algorithms can be viewed as adaptive versions
of the past extra-gradient method developed by \citet{Popov80} and
further analyzed by many subsequent works including most recently
\citep{GidelBVVL19,HIMM19}. Our algorithms are universal: they work
simultaneously for non-smooth functions, smooth functions, and with
stochastic oracle access. In each of these settings, the algorithm
adapting to the scalar smoothness parameter achieves the same convergence
guarantees as the best-known algorithms using the smoothness parameter
in their step sizes. In contrast, previous adaptive algorithms \citep{BachL19,ene2020adaptive}
lose logarithmic factors compared with non-adaptive algorithms and
use twice as many operator evaluations. Furthermore, our algorithm
for scalar smoothness allows for arbitrary initialization of the normalization
factor, which is in line with the practice of initializing it to a
small constant such as $10^{-10}$. In contrast, previous works need
the initial value to be at least the maximum operator value or the
radius of the domain. Our analysis framework is general and versatile,
and it allows us to analyze several variants of our algorithms, including
algorithms based on the extra-gradient method \citep{Korpelevich76}
and algorithms that are suitable for unbouded optimization domains.
A detailed comparison of the convergence rates is described in Table
\ref{tb:Comparison-of-adaptive}. We provide a discussion of the algorithmic
and technical contributions in Sections \ref{sec:algorithms} and
\ref{sec:adapeg-analysis}. We note that the convergence guarantees
obtained by our scalar algorithm are optimal in all settings (non-smooth,
smooth, and stochastic), as they match known lower bounds for convex
optimization and saddle-point problems \citep{nemirovsky1983problem,nemirovsky1992information,OuyangX21}.
Moreover, all of our algorithms automatically adapt to unknown problem
parameters such as the smoothness and the norm of the operator, and
the variance of the stochastic evaluation oracle.

\begin{table}
\centering{}%
\begin{tabular}{|c|c|c|}
\hline 
\textbf{\scriptsize{}Reference} & \textbf{\scriptsize{}Non-smooth} & \textbf{\scriptsize{}Smooth}\tabularnewline
\hline 
\textbf{\scriptsize{}Thm. \ref{thm:adapeg-convergence}} & {\scriptsize{}$O\left(\frac{R\left(G+\sigma\right)}{\sqrt{T}}\right)$} & {\scriptsize{}$O\left(\frac{\beta R^{2}}{T}+\frac{R\sigma}{\sqrt{T}}\right)$}\tabularnewline
\hline 
\multicolumn{3}{|c|}{\textbf{\scriptsize{}scalar}{\scriptsize{} step sizes, }\textbf{\scriptsize{}bounded}{\scriptsize{}
domain, }\textbf{\scriptsize{}1 evaluation}{\scriptsize{} per iteration}}\tabularnewline
\hline 
\hline 
{\scriptsize{}\citet{BachL19}} & {\scriptsize{}$O\left(\frac{\widehat{G}R\sqrt{\ln T}}{\sqrt{T}}\right)$} & {\scriptsize{}$O\left(\frac{\widehat{G}R+\beta R^{2}\left(1+\ln\left(\beta R/\widehat{G}\right)\right)+G^{2}}{T}+\frac{R\sigma\sqrt{\ln T}}{\sqrt{T}}\right)$}\tabularnewline
\hline 
\multicolumn{3}{|c|}{\textbf{\scriptsize{}scalar}{\scriptsize{} step sizes, }\textbf{\scriptsize{}bounded}{\scriptsize{}
domain, }\textbf{\scriptsize{}2 evaluations}{\scriptsize{} per iteration}}\tabularnewline
\hline 
\hline 
\textbf{\scriptsize{}Thm. \ref{thm:adapeg-unbounded-convergence}} & {\scriptsize{}$O\left(\frac{\left\Vert x_{0}-x^{*}\right\Vert ^{2}+G^{2}}{T}+\frac{\left\Vert x_{0}-x^{*}\right\Vert \left(G+\sigma\right)}{\sqrt{T}}\right)$} & {\scriptsize{}$O\left(\frac{\beta\left\Vert x_{0}-x^{*}\right\Vert ^{2}+\left\Vert x_{0}-x^{*}\right\Vert G+G^{2}}{T}+\frac{\left\Vert x_{0}-x^{*}\right\Vert \sigma}{\sqrt{T}}\right)$}\tabularnewline
\hline 
\multicolumn{3}{|c|}{\textbf{\scriptsize{}scalar}{\scriptsize{} step sizes, }\textbf{\scriptsize{}arbitrary}{\scriptsize{}
domain, }\textbf{\scriptsize{}1 evaluation}{\scriptsize{} per iteration}}\tabularnewline
\hline 
\hline 
\multirow{2}{*}{{\scriptsize{}\citet{antonakopoulos2021adaptive}}} & \multirow{2}{*}{{\scriptsize{}$O\left(\frac{\left\Vert x_{0}-x^{*}\right\Vert ^{2}+G^{3}+G\ln\left(1+G^{2}T\right)}{\sqrt{T}}\right)$}} & {\scriptsize{}$O\left(\frac{\left(\beta\left\Vert x_{0}-x^{*}\right\Vert ^{2}+\beta^{4}+\beta^{2}\left\Vert F(x_{t_{0}+1/2})-F(x_{t_{0}})\right\Vert ^{2}\right)^{3/2}}{T}\right)$}\tabularnewline
 &  & {\scriptsize{}$t_{0}$ is the last iteration $t$ such that $\eta_{t_{0}}\geq\frac{c}{\beta}$
for constant $c$}\tabularnewline
\hline 
\multicolumn{3}{|c|}{\textbf{\scriptsize{}deterministic}{\scriptsize{} ($\sigma=0$),}\textbf{\scriptsize{}
scalar}{\scriptsize{} step sizes, }\textbf{\scriptsize{}arbitrary
}{\scriptsize{}optimization domain, }\textbf{\scriptsize{}2 evaluations}{\scriptsize{}
per iteration}}\tabularnewline
\hline 
\hline 
\multirow{2}{*}{\textbf{\scriptsize{}Lem. \ref{lem:without-G}}} & \multirow{2}{*}{{\scriptsize{}$O\left(\frac{\left\Vert x_{0}-x^{*}\right\Vert ^{2}+G^{2}}{T}+\frac{\left\Vert x_{0}-x^{*}\right\Vert G}{\sqrt{T}}\right)$}} & {\scriptsize{}$O\left(\frac{\beta^{2}\left\Vert x_{0}-x^{*}\right\Vert ^{2}+\left\Vert F(x_{\tau})-F(x_{\tau-1})\right\Vert ^{2}+\left\Vert F(x_{\tau-1})-F(x_{\tau-2})\right\Vert ^{2}}{T}\right)$}\tabularnewline
 &  & {\scriptsize{}$\tau$ is the last iteration $t$ such that $\gamma_{t-2}\leq c\beta$
for constant $c$}\tabularnewline
\hline 
\multicolumn{3}{|c|}{\textbf{\scriptsize{}deterministic}{\scriptsize{} ($\sigma=0$),}\textbf{\scriptsize{}
scalar}{\scriptsize{} step sizes, }\textbf{\scriptsize{}arbitrary
}{\scriptsize{}optimization domain, }\textbf{\scriptsize{}1 evaluation}{\scriptsize{}
per iteration}}\tabularnewline
\hline 
\hline 
\textbf{\scriptsize{}Thm. \ref{thm:adapeg-vector-convergence}} & {\scriptsize{}$O\left(\frac{dR_{\infty}^{2}}{T}+\frac{\sqrt{d}R_{\infty}G+\left(\sqrt{d}R_{\infty}+R\right)\sigma}{\sqrt{T}}\right)$} & {\scriptsize{}$O\left(\frac{dR_{\infty}^{2}\beta^{2}}{T}+\frac{\left(\sqrt{d}R_{\infty}+R\right)\sigma}{\sqrt{T}}\right)$}\tabularnewline
\hline 
\multicolumn{3}{|c|}{\textbf{\scriptsize{}vector }{\scriptsize{}step sizes, }\textbf{\scriptsize{}bounded}{\scriptsize{}
optimization domain, }\textbf{\scriptsize{}1 evaluation}{\scriptsize{}
per iteration}}\tabularnewline
\hline 
\hline 
\textbf{\scriptsize{}Thm. \ref{thm:adapeg-vector-movement-convergence}} & {\scriptsize{}$O\left(\frac{dR_{\infty}^{2}}{T}+\frac{\sqrt{d}R_{\infty}G\left(\sqrt{\ln\left(\frac{GT}{R_{\infty}}\right)}\right)+R\sigma}{\sqrt{T}}\right)$} & {\scriptsize{}$O\left(\frac{R_{\infty}^{2}\sum_{i=1}^{d}\beta_{i}\ln\beta_{i}}{T}+\frac{R\sigma}{\sqrt{T}}\right)$}\tabularnewline
\hline 
\multicolumn{3}{|c|}{\textbf{\scriptsize{}vector }{\scriptsize{}step sizes, }\textbf{\scriptsize{}bounded}{\scriptsize{}
optimization domain, }\textbf{\scriptsize{}1 evaluation}{\scriptsize{}
per iteration}}\tabularnewline
\hline 
\hline 
{\scriptsize{}\citet{ene2020adaptive}} & {\scriptsize{}$O\left(\frac{dR_{\infty}^{2}}{T}+\frac{\sqrt{d}R_{\infty}\left(G\sqrt{\ln\left(\frac{GT}{R_{\infty}}\right)}+\sigma\sqrt{\ln\left(\frac{T\sigma}{R_{\infty}}\right)}\right)}{\sqrt{T}}\right)$} & {\scriptsize{}$O\left(\frac{R_{\infty}^{2}\sum_{i=1}^{d}\beta_{i}\ln\beta_{i}}{T}+\frac{\sqrt{d}R_{\infty}\sigma\sqrt{\ln\left(\frac{T\sigma}{R_{\infty}}\right)}}{\sqrt{T}}\right)$}\tabularnewline
\hline 
\multicolumn{3}{|c|}{\textbf{\scriptsize{}vector }{\scriptsize{}step sizes, }\textbf{\scriptsize{}bounded}{\scriptsize{}
optimization domain, }\textbf{\scriptsize{}2 evaluations}{\scriptsize{}
per iteration}}\tabularnewline
\hline 
\end{tabular}
\caption{\label{tb:Comparison-of-adaptive}Comparison of adaptive algorithms
for variational inequalities. $R,R_{\infty}$ are the $\ell_{2}$
and $\ell_{\infty}$ diameter of the domain. $G$ is an upper bound
on the $\ell_{2}$-norm of $F(\cdot)$. $\sigma^{2}$ is the variance
of the stochastic oracle for $F(\cdot)$ (for deterministic setting,
set $\sigma=0$). $d$ is the dimensions of the domain. In the smooth
setting, $F$ is smooth with respect to a norm $\left\Vert \cdot\right\Vert _{\protect\sm}$,
where $\protect\sm=\protect\diag\left(\beta_{1},\dots,\beta_{d}\right)$
is a diagonal matrix with $\beta_{1},\dots,\beta_{d}>0$; we let $\beta=\max_{i}\beta_{i}$.
The scalar algorithms set a single step size for all coordinates,
whereas the vector algorithms set a per-coordinate step size. The
stated bounds are obtained by setting $\gamma_{0}=0$ in Theorem \ref{thm:adapeg-convergence}
and $\gamma_{0}=1$ In Theorems \ref{thm:adapeg-unbounded-convergence},
\ref{thm:adapeg-vector-convergence}, \ref{thm:adapeg-vector-movement-convergence}.
The analysis of \citet{BachL19} requires the stochastic gradients
to be bounded almost surely by a parameter $\widehat{G}$, which is
stronger than the variance assumption we use in this paper. Additionally,
the algorithm of \citet{BachL19} requires an estimate for $\widehat{G}$
in order to step size.}
\end{table}

\section{Preliminaries}

\label{sec:prelim}

\textbf{Variational inequalities:} In this paper, we consider the
problem of finding strong solutions to variational inequalities with
monotone operators. Let $\dom\subseteq\R^{d}$ be a non-empty closed
convex set ($\dom$ may be unbounded). Let $F\colon\dom\to\R^{d}$
be a map. The variational inequality problem is to find a solution
$x^{*}\in\dom$ satisfying
\begin{equation}
\left\langle F(x^{*}),x^{*}-x\right\rangle \leq0\quad\forall x\in\dom\label{eq:vi-strong-sol}
\end{equation}
A solution $x^{*}$ satisfying the above condition is often called
a \emph{strong solution} to the variational inequality.

The operator $F$ is \emph{monotone} if it satisfies
\begin{equation}
\left\langle F(x)-F(y),x-y\right\rangle \geq0\quad\forall x,y\in\dom\label{eq:monotone-operator}
\end{equation}
A related notion is a \emph{weak solution}, i.e., a point $x^{*}\in\dom$
satisfying
\begin{equation}
\left\langle F(x),x^{*}-x\right\rangle \leq0\quad\forall x\in\dom\label{eq:vi-weak-sol}
\end{equation}
If $F$ is monotone and continuous, a weak solution is a strong solution
and vice-versa. 

Let $\left\Vert \cdot\right\Vert $ be a norm and let $\left\Vert \cdot\right\Vert _{*}$
be its dual norm. The operator $F$ is \emph{$\beta$-smooth} with
respect to the norm $\left\Vert \cdot\right\Vert $ if it satisfies
\begin{equation}
\left\Vert F(x)-F(y)\right\Vert _{*}\leq\beta\left\Vert x-y\right\Vert \label{eq:smooth-operator}
\end{equation}
The operator $F$ is \emph{$\beta$-cocoercive }with respect to the
norm $\left\Vert \cdot\right\Vert $ if it satisfies
\begin{equation}
\left\langle F(x)-F(y),x-y\right\rangle \geq\frac{1}{\beta}\left\Vert F(x)-F(y)\right\Vert _{*}^{2}\quad\forall x,y\in\dom\label{eq:cocoercive-operator}
\end{equation}
Using Holder's inequality, we can readily verify that, if $F$ is
$\beta$-cocoercive, then it is monotone and $\beta$-smooth. 

\textbf{Special cases:} Two well-known special cases of the variational
inequality problem with monotone operators are convex minimization
and convex-concave saddle point problems.

In convex minimization, we are given a convex function $f\colon\dom\to\R^{d}$
and the goal is to find a solution $x^{*}\in\arg\min_{x\in\dom}f(x)$.
The operator is the gradient of $f$, i.e., $F=\nabla f$ (if $f$
is not differentiable, the operator is a subgradient of $f$). The
monotonicity condition (\ref{eq:monotone-operator}) is equivalent
to $f$ being convex. A strong solution is a point $x^{*}$ that satisfies
the first-order optimality condtion and thus it is a global minimizer
of $f$. The smoothness condition (\ref{eq:smooth-operator}) coincides
with the usual smoothness condition from convex optimization. If $f$
is convex and $\beta$-smooth, then $F=\nabla f$ is $\beta$-cocoercive
(see, e.g., Theorem 2.1.5 in the textbook \citep{nesterov2013introductory}).

In convex-concave saddle point problems, we are given a function $f\colon\mathcal{U}\times\mathcal{V}\to\R^{d}$
such that $f(u,v)$ is convex in $u$ and concave in $v$, and the
goal is to solve $\min_{u\in\mathcal{U}}\max_{v\in\mathcal{V}}f(u,v)$.
The operator is $F=\left(\nabla_{u}f,-\nabla_{v}f\right)$. A strong
solution is a point $(u^{*},v^{*})$ that is a global saddle point,
i.e.,
\[
f(u^{*},v)\leq f(u^{*},v^{*})\leq f(u,v^{*})\quad\forall(u,v)\in\mathcal{U}\times\mathcal{V}
\]

\textbf{Error function:} Following previous work \citep{Nemirovski04,Nesterov07},
we analyze convergence via the error (or merit) function. Following
\citep{Nesterov07}, we choose an arbitrary point $x_{0}\in\dom$.
For any fixed positive value $D$, we define
\begin{equation}
\err_{D}(x)=\sup_{y\in\dom}\left\{ \left\langle F(y),x-y\right\rangle \colon\left\Vert y-x_{0}\right\Vert \leq D\right\} \label{eq:restricted-error-fn}
\end{equation}
If $\dom$ is a bounded domain, we define
\begin{equation}
\err(x)=\sup_{y\in\dom}\left\langle F(y),x-y\right\rangle \label{eq:error-fn}
\end{equation}
 The following lemma, shown in \citep{Nesterov07}, justifies the
use of the error function to analyze convergence.
\begin{lem}
\emph{\label{lem:error-function}\citep{Nesterov07}} Let $D$ be
any fixed positive value. The function $\err_{D}$ is well-defined
and convex on $\mathbb{R}^{d}$. For any $x\in\dom$ such that $\left\Vert x-x_{0}\right\Vert \leq D$,
we have $\err_{D}(x)\geq0$. If $x^{*}$ is a weak solution and $\left\Vert x^{*}-x_{0}\right\Vert \leq D$,
then $\err_{D}(x^{*})=0$. Moreover, if $\err_{D}(x)=0$ for some
$x\in\dom$ with $\left\Vert x-x_{0}\right\Vert <D$, then $x$ is
a weak solution.
\end{lem}
We will use the following inequalities that were shown in previous
work.
\begin{lem}
\label{lem:ineq} \emph{\citep{duchi2011adaptive,McMahanS10}} Let
$a_{1},\dots,a_{T}$ be non-negative scalars. We have
\[
\sqrt{\sum_{t=1}^{T}a_{t}}\leq\sum_{t=1}^{T}\frac{a_{t}}{\sqrt{\sum_{s=1}^{t}a_{s}}}\leq2\sqrt{\sum_{t=1}^{T}a_{t}}
\]
\end{lem}
\begin{lem}
\label{lem:ineq-off-by-one} \emph{\citep{BachL19}} Let $a_{1},\dots,a_{T}\in[0,a]$
be non-negative scalars that are at most $a$. Let $a_{0}\geq0$.
We have
\[
\sqrt{a_{0}+\sum_{t=1}^{T-1}a_{t}}-\sqrt{a_{0}}\leq\sum_{t=1}^{T}\frac{a_{t}}{\sqrt{a_{0}+\sum_{s=1}^{t-1}a_{s}}}\leq\frac{2a}{\sqrt{a_{0}}}+3\sqrt{a}+3\sqrt{a_{0}+\sum_{t=1}^{T-1}a_{t}}
\]
\end{lem}
We will also make use of the following facts from Fenchel duality.
Let $\phi\colon\dom\to\R$ be a differentiable convex function. The
function $\phi$ is $\beta$-smooth with respect to the norm $\left\Vert \cdot\right\Vert $
if
\[
\phi(y)\leq\phi(x)+\left\langle \nabla\phi(x),y-x\right\rangle +\frac{\beta}{2}\left\Vert y-x\right\Vert ^{2}\quad\forall x,y\in\dom
\]
The function $\phi$ is $\alpha$-strongly convex with respect to
the norm $\left\Vert \cdot\right\Vert $ if
\[
\phi(y)\geq\phi(x)+\left\langle \nabla\phi(x),y-x\right\rangle +\frac{\alpha}{2}\left\Vert y-x\right\Vert ^{2}\quad\forall x,y\in\dom
\]
The Fenchel conjugate of $\phi$ is the function $\phi^{*}\colon\dom\to\R$
such that
\[
\phi^{*}(z)=\max_{x\in\dom}\left\{ \left\langle x,z\right\rangle -\phi(x)\right\} \quad\forall z\in\dom
\]

\begin{lem}
\emph{\label{lem:duality} (\citep{shalev2011online}, Lemma 2.19)}
Let $\phi\colon\dom\to\R$ be a closed convex function. The function
$\phi$ is $\alpha$-strongly convex with respect to a norm $\left\Vert \cdot\right\Vert $
if and only if $\phi^{*}$ is $\frac{1}{\alpha}$-smooth with respect
to the dual norm $\left\Vert \cdot\right\Vert _{*}$.
\end{lem}
\begin{lem}
\label{lem:danskin}\emph{ (Danskin's theorem, \citep{bertsekas2003convex},
Proposition 4.5.1)} Let $\phi\colon\dom\to\R$ be a strongly convex
function. For all $v\in\dom$, we have
\begin{align*}
\nabla\phi^{*}(v) & =\arg\min_{u\in\dom}\left\{ \phi(u)-\left\langle u,v\right\rangle \right\} 
\end{align*}
\end{lem}
\textbf{Additional notation:} Throughout the paper, the norm $\left\Vert \cdot\right\Vert $
without a subscript denotes the standard $\ell_{2}$-norm. We also
use the Mahalanobis norm $\left\Vert x\right\Vert _{\A}:=\sqrt{x^{\top}\A x}$,
where $\A\in\R^{d\times d}$ is a positive definite matrix. The dual
norm of $\left\Vert \cdot\right\Vert _{\A}$ is $\left\Vert \cdot\right\Vert _{\A^{-1}}$.
For a diagonal matrix $\D\in\R^{d\times d}$, we let $\D_{i}$ denote
the $i$-th diagonal entry of $\D$ and we let $\tr(\D)=\sum_{i=1}^{d}\D_{i}$
denote the trace of $\D$. For bounded domains $\dom$, we let $R$
and $R_{\infty}$ denote the $\ell_{2}$ and $\ell_{\infty}$ diameter
of $\dom$: $R=\max_{x,y\in\dom}\left\Vert x-y\right\Vert $, $R_{\infty}=\max_{x,y\in\dom}\left\Vert x-y\right\Vert _{\infty}$.
We let $G=\max_{x\in\dom}\left\Vert F(x)\right\Vert $.

\section{Algorithms and convergence guarantees}

\label{sec:algorithms}

\begin{algorithm}
\caption{AdaPEG algorithm for bounded domains $\protect\dom$.}
\label{alg:adapeg}

Let $x_{0}=z_{0}\in\dom$, $\gamma_{0}\geq0$, $\eta>0$.

For $t=1,\dots,T$, update:
\begin{align*}
x_{t} & =\arg\min_{u\in\dom}\left\{ \left\langle \widehat{F(x_{t-1})},u\right\rangle +\frac{1}{2}\gamma_{t-1}\left\Vert u-z_{t-1}\right\Vert ^{2}\right\} \\
z_{t} & =\arg\min_{u\in\dom}\left\{ \left\langle \widehat{F(x_{t})},u\right\rangle +\frac{1}{2}\gamma_{t-1}\left\Vert u-z_{t-1}\right\Vert ^{2}+\frac{1}{2}\left(\gamma_{t}-\gamma_{t-1}\right)\left\Vert u-x_{t}\right\Vert ^{2}\right\} \\
\gamma_{t} & =\frac{1}{\eta}\sqrt{\eta^{2}\gamma_{0}^{2}+\sum_{s=1}^{t}\left\Vert \widehat{F(x_{s})}-\widehat{F(x_{s-1})}\right\Vert ^{2}}
\end{align*}

Return $\overline{x}_{T}=\frac{1}{T}\sum_{t=1}^{T}x_{t}$.
\end{algorithm}

\begin{algorithm}
\caption{AdaPEG algorithm for unbounded domains $\protect\dom$.}
\label{alg:adapeg-unbounded}

Let $x_{0}=z_{0}\in\dom$, $\gamma_{0}\geq0$,$\gamma_{-1}=0$, $\eta>0$.

For $t=1,\dots,T$, update:
\begin{align*}
x_{t} & =\arg\min_{u\in\dom}\left\{ \left\langle \widehat{F(x_{t-1})},u\right\rangle +\frac{1}{2}\gamma_{t-2}\left\Vert u-z_{t-1}\right\Vert ^{2}+\frac{1}{2}\left(\gamma_{t-1}-\gamma_{t-2}\right)\left\Vert u-x_{0}\right\Vert ^{2}\right\} \\
z_{t} & =\arg\min_{u\in\dom}\left\{ \left\langle \widehat{F(x_{t})},u\right\rangle +\frac{1}{2}\gamma_{t-2}\left\Vert u-z_{t-1}\right\Vert ^{2}+\frac{1}{2}\left(\gamma_{t-1}-\gamma_{t-2}\right)\left\Vert u-x_{0}\right\Vert ^{2}\right\} \\
\gamma_{t} & =\frac{1}{\eta}\sqrt{\eta^{2}\gamma_{0}^{2}+\sum_{s=1}^{t}\left\Vert \widehat{F(x_{s})}-\widehat{F(x_{s-1})}\right\Vert ^{2}}
\end{align*}

Return $\overline{x}_{T}=\frac{1}{T}\sum_{t=1}^{T}x_{t}$.
\end{algorithm}

In this section, we describe our algorithms for variational inequalities
and state their convergence guarantees. For all of our theoretical
results, we assume that the operator $F$ is monotone. We also assume
that we can perform projections onto $\dom$. We assume that the algorithms
have access to a stochastic evaluation oracle that, on input $x_{t}$,
it returns a random vector $\widehat{F(x_{t})}$ satisfying the following
standard assumptions for a fixed (but unknown) scalar $\sigma$:
\begin{align}
\E\left[\widehat{F(x_{t})}\vert x_{1},\dots,x_{t}\right] & =F(x_{t})\label{eq:stoch-assumption-unbiased}\\
\E\left[\left\Vert \widehat{F(x_{t})}-F(x_{t})\right\Vert ^{2}\right] & \leq\sigma^{2}\label{eq:stoch-assumption-variance}
\end{align}

\subsection{Algorithm for bounded domains}

Our algorithm for bounded domains is shown in Algorithm \ref{alg:adapeg}.
Its analysis assumes that the optimization domain $\dom$ has bounded
$\ell_{2}$-norm radius, $R=\max_{x,y\in\dom}\left\Vert x-y\right\Vert $.
The algorithm can be viewed as an adaptive version of the Past Extra-Gradient
method of \citet{Popov80}. Our update rule for the step sizes can
be viewed as a generalization to the variational inequalities setting
of the step sizes used by \citet{MohriYang16,KavisLBC19,joulanisimpler}
for convex optimization.

The following theorem states the convergence guarantees for Algorithm
\ref{alg:adapeg}. We give the analysis in Section \ref{sec:adapeg-analysis}.
Similarly to Adagrad, setting $\eta$ proportional to the radius of
the domain leads to the optimal dependence on the radius and the guarantee
smoothly degrades as $\eta$ moves further away from the optimal choice.
For simplicity, the theorem below states the convergence guarantee
for $\eta=\Theta(R)$, and we give the guarantee and analysis for
arbitrary $\eta$ in Section \ref{sec:adapeg-analysis}.
\begin{thm}
\label{thm:adapeg-convergence} Let $F$ be a monotone operator. Let
$\eta=\Theta(R)$, where $R=\max_{x,y\in\dom}\left\Vert x-y\right\Vert $
is the $\ell_{2}$-diameter of the domain. Let $\avx_{T}$ be the
solution returned by Algorithm \ref{alg:adapeg}. If $F$ is non-smooth,
we have
\[
\E\left[\err(\avx_{T})\right]\leq O\left(\frac{\gamma_{0}R^{2}}{T}+\frac{R\left(G+\sigma\right)}{\sqrt{T}}\right)
\]
where $G=\max_{x\in\dom}\left\Vert F(x)\right\Vert $ and $\sigma^{2}$
is the variance parameter from assumption (\ref{eq:stoch-assumption-variance}).

If $F$ is $\beta$-smooth with respect to the $\ell_{2}$-norm, we
have
\[
\E\left[\err(\avx_{T})\right]\leq O\left(\frac{\left(\beta+\gamma_{0}\right)R^{2}}{T}+\frac{R\sigma}{\sqrt{T}}\right)
\]
\end{thm}
\begin{proof}
\textbf{(Sketch)} Similarly to prior works, we first upper bound the
error function using the stochastic regret ($\xi_{t}:=F(x_{t})-\widehat{F(x_{t})}$):
\[
T\cdot\err(\avx_{T})\leq\underbrace{\sup_{y\in\dom}\left(\sum_{t=1}^{T}\left\langle \widehat{F(x_{t})},x_{t}-y\right\rangle \right)}_{\text{stochastic regret}}+\underbrace{R\left\Vert \sum_{t=1}^{T}\xi_{t}\right\Vert +\sum_{t=1}^{T}\left\langle \xi_{t},x_{t}-x_{0}\right\rangle }_{\text{stochastic error}}
\]
Next, we analyze the stochastic regret. We split the regret into three
terms and analyze each term separately:
\[
\left\langle \widehat{F(x_{t})},x_{t}-y\right\rangle =\left\langle \widehat{F(x_{t})},z_{t}-y\right\rangle +\left\langle \widehat{F(x_{t})}-\widehat{F(x_{t-1})},x_{t}-z_{t}\right\rangle +\left\langle \widehat{F(x_{t-1})},x_{t}-z_{t}\right\rangle 
\]
The first and third terms can be readily upper bounded via the optimality
condition for $z_{t}$ and $x_{t}$. For the second term, prior works
upper bound it in terms of the iterate movement via Cauchy-Schwartz
and smoothness. We crucially depart from this approach, and upper
bound the term using the operator value difference $\left\Vert \widehat{F(x_{t})}-\widehat{F(x_{t-1})}\right\Vert ^{2}$,
which can be significantly smaller than the iterate movement, especially
in the initial iterations. Using the resulting bound on the regret,
we obtain
\begin{align*}
T\cdot\err(\avx_{T}) & \leq\underbrace{O(R)\sqrt{\sum_{t=1}^{T}\left\Vert \widehat{F(x_{t})}-\widehat{F(x_{t-1})}\right\Vert ^{2}}}_{\text{loss}}\\
 & -\underbrace{\frac{1}{2}\sum_{t=1}^{T}\gamma_{t-1}\left(\left\Vert x_{t}-z_{t-1}\right\Vert ^{2}+\left\Vert x_{t-1}-z_{t-1}\right\Vert ^{2}\right)}_{\text{gain}}\\
 & +\underbrace{R\left\Vert \sum_{t=1}^{T}\xi_{t}\right\Vert +\sum_{t=1}^{T}\left\langle \xi_{t},x_{t}-x_{0}\right\rangle }_{\text{stochastic error}}\\
 & +\frac{1}{2}R^{2}\gamma_{0}
\end{align*}
 Next, we upper bound the net loss. For non-smooth operators, we ignore
the gain and simply upper bound the loss by $O(G\sqrt{T})$ plus an
additional stochastic error term. For smooth operators, we crucially
use the gain to offset the loss. Using a careful and involved analysis,
we upper bound the net loss by $O(\beta R^{2})$ plus an additional
stochastic error term. 

Finally, we upper bound the expected stochastic error. We do so by
leveraging the martingale assumption (\ref{eq:stoch-assumption-unbiased})
and the variance assumption (\ref{eq:stoch-assumption-variance}),
and show that the expected error is $O(\sigma\sqrt{T})$.
\end{proof}
\textbf{Comparison with prior work:} Compared to the prior works \citep{BachL19,ene2020adaptive},
our algorithms set the step sizes based on the operator value differences
instead of the iterate movement. This choice is key to obtaining optimal
convergence guarantees in all settings and optimal dependencies on
all of the problem parameters, matching the non-adaptive algorithms.
Prior works attain convergence rates that are suboptimal by a $\Omega\left(\sqrt{\ln T}\right)$
factor (Table \ref{tb:Comparison-of-adaptive}). Moreover, the prior
algorithms use the off-by-one iterate \citep{mcmahan2017survey},
which is unavoidable due to the use of the iterate movement in the
step size. These works address the off-by-one issue using additional
assumptions and pay additional error terms in the convergence. Specifically,
\citet{BachL19} require the assumption that $G:=\max_{x\in\dom}\left\Vert F(x)\right\Vert $
is bounded even when $F$ is smooth. The algorithm requires an estimate
for $G$ in order to set the step size. Additonally, the convergence
guarantee has additional error terms, including an error term of at
least $G^{2}/\gamma_{0}$. In the stochastic setting, the analysis
requires the stochastic operators to be bounded almost surely by a
parameter $\widehat{G}$, and the algorithm requires an estimate for
$\widehat{G}$ in order to set the step size. The algorithm and analysis
of \citet{ene2020adaptive} requires knowing the radius $R$ in order
to address the off-by-one issue. The algorithm of \citet{ene2020adaptive}
scales the update by $R$ to ensure that the step sizes increase by
at most a constant factor, and the analysis breaks if this is not
ensured.

In contrast, Algorithm \ref{alg:adapeg} does not suffer from the
off-by-one issue. Our analysis for smooth operators does not require
the operator norms to be bounded. Our convergence guarantee has optimal
dependence on $T$ and all problem parameters. Moreover, in the stochastic
setting, our analysis relies only on the variance assumption (\ref{eq:stoch-assumption-variance}),
which is a weaker assumption than the stochastic operators being bounded
almost surely.

Compared to standard methods such as the Past Extra-Gradient method
\citep{Popov80}, our algorithms use an additional term $(\gamma_{t}-\gamma_{t-1})\|u-x_{t}\|^{2}$
in the update rule for $z_{t}$. Our analysis framework is versatile
and allows us to analyze several variants of the algorithm, including
variants that do not include this additional term. We discuss the
variants and provide experimental results in Section \ref{sec:experiments-extra}.
The additional term leads to a tighter analysis with optimal dependencies
on all problem parameters and improved constant factors. The algorithm
variants performed similarly in our experiments involving bilinear
saddle point problems. Our analysis readily extends to the 2-call
variants of the algorithms based on the Extra-Gradient algorithm \citep{Korpelevich76}.
We discuss the 2-call variants in Section \ref{sec:extensions} and
give experimental results in Section \ref{sec:experiments-extra}.
In all of the experiments, the 1-call algorithms performed equally
well or better than their 2-call counterparts.

Our algorithm and analysis allows us to set $\gamma_{0}$ and $\eta$
to arbitrary constants, analogous to how adaptive algorithms such
as Adagrad are implemented and used in practice ($\gamma_{0}$ is
analogous to the $\epsilon$ paramater for Adagrad). For example,
the implementation of Adagrad in pytorch sets $\epsilon=10^{-10}$
and $\eta=0.01$. In contrast, previous works \citep{BachL19,ene2020adaptive}
need the initial value $\gamma_{0}$ to be at least the maximum operator
norm or the radius of the domain. Moreover, the analysis of \citet{ene2020adaptive}
does not allow the algorithm to be used with a base learning rate
$\eta\neq\Theta(R)$: as noted above, the algorithm needs to scale
the update by the radius to ensure that the step sizes increase by
at most a constant factor, and the analysis breaks if this is not
ensured.

\subsection{Algorithm for unbounded domains}

Our algorithm for unbounded domains is shown in Algorithm \ref{alg:adapeg-unbounded}.
The algorithm uses the distance from the initial point $x_{0}$ that
ensures that the iterates do not diverge. The approach is inspired
by the work of \citet{fang2020online} for online convex optimization,
which used the distance to $x_{0}$ to stabilize mirror descent in
the setting where the step sizes are chosen non-adaptively (the algorithm
of \citet{fang2020online} uses the step size for the future iteration
$t+1$ to perform the update for the current iteration $t$).

To the best of our knowledge, this is the first adaptive method for
general unbounded domains, even in the special case of convex minimization.
The convergence guarantees of existing adaptive algorithms in the
Adagrad family depends on the maximum distance $\left\Vert x_{t}-x^{*}\right\Vert $
between the iterates and thrainede unconst optimum (a point $x^{*}$
with $\nabla f(x^{*})=0$). Since these distances could diverge if
the domain is unbounded, an approach employed in prior work (e.g.,
\citep{Levy17}) is to project the iterates onto a bounded domain
containing $x^{*}$ (such as a ball). The resulting algorithms require
access to an optimization domain containing the unconstrained optimum
which may not be available or requires additional tuning (for example,
for $\dom=\R^{d}$, the distance $\left\Vert x_{0}-x^{*}\right\Vert $
is an unknown parameter that we would need to tune in order to restrict
the optimization to a ball centered at $x_{0}$ that contains $x^{*}$).
Moreover, the restriction that the optimization domain contains the
unconstrained optimum limits the applicability of the algorithms,
as it does not allow for arbitrary constrains. Additionally, our algorithms
readily extend to the more general setting of Bregman distances. Even
if the domain $\dom$ is bounded, the Bregman distances are potentially
unbounded (e.g., KL-divergence distances on the simplex), and previous
adaptive methods cannot be applied.

The following theorem states the convergence guarantees for Algorithm
\ref{alg:adapeg-unbounded}. We give the analysis in Section \ref{sec:adapeg-unbounded-analysis}.
As before, for simplicity, the theorem states the guarantees when
$\eta$ is set optimally, and we give the guarantee and analysis for
arbitrary $\eta$ in Section \ref{sec:adapeg-unbounded-analysis}.
In contrast to Algorithm \ref{alg:adapeg}, Algorithm \ref{alg:adapeg-unbounded}
has the off-by-one iterate (see the discussion above) and we incur
an additional error term. In the following theorem, to allow for a
direct comparison with \citep{BachL19}, we assume that the operator
norms are bounded even for smooth operators. This assumption is not
necessary (see Lemma \ref{lem:without-G}).
\begin{thm}
\label{thm:adapeg-unbounded-convergence} Let $F$ be a monotone operator.
Let $D>0$ be any fixed positive value. Let $\eta=\Theta(D)$. Let
$\avx_{T}$ be the solution returned by Algorithm \ref{alg:adapeg-unbounded}.
If $F$ is non-smooth, we have
\[
\E\left[\err_{D}(\avx_{T})\right]\leq O\left(\frac{\gamma_{0}D^{2}+\gamma_{0}^{-1}G^{2}}{T}+\frac{DG+\left(D+\gamma_{0}^{-1}\right)\sigma}{\sqrt{T}}\right)
\]
where $G=\max_{x\in\dom}\left\Vert F(x)\right\Vert $ and $\sigma^{2}$
is the variance parameter from assumption (\ref{eq:stoch-assumption-variance}).

If $F$ is $\beta$-smooth with respect to the $\ell_{2}$-norm, we
have
\[
\E\left[\err_{D}(\avx_{T})\right]\leq O\left(\frac{\left(\beta+\gamma_{0}\right)D^{2}+DG+\gamma_{0}^{-1}G^{2}}{T}+\frac{\left(D+\gamma_{0}^{-1}\right)\sigma}{\sqrt{T}}\right)
\]
 
\end{thm}
\textbf{Contemporaneous work:} \citet{antonakopoulos2021adaptive}
propose to use adaptive step sizes based on operator value differences
for the Extra-Gradient method, and return the weighted average of
the iterates with the weights given by the step sizes. In contrast
to our work, their algorithm and analysis does not extend to per-coordinate
step sizes or the stochastic setting, and the convergence rate is
sub-optimal by a $\Omega(\ln T)$ factor and has higher dependencies
on the problem parameters (Table \ref{tb:Comparison-of-adaptive}).
We note that Theorem \ref{thm:adapeg-unbounded-convergence} states
the convergence in terms of $G$, so that it can be directly compared
to \citep{BachL19}. The stated bound is incomparable to \citep{antonakopoulos2021adaptive}
in the smooth setting. However, our analysis can also be used to provide
a bound in the same spirit (Lemma \ref{lem:without-G}).

\subsection{Extensions}

Our analysis framework is versatile and it allows us to analyze several
variants and extensions of our main algorithms. In Section \ref{sec:extensions},
we consider the extension to the 2-call versions of our algorithms
based on the Extra-Gradient algorithm \citep{Korpelevich76}. In Section
\ref{sec:adapeg-bregman}, we consider the more general setting of
Bregman distances. Our analysis establishes the same convergence rate,
up to constant factors. In our experimental evaluation, given in Section
\ref{sec:experiments-extra}, the 1-call algorithms performed equally
well or better than their 2-call counterparts.

In Section \ref{sec:vector-algorithms}, we extend the algorithms
and their analysis to the vector setting where we adaptively set a
per-coordinate learning rate. The vector version of Algorithm \ref{alg:adapeg}
improves over the previous work of \citet{ene2020adaptive} by a $\Omega\left(\sqrt{\ln T}\right)$
factor (Table \ref{tb:Comparison-of-adaptive}). The algorithm has
optimal convergence for non-smooth operators and smooth operators
that are cocoercive. For smooth operators that are not cocoercive,
our convergence guarantee has a dependence of $\beta^{2}$ on the
smoothness parameter whereas the algorithm of \citet{ene2020adaptive}
has a better dependence of $\beta\ln\beta$. We note that, by building
on the work of \citet{ene2020adaptive} and our approach, we can analyze
a a single-call variant of the algorithm of their algorithm. For completeness,
we give this analysis in Section \ref{sec:adapeg-vector-movement-analysis}. 

The per-coordinate methods enjoy a speed-up compared with the scalar
method in many common scenarios, including learning problems with
sparse gradient, as discussed in more detail in Sections 1.3 and 6
in the work of \citet{duchi2011adaptive}. In our experimental evaluation,
given in Sections \ref{sec:experiments} and \ref{sec:experiments-extra}
the per-coordinate methods outperformed their scalar counterparts
in certain settings.

\section{Experimental evaluation}

\label{sec:experiments}

\begin{figure}
\label{fig:bilinear-experiments}

\includegraphics[width=0.49\linewidth]{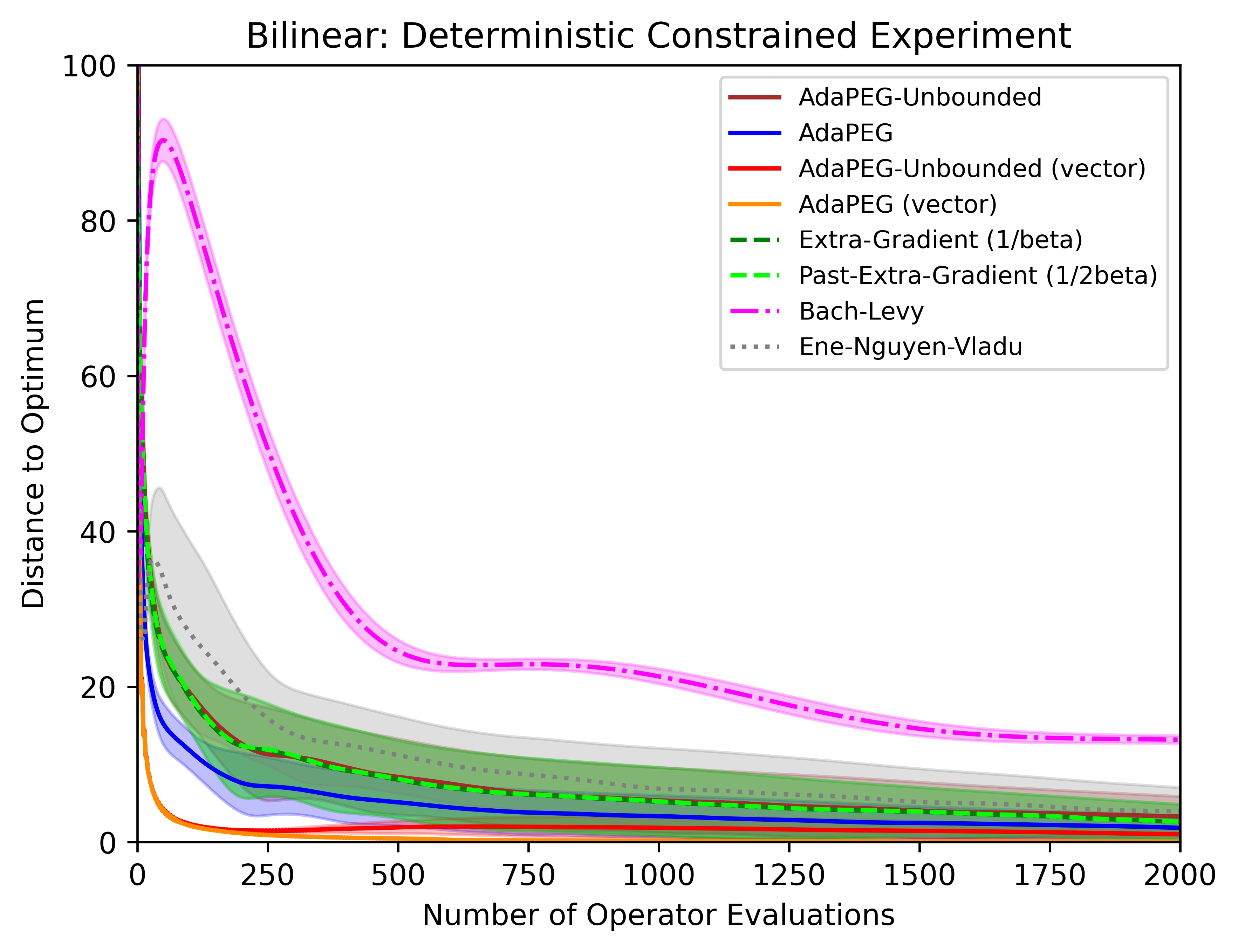}\includegraphics[width=0.49\linewidth]{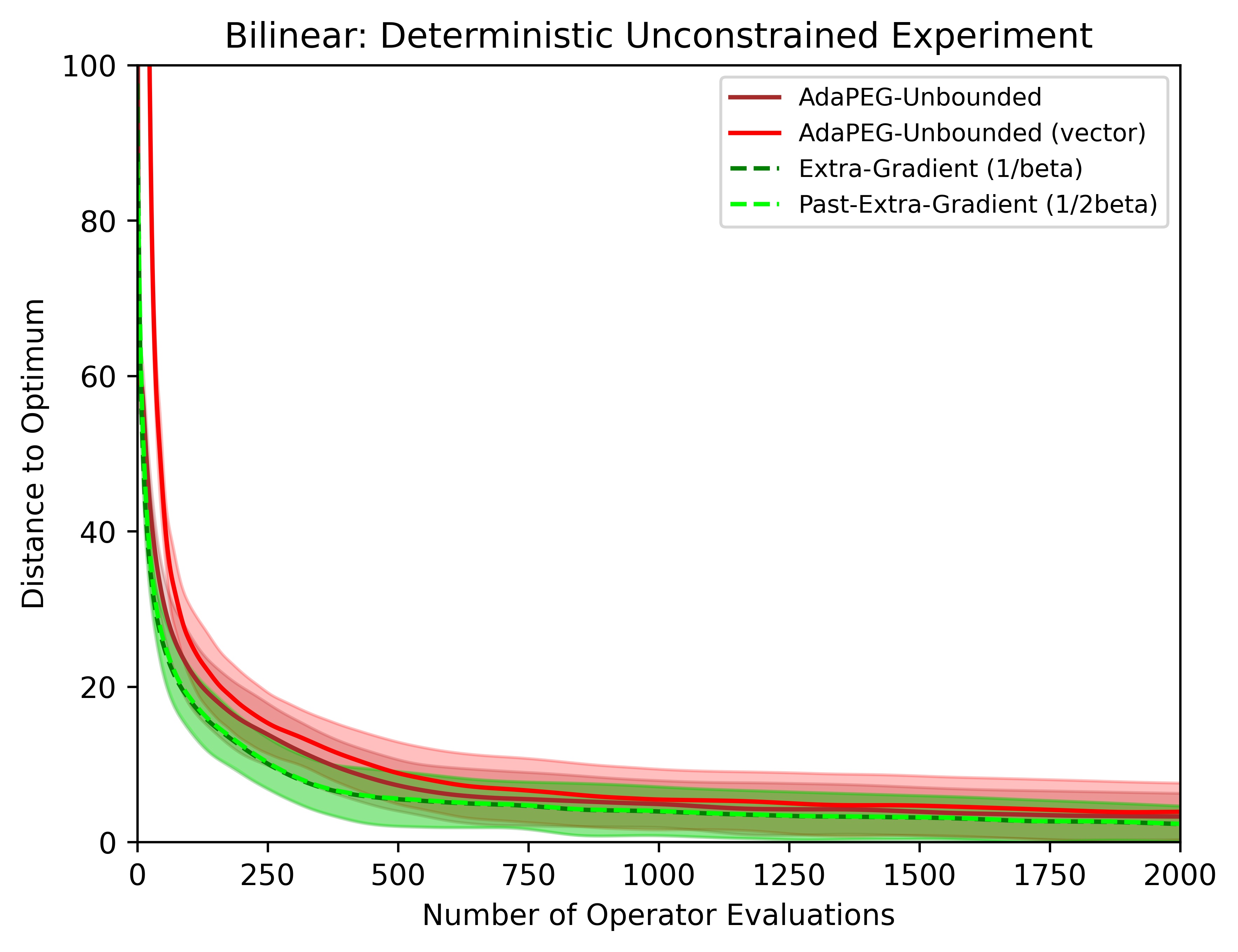}

\includegraphics[width=0.49\linewidth]{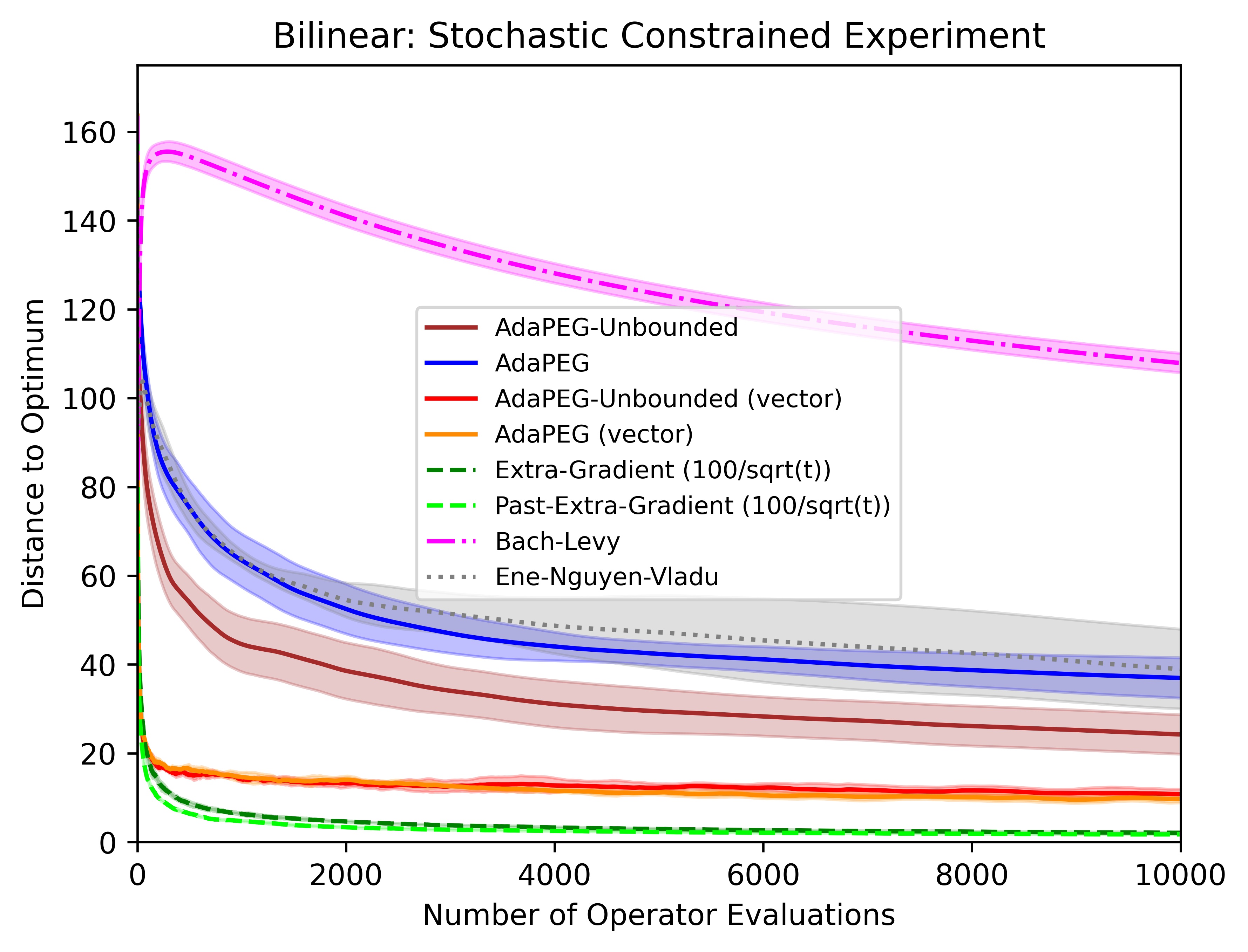}\includegraphics[width=0.49\linewidth]{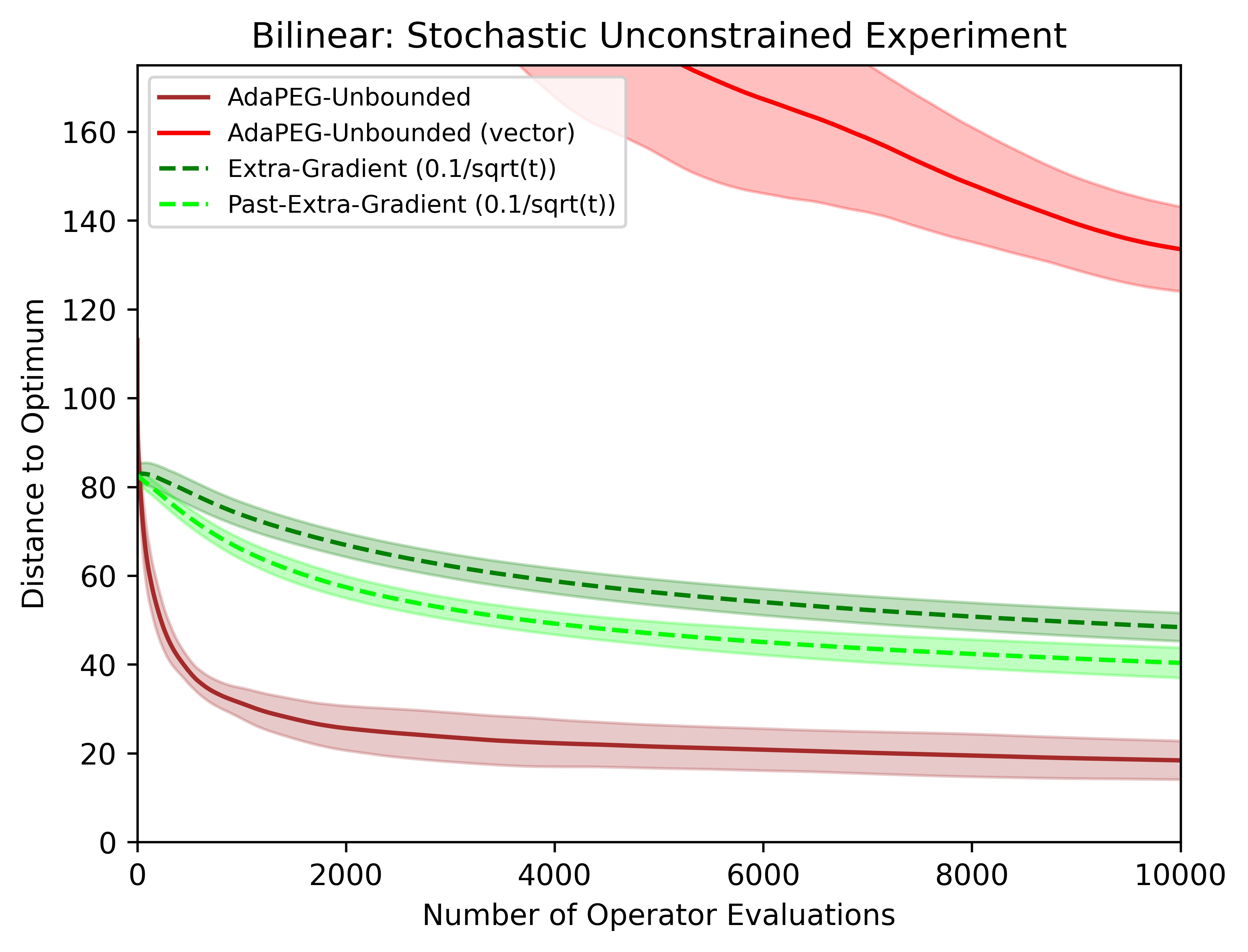}

\caption{Convergence on bilinear instances. We report the mean and standard
deviation over $5$ runs.}
\end{figure}

In this section, we give experimental results on bilinear saddle point
instances. We provide additional experimental results, including an
experiment on training generative adversarial networks, in Section
\ref{sec:experiments-extra}.

\textbf{Instances:} We consider bilinear saddle point problems $\min_{u\in\mathcal{U}}\max_{v\in\mathcal{V}}f(u,v)$,
where
\[
f(u,v)=\frac{1}{n}\sum_{i=1}^{n}u^{\top}\A^{(i)}v
\]
and $\A^{(i)}\in\mathbb{R}^{d\times d}$ for each $i\in[n]$. The
strong solution is $x^{*}=(u^{*},v^{*})=0$. Each matrix $\A^{(i)}$
was generated by first sampling a diagonal matrix with entries drawn
from the $\mathrm{Uniform}([-10,10])$ distribution, and then applying
a random rotation drawn from the Haar distribution. The initial point
$x_{0}$ was generated by sampling each entry from the $\mathrm{Uniform}([-10,10])$
distribution. We used $d=100$ in all experiments. In the deterministic
experiments, we used $n=1$. In the stochastic experiments, we used
$n=100$ and a minibatch of size $16$ for computing the stochatic
evaluations. In the unconstrained experiments, the feasible domain
is $\mathcal{X}=\mathcal{U}\times\mathcal{V}=\R^{2d}$. In the constrained
experiments, $\mathcal{X}=\mathcal{U}\times\mathcal{V}$ is an $\ell_{2}$-ball
of radius $R=2\left\Vert x_{0}-x^{*}\right\Vert $ centered at $x^{*}=0$.

\textbf{Algorithms:} We compare the following algorithms: our algorithms
with scalar step sizes (Algorithms \ref{alg:adapeg} and \ref{alg:adapeg-unbounded})
and per-coordinate step sizes (Algorithms \ref{alg:adapeg-vector}
and \ref{alg:adapeg-vector-unbounded}), the adaptive methods of \citet{BachL19}
and \citet{ene2020adaptive}, and the non-adaptive methods Extra-Gradient
\citep{Korpelevich76} and Past Extra-Gradient \citep{Popov80}. 

An experimental comparison between the 1-call algorithms and their
2-call variants can be found in Section \ref{sec:experiments-extra}.
In all of the experiments, the 1-call algorithms performed equally
well or better than their 2-call counterparts.

We also include in Section \ref{sec:experiments-extra} experimental
results that include variants of our algorithms that do not include
the extra term $\left\Vert u-x_{t}\right\Vert ^{2}$ in the update
rule for $z_{t}$. We observe that the algorithm variants perform
similarly in the experiments with bounded feasible domain. We also
evaluated the algorithm variants in the unconstrained setting, even
though this is not supported by theory. We observe that one of the
variants performs slightly better in the unconstrained stochastic
setting.

\textbf{Hyperparameters:} In the deterministic experiments, we used
a uniform step size $\eta=\frac{1}{\beta}$ for the Extra-Gradient
method and $\eta=\frac{1}{2\beta}$ for the Past Extra-Gradient method,
as suggested by the theoretical analysis \citep{HIMM19}. We observed
in our experiments that the additional factor of $2$ is neccessary
for the Past Extra-Gradient method, and the algorithm did not converge
when run with step sizes larger than $\frac{1}{2\beta}$. In the stochastic
experiments, we used decaying step sizes $\eta_{t}=\frac{c}{\sqrt{t}}$
for Extra-Gradient and Past Extra-Gradient, where $c$ was set via
a hyperparameter search. We set the parameter $G_{0}$ used by the
algorithm of \citet{BachL19} via a hyperparameter search. For our
algorithms, we set the parameter $\gamma_{0}$ via a hyperparameter
search, and we set $\eta=R$ in the constrained experiments and $\eta=\left\Vert x_{0}-x^{*}\right\Vert $
in the unconstrained experiments. All of the hyperparameter searches
picked the best value from the set $\left\{ 1,5\right\} \times\left\{ 10^{5},10^{4},\dots,10^{1},1,10^{-1},\dots,10^{-4},10^{-5}\right\} $.

\textbf{Results:} The results are shown in Figure \ref{fig:bilinear-experiments}.
We report the mean and standard deviation over $5$ runs. We note
that our algorithms have the best performance among the adaptive methods.
Moreover, our algorithms' performance was competitive with the non-adaptive
methods that have access to the smoothness parameter.

\bibliographystyle{abbrv}
\bibliography{adagrad}

\appendix

\section{Appendix outline}

The appendix is organized as follows. The reader interested in getting
an overview of the main ideas and techniques may read Section \ref{sec:adapeg-analysis}.
The analyses provided in subsequent sections are extensions of the
analysis presented in Section \ref{sec:adapeg-analysis}, and we have
included them separately for clarity and completeness.

\begin{longtable}[l]{>{\raggedright}p{2cm}>{\raggedright}p{14cm}}
\textbf{Section \ref{sec:adapeg-analysis}} & We analyze Algorithm \ref{alg:adapeg} and prove Theorem \ref{thm:adapeg-convergence}.\tabularnewline
\textbf{Section \ref{sec:adapeg-unbounded-analysis}} & We analyze Algorithm \ref{alg:adapeg-unbounded} and prove Theorem
\ref{thm:adapeg-unbounded-convergence}.\tabularnewline
\textbf{Section \ref{sec:extensions}} & We extend the algorithms and analysis to the 2-call variants based
on Extra-Gradient.\tabularnewline
\textbf{Section \ref{sec:adapeg-bregman}} & We extend the algorithms and analysis to Bregman distances.\tabularnewline
\textbf{Section \ref{sec:vector-algorithms}} & We give the algorithms and analysis for the algorithms with per-coordinate
step sizes.\tabularnewline
\textbf{Section \ref{sec:adapeg-vector-movement-analysis}} & We analyze a single-call variant of the algorithm of \citet{ene2020adaptive}.\tabularnewline
\textbf{Section \ref{sec:experiments-extra}} & We give additional experimental results.\tabularnewline
\end{longtable}

\section{Analysis of algorithm \ref{alg:adapeg}}

\label{sec:adapeg-analysis}

In this section, we analyze Algorithm \ref{alg:adapeg} and prove
Theorem \ref{thm:adapeg-convergence}. Throughout this section, we
let $\xi_{t}:=F(x_{t})-\widehat{F(x_{t})}$. As noted in Section \ref{sec:prelim},
we analyze convergence via the error function. The starting point
of our analysis is to upper bound the error function in terms of the
stochastic regret. Using the definition of the error function (\ref{eq:error-fn}),
the definition of $\avx_{T}=\frac{1}{T}\sum_{t=1}^{T}x_{t}$, and
the monotonicity of $F$ (\ref{eq:smooth-operator}), we obtain:
\begin{lem}
\label{lem:error-fn-ub} Let $R\geq\max_{x,y\in\dom}\left\Vert x-y\right\Vert $.
Let $\xi_{t}:=F(x_{t})-\widehat{F(x_{t})}$. We have
\[
T\cdot\err(\avx_{T})\leq\underbrace{\sup_{y\in\dom}\left(\sum_{t=1}^{T}\left\langle \widehat{F(x_{t})},x_{t}-y\right\rangle \right)}_{\text{stochastic regret}}+\underbrace{R\left\Vert \sum_{t=1}^{T}\xi_{t}\right\Vert +\sum_{t=1}^{T}\left\langle \xi_{t},x_{t}-x_{0}\right\rangle }_{\text{stochastic error}}
\]
\end{lem}
\begin{proof}
Using the definition of the error function (\ref{eq:error-fn}), the
definition of $\avx_{T}=\frac{1}{T}\sum_{t=1}^{T}x_{t}$, and the
monotonicity of $F$ (\ref{eq:monotone-operator}), we obtain
\[
\err(\avx_{T})=\sup_{y\in\dom}\left\langle F(y),\avx_{T}-y\right\rangle =\frac{1}{T}\sup_{y\in\dom}\left(\sum_{t=1}^{T}\left\langle F(y),x_{t}-y\right\rangle \right)\leq\frac{1}{T}\sup_{y\in\dom}\left(\sum_{t=1}^{T}\left\langle F(x_{t}),x_{t}-y\right\rangle \right)
\]
We further write
\begin{align*}
\left\langle F(x_{t}),x_{t}-y\right\rangle  & =\left\langle \widehat{F(x_{t})},x_{t}-y\right\rangle +\left\langle F(x_{t})-\widehat{F(x_{t})},x_{t}-y\right\rangle \\
 & =\left\langle \widehat{F(x_{t})},x_{t}-y\right\rangle +\left\langle F(x_{t})-\widehat{F(x_{t})},x_{0}-y\right\rangle +\left\langle F(x_{t})-\widehat{F(x_{t})},x_{t}-x_{0}\right\rangle \\
 & =\left\langle \widehat{F(x_{t})},x_{t}-y\right\rangle +\left\langle \xi_{t},x_{0}-y\right\rangle +\left\langle \xi_{t},x_{t}-x_{0}\right\rangle 
\end{align*}
where we let $\xi_{t}:=F(x_{t})-\widehat{F(x_{t})}$. Thus we obtain
\begin{align*}
\err(\avx_{T}) & \leq\frac{1}{T}\sup_{y\in\dom}\left(\sum_{t=1}^{T}\left\langle \widehat{F(x_{t})},x_{t}-y\right\rangle +\sum_{t=1}^{T}\left\langle \xi_{t},x_{0}-y\right\rangle +\sum_{t=1}^{T}\left\langle \xi_{t},x_{t}-x_{0}\right\rangle \right)\\
 & \leq\frac{1}{T}\left(\sup_{y\in\dom}\left(\sum_{t=1}^{T}\left\langle \widehat{F(x_{t})},x_{t}-y\right\rangle \right)+\sup_{y\in\dom}\left(\sum_{t=1}^{T}\left\langle \xi_{t},x_{0}-y\right\rangle \right)+\sum_{t=1}^{T}\left\langle \xi_{t},x_{t}-x_{0}\right\rangle \right)
\end{align*}
Using the Cauchy-Schwartz inequality, we obtain the following upper
bound on the second term above:

\[
\left\langle \sum_{t=1}^{T}\xi_{t},x_{0}-y\right\rangle \leq\left\Vert \sum_{t=1}^{T}\xi_{t}\right\Vert \left\Vert x_{0}-y\right\Vert \leq\left\Vert \sum_{t=1}^{T}\xi_{t}\right\Vert R
\]
Therefore
\[
\err(\avx_{T})\leq\frac{1}{T}\left(\sup_{y\in\dom}\left(\sum_{t=1}^{T}\left\langle \widehat{F(x_{t})},x_{t}-y\right\rangle \right)+R\left\Vert \sum_{t=1}^{T}\xi_{t}\right\Vert +\sum_{t=1}^{T}\left\langle \xi_{t},x_{t}-x_{0}\right\rangle \right)
\]
as needed.
\end{proof}

Next, we analyze each of the two terms in Lemma \ref{lem:error-fn-ub}
in turn. 

\subsection{Analysis of the stochastic regret}

\label{sec:adapeg-stoch-regret-analysis}

Here we analyze the stochastic regret in Lemma \ref{lem:error-fn-ub}:
\[
\underbrace{\sup_{y\in\dom}\left(\sum_{t=1}^{T}\left\langle \widehat{F(x_{t})},x_{t}-y\right\rangle \right)}_{\text{stochastic regret}}
\]
We fix an arbitrary $y\in\dom$, and we analyze the stochastic regret
$\sum_{t=1}^{T}\left\langle \widehat{F(x_{t})},x_{t}-y\right\rangle $.
A key idea is to split the inner product $\left\langle \widehat{F(x_{t})},x_{t}-y\right\rangle $
as follows:
\begin{equation}
\left\langle \widehat{F(x_{t})},x_{t}-y\right\rangle =\left\langle \widehat{F(x_{t})},z_{t}-y\right\rangle +\left\langle \widehat{F(x_{t})}-\widehat{F(x_{t-1})},x_{t}-z_{t}\right\rangle +\left\langle \widehat{F(x_{t-1})},x_{t}-z_{t}\right\rangle \label{eq:regret-split}
\end{equation}
The above split is particularly useful for the following reasons.
By inspecting the definition of $z_{t}$ and $x_{t}$, we see that
the first and the third term can be easily upper bounded using the
optimality condition. Applying the optimality condition for $z_{t}$
gives Lemma \ref{lem:regret-first-term}, and applying the optimality
condition for $x_{t}$ gives Lemma \ref{lem:regret-third-term}. 

The heart of the regret analysis is to analyze the second term $\left\langle \widehat{F(x_{t})}-\widehat{F(x_{t-1})},x_{t}-z_{t}\right\rangle $.
A common approach in previous analyses \citep{BachL19,ene2020adaptive}
is to bound this term using Cauchy-Schwartz and smoothness, leading
to a loss that is proportional to the iterate movement. When applied
to our setting, this approach gives:
\begin{align*}
\left\langle F(x_{t})-F(x_{t-1}),x_{t}-z_{t}\right\rangle  & \leq\left\Vert F(x_{t})-F(x_{t-1})\right\Vert \left\Vert x_{t}-z_{t}\right\Vert \\
 & \leq\beta\left\Vert x_{t}-x_{t-1}\right\Vert \left\Vert x_{t}-z_{t}\right\Vert \\
 & \leq\beta\left(\frac{1}{2}\left\Vert x_{t}-x_{t-1}\right\Vert ^{2}+\frac{1}{2}\left\Vert x_{t}-z_{t}\right\Vert ^{2}\right)
\end{align*}
This approach naturally leads to the use of the iterate movement as
part of the step sizes as in the previous adaptive methods \citep{BachL19,ene2020adaptive},
but it also leads to convergence guarantees that are suboptimal by
a $\sqrt{\ln T}$ factor. Our algorithm and analysis crucially departs
from this approach. In Lemma \ref{lem:regret-second-term}, we upper
bound the term using the operator value difference $\left\Vert \widehat{F(x_{t})}-\widehat{F(x_{t-1})}\right\Vert ^{2}$,
which can be significantly smaller than the iterate movement, especially
in the initial iterations. Showing that this loss is indeed smaller
requires a careful analysis, and is done in Lemma \ref{lem:net-loss-smooth}.
Lemma \ref{lem:net-loss-smooth} also accounts for the loss arising
from using the evaluations from the past.

We now return to the analysis of the stochastic regret, and upper
bound each term in (\ref{eq:regret-split}) in turn. For the first
term, we apply the optimality condition for $z_{t}$ and obtain:
\begin{lem}
\label{lem:regret-first-term}For any $y\in\dom$, we have
\begin{align*}
\left\langle \widehat{F(x_{t})},z_{t}-y\right\rangle  & \leq(\gamma_{t}-\gamma_{t-1})\frac{1}{2}\left\Vert x_{t}-y\right\Vert ^{2}+\gamma_{t-1}\frac{1}{2}\left\Vert z_{t-1}-y\right\Vert ^{2}-(\gamma_{t}-\gamma_{t-1})\frac{1}{2}\left\Vert x_{t}-z_{t}\right\Vert ^{2}\\
 & -\gamma_{t}\frac{1}{2}\left\Vert z_{t}-y\right\Vert ^{2}-\gamma_{t-1}\frac{1}{2}\left\Vert z_{t-1}-z_{t}\right\Vert ^{2}
\end{align*}
\end{lem}
\begin{proof}
By the optimality condition for $z_{t}$, for all $u\in\dom$, we
have
\begin{align*}
\left\langle \widehat{F(x_{t})}+\gamma_{t-1}(z_{t}-z_{t-1})+(\gamma_{t}-\gamma_{t-1})(z_{t}-x_{t}),z_{t}-u\right\rangle  & \le0
\end{align*}
We apply the above inequality with $u=y$ and obtain
\[
\left\langle \widehat{F(x_{t})}+\gamma_{t-1}(z_{t}-z_{t-1})+(\gamma_{t}-\gamma_{t-1})(z_{t}-x_{t}),z_{t}-y\right\rangle \le0
\]
By rearranging the above inequality and using the identity $ab=\frac{1}{2}\left(\left(a+b\right)^{2}-a^{2}-b^{2}\right)$,
we obtain
\begin{align*}
\left\langle \widehat{F(x_{t})},z_{t}-y\right\rangle  & \le(\gamma_{t}-\gamma_{t-1})\left\langle x_{t}-z_{t},z_{t}-y\right\rangle +\gamma_{t-1}\left\langle z_{t-1}-z_{t},z_{t}-y\right\rangle \\
 & =(\gamma_{t}-\gamma_{t-1})\frac{1}{2}\left(\left\Vert x_{t}-y\right\Vert ^{2}-\left\Vert x_{t}-z_{t}\right\Vert ^{2}-\left\Vert z_{t}-y\right\Vert ^{2}\right)\\
 & +\gamma_{t-1}\frac{1}{2}\left(\left\Vert z_{t-1}-y\right\Vert ^{2}-\left\Vert z_{t-1}-z_{t}\right\Vert ^{2}-\left\Vert z_{t}-y\right\Vert ^{2}\right)\\
 & =(\gamma_{t}-\gamma_{t-1})\frac{1}{2}\left\Vert x_{t}-y\right\Vert ^{2}+\gamma_{t-1}\frac{1}{2}\left\Vert z_{t-1}-y\right\Vert ^{2}\\
 & -(\gamma_{t}-\gamma_{t-1})\frac{1}{2}\left\Vert x_{t}-z_{t}\right\Vert ^{2}-\gamma_{t}\frac{1}{2}\left\Vert z_{t}-y\right\Vert ^{2}-\gamma_{t-1}\frac{1}{2}\left\Vert z_{t-1}-z_{t}\right\Vert ^{2}
\end{align*}
as needed.
\end{proof}

For the third term, we apply the optimality condition for $x_{t}$
and obtain:
\begin{lem}
\label{lem:regret-third-term} We have
\[
\left\langle \widehat{F(x_{t-1})},x_{t}-z_{t}\right\rangle \leq\gamma_{t-1}\frac{1}{2}\left\Vert z_{t-1}-z_{t}\right\Vert ^{2}-\gamma_{t-1}\frac{1}{2}\left\Vert x_{t}-z_{t-1}\right\Vert ^{2}-\gamma_{t-1}\frac{1}{2}\left\Vert x_{t}-z_{t}\right\Vert ^{2}
\]
\end{lem}
\begin{proof}
By the optimality condition for $x_{t}$, for all $u\in\dom$, we
have
\[
\left\langle \widehat{F(x_{t-1})}+\gamma_{t-1}\left(x_{t}-z_{t-1}\right),x_{t}-u\right\rangle \leq0
\]
We apply the above inequality with $u=z_{t}$ and obtain
\[
\left\langle \widehat{F(x_{t-1})}+\gamma_{t-1}\left(x_{t}-z_{t-1}\right),x_{t}-z_{t}\right\rangle \leq0
\]
By rearranging the above inequality and using the identity $ab=\frac{1}{2}\left(\left(a+b\right)^{2}-a^{2}-b^{2}\right)$,
we obtain
\begin{align*}
\left\langle \widehat{F(x_{t-1})},x_{t}-z_{t}\right\rangle  & \leq\gamma_{t-1}\left\langle z_{t-1}-x_{t},x_{t}-z_{t}\right\rangle \\
 & =\gamma_{t-1}\frac{1}{2}\left(\left\Vert z_{t-1}-z_{t}\right\Vert ^{2}-\left\Vert z_{t-1}-x_{t}\right\Vert ^{2}-\left\Vert x_{t}-z_{t}\right\Vert ^{2}\right)
\end{align*}
as needed.
\end{proof}

We now analyze the second term. We note that there are several approaches
for obtaining the desired convergenge guarantee, up constant factors.
In order to obtain the sharpest constant factors, we use an argument
that is inspired by the work of \citet{MohriYang16} for online convex
minimization. We make careful use of the definition of $z_{t}$ and
duality and obtain the following guarantee:
\begin{lem}
\label{lem:regret-second-term}We have
\[
\left\langle \widehat{F(x_{t})}-\widehat{F(x_{t-1})},x_{t}-z_{t}\right\rangle \leq\frac{1}{\gamma_{t}}\left\Vert \widehat{F(x_{t})}-\widehat{F(x_{t-1})}\right\Vert ^{2}
\]
\end{lem}
\begin{proof}
A key idea is to consider the function $\phi_{t}$, defined below,
and show that $x_{t}$ is a minimizer of $\phi_{t}$ and $z_{t}$
is a minimizer of a function that is closely related to $\phi_{t}$.
These facts together with the strong convexity of $\phi_{t}$ and
duality allow us to relate the distance between the iterates to the
operator values. 

Let
\[
\phi_{t}(u)=\left\langle \widehat{F(x_{t-1})},u\right\rangle +\frac{1}{2}\gamma_{t-1}\left\Vert u-z_{t-1}\right\Vert ^{2}+\frac{1}{2}\left(\gamma_{t}-\gamma_{t-1}\right)\left\Vert u-x_{t}\right\Vert ^{2}
\]
Since $x_{t}$ is the minimizer of both $\left\langle \widehat{F(x_{t-1})},u\right\rangle +\frac{1}{2}\gamma_{t-1}\left\Vert u-z_{t-1}\right\Vert ^{2}$
and $\frac{1}{2}\left(\gamma_{t}-\gamma_{t-1}\right)\left\Vert u-x_{t}\right\Vert ^{2}$,
we have
\[
x_{t}=\arg\min_{u\in\dom}\phi_{t}(u)
\]
Moreover
\[
z_{t}=\arg\min_{u\in\dom}\left\{ \phi_{t}(u)+\left\langle \widehat{F(x_{t})}-\widehat{F(x_{t-1})},u\right\rangle \right\} 
\]
By Lemma \ref{lem:danskin}, for all $v$, we have
\[
\nabla\phi_{t}^{*}(v)=\arg\min_{u\in\dom}\left\{ \phi_{t}(u)-\left\langle u,v\right\rangle \right\} 
\]
Thus
\begin{align*}
x_{t} & =\nabla\phi_{t}^{*}(0)\\
z_{t} & =\nabla\phi_{t}^{*}\left(-\left(\widehat{F(x_{t})}-\widehat{F(x_{t-1})}\right)\right)
\end{align*}
Since $\phi_{t}$ is $\gamma_{t}$-strongly convex, Lemma \ref{lem:duality}
implies that $\phi_{t}^{*}$ is $\frac{1}{\gamma_{t}}$-smooth. Thus
\begin{align*}
\left\Vert x_{t}-z_{t}\right\Vert  & =\left\Vert \nabla\phi_{t}^{*}(0)-\nabla\phi_{t}^{*}\left(-\left(\widehat{F(x_{t})}-\widehat{F(x_{t-1})}\right)\right)\right\Vert \\
 & \leq\frac{1}{\gamma_{t}}\left\Vert \widehat{F(x_{t})}-\widehat{F(x_{t-1})}\right\Vert 
\end{align*}
Using Cauchy-Schwartz and the above inequality, we obtain
\begin{align*}
\left\langle \widehat{F(x_{t})}-\widehat{F(x_{t-1})},x_{t}-z_{t}\right\rangle  & \leq\left\Vert \widehat{F(x_{t})}-\widehat{F(x_{t-1})}\right\Vert \left\Vert x_{t}-z_{t}\right\Vert \\
 & \leq\frac{1}{\gamma_{t}}\left\Vert \widehat{F(x_{t})}-\widehat{F(x_{t-1})}\right\Vert ^{2}
\end{align*}
as needed.

\end{proof}

We now combine (\ref{eq:regret-split}) with Lemmas \ref{lem:regret-first-term},
\ref{lem:regret-third-term}, \ref{lem:regret-second-term}. By summing
up over all iterations and telescoping the sums appropriately, we
obtain:
\begin{lem}
\label{lem:regret-combined}Let $R\geq\max_{x,y\in\dom}\left\Vert x-y\right\Vert $.
For all $y\in\dom$, we have
\begin{align*}
 & \sum_{t=1}^{T}\left\langle \widehat{F(x_{t})},x_{t}-y\right\rangle \\
 & \leq\frac{1}{2}R^{2}\gamma_{0}+\left(\frac{1}{2}\frac{R^{2}}{\eta}+2\eta\right)\sqrt{\sum_{t=1}^{T}\left\Vert \widehat{F(x_{t})}-\widehat{F(x_{t-1})}\right\Vert ^{2}}\\
 & -\frac{1}{2}\sum_{t=1}^{T}\gamma_{t-1}\left(\left\Vert x_{t}-z_{t-1}\right\Vert ^{2}+\left\Vert x_{t-1}-z_{t-1}\right\Vert ^{2}\right)
\end{align*}
\end{lem}
\begin{proof}
By plugging in the guarantees provided by Lemmas \ref{lem:regret-first-term},
\ref{lem:regret-third-term}, \ref{lem:regret-second-term} into (\ref{eq:regret-split}),
we obtain
\begin{align*}
\left\langle \widehat{F(x_{t})},x_{t}-y\right\rangle  & \leq(\gamma_{t}-\gamma_{t-1})\frac{1}{2}\left\Vert x_{t}-y\right\Vert ^{2}+\gamma_{t-1}\frac{1}{2}\left\Vert z_{t-1}-y\right\Vert ^{2}-\gamma_{t}\frac{1}{2}\left\Vert z_{t}-y\right\Vert ^{2}\\
 & -\gamma_{t-1}\frac{1}{2}\left\Vert x_{t}-z_{t-1}\right\Vert ^{2}-\gamma_{t}\frac{1}{2}\left\Vert x_{t}-z_{t}\right\Vert ^{2}+\frac{1}{\gamma_{t}}\left\Vert \widehat{F(x_{t})}-\widehat{F(x_{t-1})}\right\Vert ^{2}
\end{align*}
Summing up over all iterations, we obtain
\begin{align*}
 & \sum_{t=1}^{T}\left\langle \widehat{F(x_{t})},x_{t}-y\right\rangle \\
 & \leq\sum_{t=1}^{T}\left(\gamma_{t}-\gamma_{t-1}\right)\frac{1}{2}\left\Vert x_{t}-y\right\Vert ^{2}+\sum_{t=1}^{T}\left(\gamma_{t-1}\frac{1}{2}\left\Vert z_{t-1}-y\right\Vert ^{2}-\gamma_{t}\frac{1}{2}\left\Vert z_{t}-y\right\Vert ^{2}\right)\\
 & +\sum_{t=1}^{T}\frac{1}{\gamma_{t}}\left\Vert \widehat{F(x_{t})}-\widehat{F(x_{t-1})}\right\Vert ^{2}-\sum_{t=1}^{T}\gamma_{t-1}\frac{1}{2}\left\Vert x_{t}-z_{t-1}\right\Vert ^{2}-\sum_{t=1}^{T}\gamma_{t}\frac{1}{2}\left\Vert x_{t}-z_{t}\right\Vert ^{2}
\end{align*}
Note that the second sum naturally telescope. We further upper bound$\left\Vert x_{t}-y\right\Vert ^{2}\leq R^{2}$,
so that the first sum also telescopes. Thus we obtain

\begin{align}
 & \sum_{t=1}^{T}\left\langle \widehat{F(x_{t})},x_{t}-y\right\rangle \nonumber \\
 & \leq\sum_{t=1}^{T}\left(\gamma_{t}-\gamma_{t-1}\right)\frac{1}{2}\underbrace{\left\Vert x_{t}-y\right\Vert ^{2}}_{\leq R^{2}}+\sum_{t=1}^{T}\left(\gamma_{t-1}\frac{1}{2}\left\Vert z_{t-1}-y\right\Vert ^{2}-\gamma_{t}\frac{1}{2}\left\Vert z_{t}-y\right\Vert ^{2}\right)\nonumber \\
 & +\sum_{t=1}^{T}\frac{1}{\gamma_{t}}\left\Vert \widehat{F(x_{t})}-\widehat{F(x_{t-1})}\right\Vert ^{2}-\frac{1}{2}\sum_{t=1}^{T}\gamma_{t-1}\left\Vert x_{t}-z_{t-1}\right\Vert ^{2}-\sum_{t=1}^{T}\frac{1}{2}\gamma_{t}\left\Vert x_{t}-z_{t}\right\Vert ^{2}\nonumber \\
 & \leq\frac{1}{2}R^{2}\left(\gamma_{T}-\gamma_{0}\right)+\frac{1}{2}\gamma_{0}\underbrace{\left\Vert z_{0}-y\right\Vert ^{2}}_{\leq R^{2}}-\frac{1}{2}\gamma_{T}\left\Vert z_{T}-y\right\Vert ^{2}\nonumber \\
 & +\sum_{t=1}^{T}\frac{1}{\gamma_{t}}\left\Vert \widehat{F(x_{t})}-\widehat{F(x_{t-1})}\right\Vert ^{2}-\frac{1}{2}\sum_{t=1}^{T}\gamma_{t-1}\left\Vert x_{t}-z_{t-1}\right\Vert ^{2}-\sum_{t=1}^{T}\frac{1}{2}\gamma_{t}\left\Vert x_{t}-z_{t}\right\Vert ^{2}\nonumber \\
 & \leq\frac{1}{2}R^{2}\gamma_{T}+\sum_{t=1}^{T}\frac{1}{\gamma_{t}}\left\Vert \widehat{F(x_{t})}-\widehat{F(x_{t-1})}\right\Vert ^{2}-\frac{1}{2}\sum_{t=1}^{T}\gamma_{t-1}\left\Vert x_{t}-z_{t-1}\right\Vert ^{2}-\frac{1}{2}\sum_{t=1}^{T}\gamma_{t}\left\Vert x_{t}-z_{t}\right\Vert ^{2}\nonumber \\
 & =\frac{1}{2}R^{2}\gamma_{T}+\sum_{t=1}^{T}\frac{1}{\gamma_{t}}\left\Vert \widehat{F(x_{t})}-\widehat{F(x_{t-1})}\right\Vert ^{2}-\frac{1}{2}\sum_{t=1}^{T}\gamma_{t-1}\left(\left\Vert x_{t}-z_{t-1}\right\Vert ^{2}+\left\Vert x_{t-1}-z_{t-1}\right\Vert ^{2}\right)\nonumber \\
 & -\frac{1}{2}\gamma_{T}\left\Vert x_{T}-z_{T}\right\Vert ^{2}+\frac{1}{2}\gamma_{0}\underbrace{\left\Vert x_{0}-z_{0}\right\Vert ^{2}}_{=0}\nonumber \\
 & \leq\frac{1}{2}R^{2}\gamma_{T}+\sum_{t=1}^{T}\frac{1}{\gamma_{t}}\left\Vert \widehat{F(x_{t})}-\widehat{F(x_{t-1})}\right\Vert ^{2}-\frac{1}{2}\sum_{t=1}^{T}\gamma_{t-1}\left(\left\Vert x_{t}-z_{t-1}\right\Vert ^{2}+\left\Vert x_{t-1}-z_{t-1}\right\Vert ^{2}\right)\label{eq:combined1}
\end{align}
The definition of the step sizes together with Lemma \ref{lem:ineq}
allows us to show that the first sum above is proportional to the
final step size. More precisely, we apply Lemma \ref{lem:ineq} with
$a_{t}=\left\Vert \widehat{F(x_{t})}-\widehat{F(x_{t-1})}\right\Vert ^{2}$
and obtain
\begin{align*}
\sum_{t=1}^{T}\frac{1}{\gamma_{t}}\left\Vert \widehat{F(x_{t})}-\widehat{F(x_{t-1})}\right\Vert ^{2} & =\eta\sum_{t=1}^{T}\frac{\left\Vert \widehat{F(x_{t})}-\widehat{F(x_{t-1})}\right\Vert ^{2}}{\sqrt{\eta^{2}\gamma_{0}^{2}+\sum_{s=1}^{t}\left\Vert \widehat{F(x_{s})}-\widehat{F(x_{s-1})}\right\Vert ^{2}}}\\
 & \leq\eta\sum_{t=1}^{T}\frac{\left\Vert \widehat{F(x_{t})}-\widehat{F(x_{t-1})}\right\Vert ^{2}}{\sqrt{\sum_{s=1}^{t}\left\Vert \widehat{F(x_{s})}-\widehat{F(x_{s-1})}\right\Vert ^{2}}}\\
 & \leq2\eta\sqrt{\sum_{t=1}^{T}\left\Vert \widehat{F(x_{t})}-\widehat{F(x_{t-1})}\right\Vert ^{2}}
\end{align*}
 Additionally, we have
\begin{align*}
\frac{1}{2}R^{2}\gamma_{T} & =\frac{1}{2}\frac{R^{2}}{\eta}\sqrt{\eta^{2}\gamma_{0}^{2}+\sum_{t=1}^{T}\left\Vert \widehat{F(x_{t})}-\widehat{F(x_{t-1})}\right\Vert ^{2}}\\
 & \leq\frac{1}{2}\frac{R^{2}}{\eta}\left(\eta\gamma_{0}+\sqrt{\sum_{t=1}^{T}\left\Vert \widehat{F(x_{t})}-\widehat{F(x_{t-1})}\right\Vert ^{2}}\right)\\
 & =\frac{1}{2}R^{2}\gamma_{0}+\frac{1}{2}\frac{R^{2}}{\eta}\sqrt{\sum_{t=1}^{T}\left\Vert \widehat{F(x_{t})}-\widehat{F(x_{t-1})}\right\Vert ^{2}}
\end{align*}
We plug in the last two inequalities into (\ref{eq:combined1}) and
obtain
\begin{align*}
 & \sum_{t=1}^{T}\left\langle \widehat{F(x_{t})},x_{t}-y\right\rangle \\
 & \leq\frac{1}{2}R^{2}\gamma_{0}+\left(\frac{1}{2}\frac{R^{2}}{\eta}+2\eta\right)\sqrt{\sum_{t=1}^{T}\left\Vert \widehat{F(x_{t})}-\widehat{F(x_{t-1})}\right\Vert ^{2}}-\frac{1}{2}\sum_{t=1}^{T}\gamma_{t-1}\left(\left\Vert x_{t}-z_{t-1}\right\Vert ^{2}+\left\Vert x_{t-1}-z_{t-1}\right\Vert ^{2}\right)
\end{align*}
as needed.
\end{proof}

By plugging in Lemma \ref{lem:regret-combined} into Lemma \ref{lem:error-fn-ub},
we obtain the following upper bound on the error function.
\begin{lem}
\label{lem:error-fn-ub-refined} Let $R\geq\max_{x,y\in\dom}\left\Vert x-y\right\Vert $.
Let $\xi_{t}:=F(x_{t})-\widehat{F(x_{t})}$. We have
\begin{align*}
T\cdot\err(\avx_{T}) & \leq\underbrace{\left(\frac{1}{2}\frac{R^{2}}{\eta}+2\eta\right)\sqrt{\sum_{t=1}^{T}\left\Vert \widehat{F(x_{t})}-\widehat{F(x_{t-1})}\right\Vert ^{2}}}_{\text{loss}}\\
 & -\underbrace{\frac{1}{2}\sum_{t=1}^{T}\gamma_{t-1}\left(\left\Vert x_{t}-z_{t-1}\right\Vert ^{2}+\left\Vert x_{t-1}-z_{t-1}\right\Vert ^{2}\right)}_{\text{gain}}\\
 & +\underbrace{R\left\Vert \sum_{t=1}^{T}\xi_{t}\right\Vert +\sum_{t=1}^{T}\left\langle \xi_{t},x_{t}-x_{0}\right\rangle }_{\text{stochastic error}}\\
 & +\frac{1}{2}R^{2}\gamma_{0}
\end{align*}
\end{lem}

\subsection{Analysis of the loss}

\label{sec:adapeg-loss-analysis}

Here we analyze the loss and gain terms in the upper bound provided
by Lemma \ref{lem:error-fn-ub-refined} above:
\[
\underbrace{\left(\frac{1}{2}\frac{R^{2}}{\eta}+2\eta\right)\sqrt{\sum_{t=1}^{T}\left\Vert \widehat{F(x_{t})}-\widehat{F(x_{t-1})}\right\Vert ^{2}}}_{\text{loss}}-\underbrace{\frac{1}{2}\sum_{t=1}^{T}\gamma_{t-1}\left(\left\Vert x_{t}-z_{t-1}\right\Vert ^{2}+\left\Vert x_{t-1}-z_{t-1}\right\Vert ^{2}\right)}_{\text{gain}}
\]
For non-smooth operators, we ignore the gain term and bound the loss
term using an upper bound $G$ on the norm of the operator, leading
to an upper bound on the net loss of $O\left(G\sqrt{T}\right)$ (for
$\eta=\Theta(R)$) plus an additional stochastic error that we will
analyze in Subsection \ref{sec:adapeg-stoch-error-analysis}. In contrast,
for smooth operators, we crucially use the gain term to balance the
loss term, leading to an upper bound on the net loss of\emph{ $O\left(\beta R^{2}\right)$}
(for $\eta=\Theta(R)$) plus an additional stochastic error.
\begin{lem}
\label{lem:net-loss-nonsmooth}Suppose that $F$ is non-smooth. Let
$G=\max_{x\in\dom}\left\Vert F(x)\right\Vert $ and $\xi_{t}=F(x_{t})-\widehat{F(x_{t})}$.
We have
\begin{align*}
\sqrt{\sum_{t=1}^{T}\left\Vert \widehat{F(x_{t})}-\widehat{F(x_{t-1})}\right\Vert ^{2}} & \leq2\sqrt{2}G\sqrt{T}+2\sqrt{2}\sqrt{\sum_{t=0}^{T}\left\Vert \xi_{t}\right\Vert ^{2}}\\
 & =O\left(G\sqrt{T}+\sqrt{\sum_{t=0}^{T}\left\Vert \xi_{t}\right\Vert ^{2}}\right)
\end{align*}
\end{lem}
\begin{proof}
We have
\begin{align*}
\left\Vert \widehat{F(x_{t})}-\widehat{F(x_{t-1})}\right\Vert ^{2} & =\left\Vert F(x_{t})-\xi_{t}-F(x_{t-1})+\xi_{t-1}\right\Vert ^{2}\\
 & \leq4\left\Vert F(x_{t})\right\Vert ^{2}+4\left\Vert F(x_{t-1})\right\Vert ^{2}+4\left\Vert \xi_{t-1}\right\Vert ^{2}+4\left\Vert \xi_{t}\right\Vert ^{2}\\
 & \leq8G^{2}+4\left\Vert \xi_{t-1}\right\Vert ^{2}+4\left\Vert \xi_{t}\right\Vert ^{2}
\end{align*}
Therefore
\begin{align*}
\sqrt{\sum_{t=1}^{T}\left\Vert \widehat{F(x_{t})}-\widehat{F(x_{t-1})}\right\Vert ^{2}} & \leq\sqrt{8G^{2}T+8\sum_{t=0}^{T}\left\Vert \xi_{t}\right\Vert ^{2}}\\
 & \leq2\sqrt{2}G\sqrt{T}+2\sqrt{2}\sqrt{\sum_{t=0}^{T}\left\Vert \xi_{t}\right\Vert ^{2}}
\end{align*}
 as needed. 
\end{proof}

\begin{lem}
\label{lem:net-loss-smooth} Suppose that $F$ is $\beta$-smooth
with respect to the $\ell_{2}$-norm. Let $R\geq\max_{x,y\in\dom}\left\Vert x-y\right\Vert $
and $\xi_{t}=F(x_{t})-\widehat{F(x_{t})}$. We have
\begin{align*}
 & \left(\frac{1}{2}\frac{R^{2}}{\eta}+2\eta\right)\sqrt{\sum_{t=1}^{T}\left\Vert \widehat{F(x_{t})}-\widehat{F(x_{t-1})}\right\Vert ^{2}}-\frac{1}{2}\sum_{t=1}^{T}\gamma_{t-1}\left(\left\Vert x_{t}-z_{t-1}\right\Vert ^{2}+\left\Vert x_{t-1}-z_{t-1}\right\Vert ^{2}\right)\\
 & \leq\beta\left(\frac{1}{2}\frac{R^{2}}{\eta}+2\eta\right)\left(\left(2+\sqrt{2}\right)R+4\eta\right)+\left(\sqrt{2}\frac{R^{2}}{\eta}+4\sqrt{2}\eta\right)\sqrt{\sum_{t=0}^{T}\left\Vert \xi_{t}\right\Vert ^{2}}\\
 & =O\left(\beta\left(\frac{R^{2}}{\eta}+\eta\right)\left(R+\eta\right)\right)+O\left(\frac{R^{2}}{\eta}+\eta\right)\sqrt{\sum_{t=0}^{T}\left\Vert \xi_{t}\right\Vert ^{2}}
\end{align*}
\end{lem}
\begin{proof}
Recall that we want to upper bound the net loss:
\[
\underbrace{r\sqrt{\sum_{t=1}^{T}\left\Vert \widehat{F(x_{t})}-\widehat{F(x_{t-1})}\right\Vert ^{2}}}_{\text{loss}}-\underbrace{\frac{1}{2}\sum_{t=1}^{T}\gamma_{t-1}\left(\left\Vert x_{t}-z_{t-1}\right\Vert ^{2}+\left\Vert x_{t-1}-z_{t-1}\right\Vert ^{2}\right)}_{\text{gain}}
\]
 where we let $r:=\left(\frac{1}{2}\frac{R^{2}}{\eta}+2\eta\right)$.

The loss is proportional to the stochastic operator value differences,
whereas the gain is proportional to the iterate movement. Our main
approach is to relate the gain to the loss, and show that we can use
the gain to offset the loss. We start by using smoothness and the
inequality $(a+b)^{2}\leq2a^{2}+2b^{2}$, and relate the gain to the
deterministic operator value differences:
\begin{align*}
\left\Vert F(x_{t})-F(x_{t-1})\right\Vert ^{2} & \leq\beta^{2}\left\Vert x_{t}-x_{t-1}\right\Vert ^{2}\\
 & =\beta^{2}\left\Vert x_{t}-z_{t-1}+z_{t-1}-x_{t-1}\right\Vert ^{2}\\
 & \leq2\beta^{2}\left(\left\Vert x_{t}-z_{t-1}\right\Vert ^{2}+\left\Vert x_{t-1}-z_{t-1}\right\Vert ^{2}\right)
\end{align*}
Therefore
\[
\left\Vert x_{t}-z_{t-1}\right\Vert ^{2}+\left\Vert x_{t-1}-z_{t-1}\right\Vert ^{2}\geq\frac{1}{2\beta^{2}}\left\Vert F(x_{t})-F(x_{t-1})\right\Vert ^{2}
\]
 Thus we can upper bound the net loss as follows:

\begin{align}
 & r\sqrt{\sum_{t=1}^{T}\left\Vert \widehat{F(x_{t})}-\widehat{F(x_{t-1})}\right\Vert ^{2}}-\frac{1}{2}\sum_{t=1}^{T}\gamma_{t-1}\left(\left\Vert x_{t}-z_{t-1}\right\Vert ^{2}+\left\Vert x_{t-1}-z_{t-1}\right\Vert ^{2}\right)\nonumber \\
 & \leq r\sqrt{\sum_{t=1}^{T}\left\Vert \widehat{F(x_{t})}-\widehat{F(x_{t-1})}\right\Vert ^{2}}-\frac{1}{2}\sum_{t=1}^{T}\gamma_{t-1}\frac{1}{2\beta^{2}}\left\Vert F(x_{t})-F(x_{t-1})\right\Vert ^{2}\label{eq:net-loss1}
\end{align}
 We now show that after an initial number of iterations, the gain
offsets the loss up to a stochastic error term. Recall that the step
sizes $\gamma_{t}$ are increasing with $t$. Let $\Gamma$ be a value
that we will determine later. Let $\tau$ be the last iteration $t$
such that $\gamma_{t-1}\leq\Gamma$ (we let $\tau=1$ if there is
no such iteration). Note that the definition of $\tau$ implies that
\begin{align*}
\sqrt{\sum_{t=1}^{\tau-1}\left\Vert \widehat{F(x_{t})}-\widehat{F(x_{t-1})}\right\Vert ^{2}} & \leq\sqrt{\eta^{2}\gamma_{0}^{2}+\sum_{t=1}^{\tau-1}\left\Vert \widehat{F(x_{t})}-\widehat{F(x_{t-1})}\right\Vert ^{2}}\\
 & =\eta\gamma_{\tau-1}\\
 & \leq\eta\Gamma
\end{align*}
 and
\[
\gamma_{t-1}\geq\Gamma\quad\forall t\geq\tau+1
\]
Using the above inequalities, we obtain
\begin{align}
 & r\sqrt{\sum_{t=1}^{T}\left\Vert \widehat{F(x_{t})}-\widehat{F(x_{t-1})}\right\Vert ^{2}}-\frac{1}{2}\sum_{t=1}^{T}\gamma_{t-1}\frac{1}{2\beta^{2}}\left\Vert F(x_{t})-F(x_{t-1})\right\Vert ^{2}\nonumber \\
 & \leq r\underbrace{\sqrt{\sum_{t=1}^{\tau-1}\left\Vert \widehat{F(x_{t})}-\widehat{F(x_{t-1})}\right\Vert ^{2}}}_{\leq\eta\Gamma}+r\sqrt{\sum_{t=\tau}^{T}\left\Vert \widehat{F(x_{t})}-\widehat{F(x_{t-1})}\right\Vert ^{2}}\nonumber \\
 & -\frac{1}{2}\sum_{t=\tau+1}^{T}\underbrace{\gamma_{t-1}}_{\geq\Gamma}\frac{1}{2\beta^{2}}\left\Vert F(x_{t})-F(x_{t-1})\right\Vert ^{2}\nonumber \\
 & \leq r\eta\Gamma+r\sqrt{\sum_{t=\tau}^{T}\left\Vert \widehat{F(x_{t})}-\widehat{F(x_{t-1})}\right\Vert ^{2}}-\frac{\Gamma}{4\beta^{2}}\sum_{t=\tau+1}^{T}\left\Vert F(x_{t})-F(x_{t-1})\right\Vert ^{2}\label{eq:net-loss2}
\end{align}
Next, we upper bound the second term above in terms of the deterministic
operator value differences $\left\Vert F(x_{t})-F(x_{t-1})\right\Vert $
and a stochastic error. Using the definition of $\xi_{t}:=F(x_{t})-\widehat{F(x_{t})}$
and the inequality $(a+b)^{2}\leq2a^{2}+2b^{2}$, we obtain
\begin{align*}
\left\Vert \widehat{F(x_{t})}-\widehat{F(x_{t-1})}\right\Vert ^{2} & =\left\Vert \xi_{t-1}-\xi_{t}+F(x_{t})-F(x_{t-1})\right\Vert ^{2}\\
 & \leq2\left\Vert \xi_{t-1}-\xi_{t}\right\Vert ^{2}+2\left\Vert F(x_{t})-F(x_{t-1})\right\Vert ^{2}\\
 & \leq4\left\Vert \xi_{t-1}\right\Vert ^{2}+4\left\Vert \xi_{t}\right\Vert ^{2}+2\left\Vert F(x_{t})-F(x_{t-1})\right\Vert ^{2}
\end{align*}
Therefore
\begin{align*}
\sum_{t=\tau}^{T}\left\Vert \widehat{F(x_{t})}-\widehat{F(x_{t-1})}\right\Vert ^{2} & \leq\sum_{t=\tau}^{T}\left(4\left\Vert \xi_{t-1}\right\Vert ^{2}+4\left\Vert \xi_{t}\right\Vert ^{2}+2\left\Vert F(x_{t})-F(x_{t-1})\right\Vert ^{2}\right)\\
 & \leq8\sum_{t=0}^{T}\left\Vert \xi_{t}\right\Vert ^{2}+2\sum_{t=\tau}^{T}\left\Vert F(x_{t})-F(x_{t-1})\right\Vert ^{2}
\end{align*}
and hence
\begin{align}
 & \sqrt{\sum_{t=\tau}^{T}\left\Vert \widehat{F(x_{t})}-\widehat{F(x_{t-1})}\right\Vert ^{2}}\nonumber \\
 & \leq2\sqrt{2}\sqrt{\sum_{t=0}^{T}\left\Vert \xi_{t}\right\Vert ^{2}}+\sqrt{2}\sqrt{\sum_{t=\tau}^{T}\left\Vert F(x_{t})-F(x_{t-1})\right\Vert ^{2}}\nonumber \\
 & \leq2\sqrt{2}\sqrt{\sum_{t=0}^{T}\left\Vert \xi_{t}\right\Vert ^{2}}+\sqrt{2}\underbrace{\left\Vert F(x_{\tau})-F(x_{\tau-1})\right\Vert }_{\leq\beta\left\Vert x_{\tau}-x_{\tau-1}\right\Vert \leq\beta R}+\sqrt{2}\sqrt{\sum_{t=\tau+1}^{T}\left\Vert F(x_{t})-F(x_{t-1})\right\Vert ^{2}}\nonumber \\
 & \leq2\sqrt{2}\sqrt{\sum_{t=0}^{T}\left\Vert \xi_{t}\right\Vert ^{2}}+\sqrt{2}\beta R+\sqrt{2}\sqrt{\sum_{t=\tau+1}^{T}\left\Vert F(x_{t})-F(x_{t-1})\right\Vert ^{2}}\label{eq:net-loss3}
\end{align}
 Combining (\ref{eq:net-loss1}), (\ref{eq:net-loss2}), (\ref{eq:net-loss3}),
we obtain
\begin{align}
 & r\sqrt{\sum_{t=1}^{T}\left\Vert \widehat{F(x_{t})}-\widehat{F(x_{t-1})}\right\Vert ^{2}}-\frac{1}{2}\sum_{t=1}^{T}\gamma_{t-1}\left(\left\Vert x_{t}-z_{t-1}\right\Vert ^{2}+\left\Vert x_{t-1}-z_{t-1}\right\Vert ^{2}\right)\nonumber \\
 & \leq r\eta\Gamma+2\sqrt{2}r\sqrt{\sum_{t=0}^{T}\left\Vert \xi_{t}\right\Vert ^{2}}+\sqrt{2}\beta rR+\sqrt{2}r\sqrt{\sum_{t=\tau+1}^{T}\left\Vert F(x_{t})-F(x_{t-1})\right\Vert ^{2}}-\frac{\Gamma}{4\beta^{2}}\sum_{t=\tau+1}^{T}\left\Vert F(x_{t})-F(x_{t-1})\right\Vert ^{2}\nonumber \\
 & \leq r\eta\Gamma+2\sqrt{2}r\sqrt{\sum_{t=0}^{T}\left\Vert \xi_{t}\right\Vert ^{2}}+\sqrt{2}\beta rR+\max_{y\geq0}\left\{ \sqrt{2}ry-\frac{\Gamma}{4\beta^{2}}y^{2}\right\} \nonumber \\
 & =r\eta\Gamma+2\sqrt{2}r\sqrt{\sum_{t=0}^{T}\left\Vert \xi_{t}\right\Vert ^{2}}+\sqrt{2}\beta rR+\frac{2r^{2}\beta^{2}}{\Gamma}\label{eq:net-loss4}
\end{align}
On the last line, we used the fact that the function $\phi(y)=ay-by^{2}$,
where $a,b>0$ are positive constants, is a concave function and it
is maximized at $y^{*}=\frac{a}{2b}$ and $\phi(y^{*})=\frac{a^{2}}{4b}$.

Finally, we choose $\Gamma$ in order to balance the terms:
\[
\Gamma=\beta\sqrt{\frac{2r}{\eta}}
\]
We now plug in the above choice of $\Gamma$ into (\ref{eq:net-loss4})
and recall that $r=\frac{1}{2}\frac{R^{2}}{\eta}+2\eta$. We obtain
\begin{align*}
 & r\sqrt{\sum_{t=1}^{T}\left\Vert \widehat{F(x_{t})}-\widehat{F(x_{t-1})}\right\Vert ^{2}}-\frac{1}{2}\sum_{t=1}^{T}\gamma_{t-1}\left(\left\Vert x_{t}-z_{t-1}\right\Vert ^{2}+\left\Vert x_{t-1}-z_{t-1}\right\Vert ^{2}\right)\\
 & \leq\beta r\left(2\sqrt{2}\sqrt{r\eta}+\sqrt{2}R\right)+2\sqrt{2}r\sqrt{\sum_{t=0}^{T}\left\Vert \xi_{t}\right\Vert ^{2}}\\
 & =\beta\left(\frac{1}{2}\frac{R^{2}}{\eta}+2\eta\right)\left(\sqrt{4R^{2}+16\eta^{2}}+\sqrt{2}R\right)+\left(\sqrt{2}\frac{R^{2}}{\eta}+4\sqrt{2}\eta\right)\sqrt{\sum_{t=0}^{T}\left\Vert \xi_{t}\right\Vert ^{2}}\\
 & \leq\beta\left(\frac{1}{2}\frac{R^{2}}{\eta}+2\eta\right)\left(\left(2+\sqrt{2}\right)R+4\eta\right)+\left(\sqrt{2}\frac{R^{2}}{\eta}+4\sqrt{2}\eta\right)\sqrt{\sum_{t=0}^{T}\left\Vert \xi_{t}\right\Vert ^{2}}
\end{align*}
 as needed. 
\end{proof}

\subsection{Analysis of the expected stochastic error}

\label{sec:adapeg-stoch-error-analysis}

Next, we analyze the stochastic error terms in Lemmas \ref{lem:error-fn-ub-refined},
\ref{lem:net-loss-nonsmooth}, \ref{lem:net-loss-smooth}:
\[
\underbrace{\sum_{t=1}^{T}\left\langle \xi_{t},x_{t}-x_{0}\right\rangle \quad\sqrt{\sum_{t=0}^{T}\left\Vert \xi_{t}\right\Vert ^{2}}\quad\left\Vert \sum_{t=1}^{T}\xi_{t}\right\Vert }_{\text{stochastic error terms}}
\]
We consider each of the terms in turn, and upper bound their expected
value. Using the martingale assumption (\ref{eq:stoch-assumption-unbiased}),
we obtain:
\begin{lem}
\label{lem:stoch-err1}For all $t\geq0$, we have
\[
\E\left[\left\langle \xi_{t},x_{t}-x_{0}\right\rangle \right]=0
\]
and thus
\[
\E\left[\sum_{t=1}^{T}\left\langle \xi_{t},x_{t}-x_{0}\right\rangle \right]=\sum_{t=1}^{T}\E\left[\left\langle \xi_{t},x_{t}-x_{0}\right\rangle \right]=0
\]
\end{lem}
\begin{proof}
By the martingale assumption (\ref{eq:stoch-assumption-unbiased}),
we have
\[
\E\left[\left\langle \xi_{t},x_{t}-x_{0}\right\rangle \vert x_{1},\dots,x_{t}\right]=0
\]
Taking expectation over $x_{1},\dots,x_{t}$, we obtain
\[
\E\left[\left\langle \xi_{t},x_{t}-x_{0}\right\rangle \right]=0
\]
as needed.
\end{proof}

Using concavity of the square root and the variance assumption (\ref{eq:stoch-assumption-variance}),
we obtain:
\begin{lem}
\label{lem:stoch-err2} We have
\[
\E\left[\sqrt{\sum_{t=0}^{T}\left\Vert \xi_{t}\right\Vert ^{2}}\right]\le\sqrt{\sum_{t=0}^{T}\E\left[\left\Vert \xi_{t}\right\Vert ^{2}\right]}\leq\sigma\sqrt{T+1}
\]
\end{lem}
\begin{proof}
The first inequality follows from concavity of the square root and
the second inequality follows from the variance assumption (\ref{eq:stoch-assumption-variance}).
\end{proof}

Finally, we analyze the remaining term. The first inequality in the
lemma below is due to non-negativity of variance, applied to the random
variable $Z:=\left\Vert \sum_{t=1}^{T}\xi_{t}\right\Vert $. The equality
follows from the martingale assumption (\ref{eq:stoch-assumption-unbiased}).
The second inequality follows from the variance assumption (\ref{eq:stoch-assumption-variance}),
as before.
\begin{lem}
\label{lem:stoch-err3}We have
\[
\E\left[\left\Vert \sum_{t=1}^{T}\xi_{t}\right\Vert \right]\leq\sqrt{\E\left[\left\Vert \sum_{t=1}^{T}\xi_{t}\right\Vert ^{2}\right]}=\sqrt{\sum_{t=1}^{T}\E\left[\left\Vert \xi_{t}\right\Vert ^{2}\right]}\leq\sigma\sqrt{T}
\]
\end{lem}
\begin{proof}
The first inequality follows from the non-negativity of variance.
To see this, consider the random variable $Z:=\left\Vert \sum_{t=1}^{T}\xi_{t}\right\Vert $.
We have
\[
0\leq\mathrm{Var}(Z)=\E\left[Z^{2}\right]-\left(\E\left[Z\right]\right)^{2}
\]
By rearranging and taking the square root, we obtain
\[
\E\left[Z\right]\leq\sqrt{\E\left[Z^{2}\right]}
\]
Using the martingale assumption (\ref{eq:stoch-assumption-unbiased}),
we can verify by induction on $T$ that
\[
\E\left[\left\Vert \sum_{t=1}^{T}\xi_{t}\right\Vert ^{2}\right]=\sum_{t=1}^{T}\E\left[\left\Vert \xi_{t}\right\Vert ^{2}\right]
\]
The base case $T=1$ is immediate. Therefore we may assume that $T>1$.
By the inductive hypothesis, we have
\[
\E\left[\left\Vert \sum_{t=1}^{T-1}\xi_{t}\right\Vert ^{2}\right]=\sum_{t=1}^{T-1}\E\left[\left\Vert \xi_{t}\right\Vert ^{2}\right]
\]
By the martingale assumption (\ref{eq:stoch-assumption-unbiased}),
we have
\[
\E\left[\left\langle \xi_{T},\sum_{t=1}^{T-1}\xi_{t}\right\rangle \vert x_{1},\dots,x_{T-1}\right]=0
\]
Thus
\begin{align*}
\E\left[\left\Vert \sum_{t=1}^{T}\xi_{t}\right\Vert ^{2}\right] & =\E\left[\left\Vert \sum_{t=1}^{T-1}\xi_{t}\right\Vert ^{2}\right]+\E\left[\left\Vert \xi_{T}\right\Vert ^{2}\right]+2\E\left[\left\langle \xi_{T},\sum_{t=1}^{T-1}\xi_{t}\right\rangle \right]=\sum_{t=1}^{T}\E\left[\left\Vert \xi_{t}\right\Vert ^{2}\right]
\end{align*}
The last inequality in the lemma statement follows from the variance
assumption (\ref{eq:stoch-assumption-variance}).
\end{proof}

\subsection{Putting everything together}

We now put everything together and complete the proof of Theorem \ref{thm:adapeg-convergence}.
\begin{lem}
Suppose that $F$ is non-smooth and let $R\geq\max_{x,y\in\dom}\left\Vert x-y\right\Vert $
and $G=\max_{x\in\dom}\left\Vert F(x)\right\Vert $. We have
\begin{align*}
\E\left[\err(\avx_{T})\right] & \leq\frac{1}{T}\left(\frac{1}{2}\gamma_{0}R^{2}+\left(\sqrt{2}\frac{R^{2}}{\eta}+4\sqrt{2}\eta+R\right)\left(G+\sigma\right)\sqrt{T+1}\right)\\
 & =O\left(\frac{\gamma_{0}R^{2}}{T}+\frac{\left(\frac{R^{2}}{\eta}+\eta+R\right)\left(G+\sigma\right)}{\sqrt{T}}\right)
\end{align*}
 Setting $\eta=\Theta(R)$ gives
\[
\E\left[\err(\avx_{T})\right]\leq O\left(\frac{\gamma_{0}R^{2}}{T}+\frac{R\left(G+\sigma\right)}{\sqrt{T}}\right)
\]
\end{lem}
\begin{proof}
The lemma follows by combining Lemmas \ref{lem:error-fn-ub-refined}
, \ref{lem:net-loss-nonsmooth}, \ref{lem:stoch-err1}, \ref{lem:stoch-err2},
\ref{lem:stoch-err3}.
\end{proof}
\begin{lem}
Suppose that $F$ is $\beta$-smooth with respect to the $\ell_{2}$-norm.
Let $R\geq\max_{x,y\in\dom}\left\Vert x-y\right\Vert $. We have
\begin{align*}
 & \E\left[\err(\avx_{T})\right]\\
 & \leq\frac{1}{T}\left(\frac{1}{2}\gamma_{0}R^{2}+\beta\left(\frac{1}{2}\frac{R^{2}}{\eta}+2\eta\right)\left(\left(2+\sqrt{2}\right)R+4\eta\right)+\left(\sqrt{2}\frac{R^{2}}{\eta}+4\sqrt{2}\eta+R\right)\sigma\sqrt{T+1}\right)\\
 & =O\left(\frac{\gamma_{0}R^{2}+\beta\left(\frac{R^{2}}{\eta}+\eta\right)\left(R+\eta\right)}{T}+\frac{\left(\frac{R^{2}}{\eta}+\eta+R\right)\sigma}{\sqrt{T}}\right)
\end{align*}
 Setting $\eta=\Theta(R)$ gives
\[
\E\left[\err(\avx_{T})\right]\leq O\left(\frac{\left(\beta+\gamma_{0}\right)R^{2}}{T}+\frac{R\sigma}{\sqrt{T}}\right)
\]
\end{lem}
\begin{proof}
The lemma follows by combining Lemmas \ref{lem:error-fn-ub-refined}
, \ref{lem:net-loss-smooth}, \ref{lem:stoch-err1}, \ref{lem:stoch-err2},
\ref{lem:stoch-err3}.
\end{proof}

\section{Analysis of algorithm \ref{alg:adapeg-unbounded}}

\label{sec:adapeg-unbounded-analysis}

Throughout this section, we let $\xi_{t}:=F(x_{t})-\widehat{F(x_{t})}$
and $G:=\max_{x\in\dom}\left\Vert F(x)\right\Vert $. As noted in
Section \ref{sec:prelim}, we analyze convergence via the restricted
error function. The starting point of our analysis is to upper bound
the restricted error function in terms of the stochastic regret. Using
the definition of the error function (\ref{eq:restricted-error-fn}),
the definition of $\avx_{T}=\frac{1}{T}\sum_{t=1}^{T}x_{t}$, and
the monotonicity of $F$ (\ref{eq:monotone-operator}), we obtain:
\begin{lem}
\label{lem:error-fn-ub-1} Let $D>0$ be any fixed positive value.
We have
\[
T\cdot\err_{D}(\avx_{T})\leq\underbrace{\sup_{y\in\dom\colon\left\Vert x_{0}-y\right\Vert \leq D}\sum_{t=1}^{T}\left\langle \widehat{F(x_{t})},x_{t}-y\right\rangle }_{\text{stochastic regret}}+\underbrace{\left\Vert \sum_{t=1}^{T}\xi_{t}\right\Vert D+\sum_{t=1}^{T}\left\langle \xi_{t},x_{t}-x_{0}\right\rangle }_{\text{stochastic error}}
\]
\end{lem}
\begin{proof}
Using the definition of the restricted error function (\ref{eq:restricted-error-fn}),
the definition of $\avx_{T}=\frac{1}{T}\sum_{t=1}^{T}x_{t}$, and
the monotonicity of $F$ (\ref{eq:monotone-operator}), we obtain
\begin{align*}
\err(\avx_{T}) & =\sup_{y\in\dom\colon\left\Vert x_{0}-y\right\Vert \leq D}\left\langle F(y),\avx_{T}-y\right\rangle \\
 & =\frac{1}{T}\sup_{y\in\dom\colon\left\Vert x_{0}-y\right\Vert \leq D}\left(\sum_{t=1}^{T}\left\langle F(y),x_{t}-y\right\rangle \right)\\
 & \leq\frac{1}{T}\sup_{y\in\dom\colon\left\Vert x_{0}-y\right\Vert \leq D}\left(\sum_{t=1}^{T}\left\langle F(x_{t}),x_{t}-y\right\rangle \right)
\end{align*}
We further write
\begin{align*}
\left\langle F(x_{t}),x_{t}-y\right\rangle  & =\left\langle \widehat{F(x_{t})},x_{t}-y\right\rangle +\left\langle F(x_{t})-\widehat{F(x_{t})},x_{t}-y\right\rangle \\
 & =\left\langle \widehat{F(x_{t})},x_{t}-y\right\rangle +\left\langle F(x_{t})-\widehat{F(x_{t})},x_{0}-y\right\rangle +\left\langle F(x_{t})-\widehat{F(x_{t})},x_{t}-x_{0}\right\rangle \\
 & =\left\langle \widehat{F(x_{t})},x_{t}-y\right\rangle +\left\langle \xi_{t},x_{0}-y\right\rangle +\left\langle \xi_{t},x_{t}-x_{0}\right\rangle 
\end{align*}
where we let $\xi_{t}:=F(x_{t})-\widehat{F(x_{t})}$. Thus we obtain
\begin{align*}
T\cdot\err(\avx_{T}) & \leq\sup_{y\in\dom\colon\left\Vert x_{0}-y\right\Vert \leq D}\left(\sum_{t=1}^{T}\left\langle \widehat{F(x_{t})},x_{t}-y\right\rangle +\left\langle \sum_{t=1}^{T}\xi_{t},x_{0}-y\right\rangle +\sum_{t=1}^{T}\left\langle \xi_{t},x_{t}-x_{0}\right\rangle \right)\\
 & \leq\sup_{y\in\dom\colon\left\Vert x_{0}-y\right\Vert \leq D}\left(\sum_{t=1}^{T}\left\langle \widehat{F(x_{t})},x_{t}-y\right\rangle \right)\\
 & +\sup_{y\in\dom\colon\left\Vert x_{0}-y\right\Vert \leq D}\left(\left\langle \sum_{t=1}^{T}\xi_{t},x_{0}-y\right\rangle \right)+\sum_{t=1}^{T}\left\langle \xi_{t},x_{t}-x_{0}\right\rangle 
\end{align*}
Using the Cauchy-Schwartz inequality, we obtain the following upper
bound on the second term above:

\[
\left\langle \sum_{t=1}^{T}\xi_{t},x_{0}-y\right\rangle \leq\left\Vert \sum_{t=1}^{T}\xi_{t}\right\Vert \left\Vert x_{0}-y\right\Vert 
\]
Therefore
\[
\err_{D}(\avx_{T})\leq\frac{1}{T}\left(\sup_{y\in\dom\colon\left\Vert x_{0}-y\right\Vert \leq D}\left(\sum_{t=1}^{T}\left\langle \widehat{F(x_{t})},x_{t}-y\right\rangle \right)+D\left\Vert \sum_{t=1}^{T}\xi_{t}\right\Vert +\sum_{t=1}^{T}\left\langle \xi_{t},x_{t}-x_{0}\right\rangle \right)
\]
as needed.
\end{proof}

Next, we analyze each of the two terms in Lemma \ref{lem:error-fn-ub-1}
in turn. 

\subsection{Analysis of the stochastic regret}

Here we analyze the stochastic regret in Lemma \ref{lem:error-fn-ub-1}:
\[
\underbrace{\sup_{y\in\dom\colon\left\Vert x_{0}-y\right\Vert \leq D}\sum_{t=1}^{T}\left\langle \widehat{F(x_{t})},x_{t}-y\right\rangle }_{\text{stochastic regret}}
\]
We fix an arbitrary point $y\in\dom$, and we analyze the stochastic
regret $\sum_{t=1}^{T}\left\langle \widehat{F(x_{t})},x_{t}-y\right\rangle $.
We split the inner product $\left\langle \widehat{F(x_{t})},x_{t}-y\right\rangle $
as follows:
\begin{equation}
\left\langle \widehat{F(x_{t})},x_{t}-y\right\rangle =\left\langle \widehat{F(x_{t})},z_{t}-y\right\rangle +\left\langle \widehat{F(x_{t})}-\widehat{F(x_{t-1})},x_{t}-z_{t}\right\rangle +\left\langle \widehat{F(x_{t-1})},x_{t}-z_{t}\right\rangle \label{eq:regret-split-1}
\end{equation}

We upper bound each term in (\ref{eq:regret-split-1}) in turn. For
the first term, we apply the optimality condition for $z_{t}$ and
obtain:
\begin{lem}
\label{lem:regret-first-term-1}For any $y\in\dom$, we have
\begin{align*}
\left\langle \widehat{F(x_{t})},z_{t}-y\right\rangle  & \leq\frac{1}{2}\left(\gamma_{t-1}-\gamma_{t-2}\right)\left\Vert x_{0}-y\right\Vert ^{2}+\frac{1}{2}\gamma_{t-2}\left\Vert z_{t-1}-y\right\Vert ^{2}-\frac{1}{2}\gamma_{t-1}\left\Vert z_{t}-y\right\Vert ^{2}\\
 & -\frac{1}{2}\gamma_{t-2}\left\Vert z_{t-1}-z_{t}\right\Vert ^{2}-\frac{1}{2}\left(\gamma_{t-1}-\gamma_{t-2}\right)\left\Vert x_{0}-z_{t}\right\Vert ^{2}
\end{align*}
\end{lem}
\begin{proof}
By the optimality condition for $z_{t}$, for all $y\in\dom$, we
have
\begin{align*}
\left\langle \widehat{F(x_{t})}+\gamma_{t-2}\left(z_{t}-z_{t-1}\right)+\left(\gamma_{t-1}-\gamma_{t-2}\right)\left(z_{t}-x_{0}\right),z_{t}-y\right\rangle \leq0
\end{align*}
By rearranging the above inequality and using the identity $ab=\frac{1}{2}\left(\left(a+b\right)^{2}-a^{2}-b^{2}\right)$,
we obtain
\begin{align*}
 & \left\langle \widehat{F(x_{t})},z_{t}-y\right\rangle \\
 & \le\gamma_{t-2}\left\langle z_{t-1}-z_{t},z_{t}-y\right\rangle +\left(\gamma_{t-1}-\gamma_{t-2}\right)\left\langle x_{0}-z_{t},z_{t}-y\right\rangle \\
 & =\frac{1}{2}\gamma_{t-2}\left\Vert z_{t-1}-y\right\Vert ^{2}-\frac{1}{2}\gamma_{t-2}\left\Vert z_{t}-y\right\Vert ^{2}-\frac{1}{2}\gamma_{t-2}\left\Vert z_{t-1}-z_{t}\right\Vert ^{2}\\
 & +\frac{1}{2}\left(\gamma_{t-1}-\gamma_{t-2}\right)\left\Vert x_{0}-y\right\Vert ^{2}-\frac{1}{2}\left(\gamma_{t-1}-\gamma_{t-2}\right)\left\Vert z_{t}-y\right\Vert ^{2}-\frac{1}{2}\left(\gamma_{t-1}-\gamma_{t-2}\right)\left\Vert x_{0}-z_{t}\right\Vert ^{2}\\
 & =\frac{1}{2}\left(\gamma_{t-1}-\gamma_{t-2}\right)\left\Vert x_{0}-y\right\Vert ^{2}+\frac{1}{2}\gamma_{t-2}\left\Vert z_{t-1}-y\right\Vert ^{2}-\frac{1}{2}\gamma_{t-1}\left\Vert z_{t}-y\right\Vert ^{2}\\
 & -\frac{1}{2}\gamma_{t-2}\left\Vert z_{t-1}-z_{t}\right\Vert ^{2}-\frac{1}{2}\left(\gamma_{t-1}-\gamma_{t-2}\right)\left\Vert x_{0}-z_{t}\right\Vert ^{2}
\end{align*}
as needed.
\end{proof}

For the third term, we apply the optimality condition for $x_{t}$
and obtain:
\begin{lem}
\label{lem:regret-third-term-1} We have
\begin{align*}
\left\langle \widehat{F(x_{t-1})},x_{t}-z_{t}\right\rangle  & \leq\frac{1}{2}\gamma_{t-2}\left\Vert z_{t-1}-z_{t}\right\Vert ^{2}-\frac{1}{2}\gamma_{t-1}\left\Vert x_{t}-z_{t}\right\Vert ^{2}-\frac{1}{2}\gamma_{t-2}\left\Vert z_{t-1}-x_{t}\right\Vert ^{2}\\
 & +\frac{1}{2}\left(\gamma_{t-1}-\gamma_{t-2}\right)\left\Vert x_{0}-z_{t}\right\Vert ^{2}-\frac{1}{2}\left(\gamma_{t-1}-\gamma_{t-2}\right)\left\Vert x_{0}-x_{t}\right\Vert ^{2}
\end{align*}
\end{lem}
\begin{proof}
By the optimality condition for $x_{t}$, we have
\[
\left\langle \widehat{F(x_{t-1})}+\gamma_{t-2}\left(x_{t}-z_{t-1}\right)+\left(\gamma_{t-1}-\gamma_{t-2}\right)\left(x_{t}-x_{0}\right),x_{t}-z_{t}\right\rangle \leq0
\]
By rearranging the above inequality and using the identity $ab=\frac{1}{2}\left(\left(a+b\right)^{2}-a^{2}-b^{2}\right)$,
we obtain
\begin{align*}
 & \left\langle \widehat{F(x_{t-1})},x_{t}-z_{t}\right\rangle \\
 & \leq\gamma_{t-2}\left\langle z_{t-1}-x_{t},x_{t}-z_{t}\right\rangle +\left(\gamma_{t-1}-\gamma_{t-2}\right)\left\langle x_{0}-x_{t},x_{t}-z_{t}\right\rangle \\
 & =\frac{1}{2}\gamma_{t-2}\left\Vert z_{t-1}-z_{t}\right\Vert ^{2}-\frac{1}{2}\gamma_{t-2}\left\Vert x_{t}-z_{t}\right\Vert ^{2}-\frac{1}{2}\gamma_{t-2}\left\Vert z_{t-1}-x_{t}\right\Vert ^{2}\\
 & +\frac{1}{2}\left(\gamma_{t-1}-\gamma_{t-2}\right)\left\Vert x_{0}-z_{t}\right\Vert ^{2}-\frac{1}{2}\left(\gamma_{t-1}-\gamma_{t-2}\right)\left\Vert x_{t}-z_{t}\right\Vert ^{2}-\frac{1}{2}\left(\gamma_{t-1}-\gamma_{t-2}\right)\left\Vert x_{0}-x_{t}\right\Vert ^{2}\\
 & =\frac{1}{2}\gamma_{t-2}\left\Vert z_{t-1}-z_{t}\right\Vert ^{2}-\frac{1}{2}\gamma_{t-1}\left\Vert x_{t}-z_{t}\right\Vert ^{2}-\frac{1}{2}\gamma_{t-2}\left\Vert z_{t-1}-x_{t}\right\Vert ^{2}\\
 & +\frac{1}{2}\left(\gamma_{t-1}-\gamma_{t-2}\right)\left\Vert x_{0}-z_{t}\right\Vert ^{2}-\frac{1}{2}\left(\gamma_{t-1}-\gamma_{t-2}\right)\left\Vert x_{0}-x_{t}\right\Vert ^{2}
\end{align*}
as needed.
\end{proof}

We now analyze the second term. The argument is inspired by the work
of \citet{MohriYang16} for online convex minimization. We make careful
use of the definition of $z_{t}$ and duality and obtain the following
guarantee:
\begin{lem}
\label{lem:regret-second-term-1}We have
\[
\left\langle \widehat{F(x_{t})}-\widehat{F(x_{t-1})},x_{t}-z_{t}\right\rangle \leq\frac{1}{\gamma_{t-1}}\left\Vert \widehat{F(x_{t})}-\widehat{F(x_{t-1})}\right\Vert ^{2}
\]
\end{lem}
\begin{proof}
A key idea is to consider the function $\phi_{t}$, defined below,
and show that $x_{t}$ is a minimizer of $\phi_{t}$ and $z_{t}$
is a minimizer of a function that is closely related to $\phi_{t}$.
These facts together with the strong convexity of $\phi_{t}$ and
duality allow us to relate the distance between the iterates to the
operator values.

Let
\[
\phi_{t}(u)=\left\langle \widehat{F(x_{t-1})},u\right\rangle +\frac{1}{2}\gamma_{t-2}\left\Vert u-z_{t-1}\right\Vert ^{2}+\frac{1}{2}\left(\gamma_{t-1}-\gamma_{t-2}\right)\left\Vert u-x_{0}\right\Vert ^{2}
\]
We have
\begin{align*}
x_{t} & =\arg\min_{u\in\dom}\phi_{t}(u)\\
z_{t} & =\arg\min_{u\in\dom}\left\{ \phi_{t}(u)+\left\langle \widehat{F(x_{t})}-\widehat{F(x_{t-1})},u\right\rangle \right\} 
\end{align*}
By Lemma \ref{lem:danskin}, for all $v$, we have
\[
\nabla\phi_{t}^{*}(v)=\arg\min_{u\in\dom}\left\{ \phi_{t}(u)-\left\langle u,v\right\rangle \right\} 
\]
Thus
\begin{align*}
x_{t} & =\nabla\phi_{t}^{*}(0)\\
z_{t} & =\nabla\phi_{t}^{*}\left(-\left(\widehat{F(x_{t})}-\widehat{F(x_{t-1})}\right)\right)
\end{align*}
Since $\phi_{t}$ is $\gamma_{t-1}$-strongly convex, Lemma \ref{lem:duality}
implies that $\phi_{t}^{*}$ is $\frac{1}{\gamma_{t-1}}$-smooth.
Thus
\begin{align*}
\left\Vert x_{t}-z_{t}\right\Vert  & =\left\Vert \nabla\phi_{t}^{*}(0)-\nabla\phi_{t}^{*}\left(-\left(\widehat{F(x_{t})}-\widehat{F(x_{t-1})}\right)\right)\right\Vert \\
 & \leq\frac{1}{\gamma_{t-1}}\left\Vert \widehat{F(x_{t})}-\widehat{F(x_{t-1})}\right\Vert 
\end{align*}
Using Cauchy-Schwartz and the above inequality, we obtain
\begin{align*}
\left\langle \widehat{F(x_{t})}-\widehat{F(x_{t-1})},x_{t}-z_{t}\right\rangle  & \leq\left\Vert \widehat{F(x_{t})}-\widehat{F(x_{t-1})}\right\Vert \left\Vert x_{t}-z_{t}\right\Vert \\
 & \leq\frac{1}{\gamma_{t-1}}\left\Vert \widehat{F(x_{t})}-\widehat{F(x_{t-1})}\right\Vert ^{2}
\end{align*}
as needed.

\end{proof}

We now combine (\ref{eq:regret-split-1}) with Lemmas \ref{lem:regret-first-term-1},
\ref{lem:regret-third-term-1}, \ref{lem:regret-second-term-1}. As
before, we sum up over all iterations and telescope the sums appropriately.
We note that, in contrast to Algorithm \ref{alg:adapeg}, the step
is now off-by-one. To address this issue, we use the inequality in
Lemma \ref{lem:ineq-off-by-one} and we get an additional stochastic
error term that we will bound in Subsection \ref{subsec:stoch-error-unbounded-domains}.
\begin{lem}
\label{lem:regret-combined-1}For any $y\in\dom$, we have
\begin{align*}
 & \sum_{t=1}^{T}\left\langle \widehat{F(x_{t})},x_{t}-y\right\rangle \\
 & \leq\left(\frac{1}{2}\left\Vert x_{0}-y\right\Vert ^{2}+3\eta^{2}\right)\gamma_{0}+\left(\frac{1}{2}\frac{\left\Vert x_{0}-y\right\Vert ^{2}}{\eta}+6\eta\right)\sqrt{\sum_{t=1}^{T}\left\Vert \widehat{F(x_{t})}-\widehat{F(x_{t-1})}\right\Vert ^{2}}\\
 & -\sum_{t=1}^{T}\frac{1}{2}\gamma_{t-2}\left(\left\Vert x_{t-1}-z_{t-1}\right\Vert ^{2}+\left\Vert z_{t-1}-x_{t}\right\Vert ^{2}\right)+\frac{8}{\gamma_{0}}\left(\max_{0\leq t\leq T}\left\Vert \widehat{F(x_{t})}\right\Vert ^{2}\right)
\end{align*}
\end{lem}
\begin{proof}
By plugging in the guarantees provided by Lemmas \ref{lem:regret-first-term-1},
\ref{lem:regret-third-term-1}, \ref{lem:regret-second-term-1} into
(\ref{eq:regret-split-1}), we obtain

\begin{align*}
\left\langle \widehat{F(x_{t})},x_{t}-y\right\rangle  & \leq\frac{1}{2}\left(\gamma_{t-1}-\gamma_{t-2}\right)\left\Vert x_{0}-y\right\Vert ^{2}+\frac{1}{2}\gamma_{t-2}\left\Vert z_{t-1}-y\right\Vert ^{2}-\frac{1}{2}\gamma_{t-1}\left\Vert z_{t}-y\right\Vert ^{2}\\
 & +\frac{1}{\gamma_{t-1}}\left\Vert \widehat{F(x_{t})}-\widehat{F(x_{t-1})}\right\Vert ^{2}\\
 & -\frac{1}{2}\gamma_{t-1}\left\Vert x_{t}-z_{t}\right\Vert ^{2}-\frac{1}{2}\gamma_{t-2}\left\Vert z_{t-1}-x_{t}\right\Vert ^{2}\\
 & -\frac{1}{2}\left(\gamma_{t-1}-\gamma_{t-2}\right)\left\Vert x_{0}-x_{t}\right\Vert ^{2}
\end{align*}
We drop the last term, which is negative. We sum up over all iterations
and use that $\gamma_{-1}=0$, and obtain
\begin{align}
 & \sum_{t=1}^{T}\left\langle \widehat{F(x_{t})},x_{t}-y\right\rangle \nonumber \\
 & \leq\sum_{t=1}^{T}\frac{1}{2}\left(\gamma_{t-1}-\gamma_{t-2}\right)\left\Vert x_{0}-y\right\Vert ^{2}+\sum_{t=1}^{T}\left(\frac{1}{2}\gamma_{t-2}\left\Vert z_{t-1}-y\right\Vert ^{2}-\frac{1}{2}\gamma_{t-1}\left\Vert z_{t}-y\right\Vert ^{2}\right)\nonumber \\
 & +\sum_{t=1}^{T}\frac{1}{\gamma_{t-1}}\left\Vert \widehat{F(x_{t})}-\widehat{F(x_{t-1})}\right\Vert ^{2}-\sum_{t=1}^{T}\frac{1}{2}\gamma_{t-1}\left\Vert x_{t}-z_{t}\right\Vert ^{2}-\sum_{t=1}^{T}\frac{1}{2}\gamma_{t-2}\left\Vert z_{t-1}-x_{t}\right\Vert ^{2}\nonumber \\
 & =\frac{1}{2}\gamma_{T-1}\left\Vert x_{0}-y\right\Vert ^{2}-\frac{1}{2}\gamma_{T-1}\left\Vert z_{T}-y\right\Vert ^{2}+\sum_{t=1}^{T}\frac{1}{\gamma_{t-1}}\left\Vert \widehat{F(x_{t})}-\widehat{F(x_{t-1})}\right\Vert ^{2}\nonumber \\
 & -\sum_{t=1}^{T}\frac{1}{2}\gamma_{t-2}\left(\left\Vert x_{t-1}-z_{t-1}\right\Vert ^{2}+\left\Vert z_{t-1}-x_{t}\right\Vert ^{2}\right)-\gamma_{T-1}\left\Vert x_{T}-z_{T}\right\Vert ^{2}\nonumber \\
 & \leq\frac{1}{2}\gamma_{T-1}\left\Vert x_{0}-y\right\Vert ^{2}+\sum_{t=1}^{T}\frac{1}{\gamma_{t-1}}\left\Vert \widehat{F(x_{t})}-\widehat{F(x_{t-1})}\right\Vert ^{2}\nonumber \\
 & -\sum_{t=1}^{T}\frac{1}{2}\gamma_{t-2}\left(\left\Vert x_{t-1}-z_{t-1}\right\Vert ^{2}+\left\Vert z_{t-1}-x_{t}\right\Vert ^{2}\right)\label{eq:combined1-1}
\end{align}
The definition of the step sizes together with Lemma \ref{lem:ineq-off-by-one}
allows us to relate the first sum above to the final step size. We
note that, in contrast to the analysis of Algorithm \ref{alg:adapeg},
the sum is now off by one. We let $a_{0}=\eta^{2}\gamma_{0}^{2}$
and $a_{t}=\left\Vert \widehat{F(x_{t})}-\widehat{F(x_{t-1})}\right\Vert ^{2}$.
By Lemma \ref{lem:ineq-off-by-one}, we have
\begin{align*}
 & \sum_{t=1}^{T}\frac{1}{\gamma_{t-1}}\left\Vert \widehat{F(x_{t})}-\widehat{F(x_{t-1})}\right\Vert ^{2}\\
 & =\eta\sum_{t=1}^{T}\frac{\left\Vert \widehat{F(x_{t})}-\widehat{F(x_{t-1})}\right\Vert ^{2}}{\sqrt{\eta^{2}\gamma_{0}^{2}+\sum_{s=1}^{t-1}\left\Vert \widehat{F(x_{s})}-\widehat{F(x_{s-1})}\right\Vert ^{2}}}\\
 & \leq\frac{2\max_{t\in[T]}\left\Vert \widehat{F(x_{t})}-\widehat{F(x_{t-1})}\right\Vert ^{2}}{\gamma_{0}}\\
 & +3\eta\sqrt{\max_{t\in[T]}\left\Vert \widehat{F(x_{t})}-\widehat{F(x_{t-1})}\right\Vert ^{2}}+3\eta\sqrt{\eta^{2}\gamma_{0}^{2}+\sum_{t=1}^{T}\left\Vert \widehat{F(x_{t})}-\widehat{F(x_{t-1})}\right\Vert ^{2}}
\end{align*}
We further bound
\[
\sqrt{\max_{t\in[T]}\left\Vert \widehat{F(x_{t})}-\widehat{F(x_{t-1})}\right\Vert ^{2}}\leq\sqrt{\sum_{t=1}^{T}\left\Vert \widehat{F(x_{t})}-\widehat{F(x_{t-1})}\right\Vert ^{2}}
\]
\[
\sqrt{\eta^{2}\gamma_{0}^{2}+\sum_{t=1}^{T}\left\Vert \widehat{F(x_{t})}-\widehat{F(x_{t-1})}\right\Vert ^{2}}\leq\eta\gamma_{0}+\sqrt{\sum_{t=1}^{T}\left\Vert \widehat{F(x_{t})}-\widehat{F(x_{t-1})}\right\Vert ^{2}}
\]
\[
\left\Vert \widehat{F(x_{t})}-\widehat{F(x_{t-1})}\right\Vert ^{2}\leq2\left\Vert \widehat{F(x_{t})}\right\Vert ^{2}+2\left\Vert \widehat{F(x_{t-1})}\right\Vert ^{2}
\]
 and we obtain
\begin{align*}
 & \sum_{t=1}^{T}\frac{1}{\gamma_{t-1}}\left\Vert \widehat{F(x_{t})}-\widehat{F(x_{t-1})}\right\Vert ^{2}\\
 & \leq6\eta\sqrt{\sum_{t=1}^{T}\left\Vert \widehat{F(x_{t})}-\widehat{F(x_{t-1})}\right\Vert ^{2}}+\frac{8}{\gamma_{0}}\left(\max_{0\leq t\leq T}\left\Vert \widehat{F(x_{t})}\right\Vert ^{2}\right)+3\eta^{2}\gamma_{0}
\end{align*}
 Additionally, we have
\begin{align*}
\frac{1}{2}\gamma_{T-1}\left\Vert x_{0}-y\right\Vert ^{2} & =\frac{1}{2}\frac{\left\Vert x_{0}-y\right\Vert ^{2}}{\eta}\sqrt{\eta^{2}\gamma_{0}^{2}+\sum_{t=1}^{T-1}\left\Vert \widehat{F(x_{t})}-\widehat{F(x_{t-1})}\right\Vert ^{2}}\\
 & \leq\frac{1}{2}\frac{\left\Vert x_{0}-y\right\Vert ^{2}}{\eta}\left(\eta\gamma_{0}+\sqrt{\sum_{t=1}^{T-1}\left\Vert \widehat{F(x_{t})}-\widehat{F(x_{t-1})}\right\Vert ^{2}}\right)\\
 & =\frac{1}{2}\left\Vert x_{0}-y\right\Vert ^{2}\gamma_{0}+\frac{1}{2}\frac{\left\Vert x_{0}-y\right\Vert ^{2}}{\eta}\sqrt{\sum_{t=1}^{T-1}\left\Vert \widehat{F(x_{t})}-\widehat{F(x_{t-1})}\right\Vert ^{2}}
\end{align*}
We plug in the last two inequalities into (\ref{eq:combined1-1})
and obtain
\begin{align*}
 & \sum_{t=1}^{T}\left\langle \widehat{F(x_{t})},x_{t}-y\right\rangle \\
 & \leq\left(\frac{1}{2}\left\Vert x_{0}-y\right\Vert ^{2}+3\eta^{2}\right)\gamma_{0}+\left(\frac{1}{2}\frac{\left\Vert x_{0}-y\right\Vert ^{2}}{\eta}+6\eta\right)\sqrt{\sum_{t=1}^{T}\left\Vert \widehat{F(x_{t})}-\widehat{F(x_{t-1})}\right\Vert ^{2}}\\
 & -\sum_{t=1}^{T}\frac{1}{2}\gamma_{t-2}\left(\left\Vert x_{t-1}-z_{t-1}\right\Vert ^{2}+\left\Vert z_{t-1}-x_{t}\right\Vert ^{2}\right)+\frac{8}{\gamma_{0}}\left(\max_{0\leq t\leq T}\left\Vert \widehat{F(x_{t})}\right\Vert ^{2}\right)
\end{align*}
 as needed.
\end{proof}

By plugging in Lemma \ref{lem:regret-combined-1} into Lemma \ref{lem:error-fn-ub-1},
we obtain the following upper bound on the error function.
\begin{lem}
\label{lem:error-fn-ub-refined-1} Let $\xi_{t}=F(x_{t})-\widehat{F(x_{t})}$
and $G=\max_{0\leq t\leq T}\left\Vert F(x_{t})\right\Vert $. We have
\begin{align*}
T\cdot\err_{D}(\avx_{T}) & \leq\underbrace{\left(\frac{1}{2}\frac{D^{2}}{\eta}+6\eta\right)\sqrt{\sum_{t=1}^{T}\left\Vert \widehat{F(x_{t})}-\widehat{F(x_{t-1})}\right\Vert ^{2}}}_{\text{loss}}\\
 & -\underbrace{\sum_{t=1}^{T}\frac{1}{2}\gamma_{t-2}\left(\left\Vert x_{t-1}-z_{t-1}\right\Vert ^{2}+\left\Vert z_{t-1}-x_{t}\right\Vert ^{2}\right)}_{\text{gain}}\\
 & +\underbrace{\left\Vert \sum_{t=1}^{T}\xi_{t}\right\Vert D+\sum_{t=1}^{T}\left\langle \xi_{t},x_{t}-x_{0}\right\rangle +\frac{16}{\gamma_{0}}\left(\max_{0\leq t\leq T}\left\Vert \xi_{t}\right\Vert ^{2}\right)}_{\text{stochastic error}}\\
 & +\left(\frac{1}{2}D^{2}+3\eta^{2}\right)\gamma_{0}+\frac{16G^{2}}{\gamma_{0}}
\end{align*}
\end{lem}
\begin{proof}
By plugging in Lemma \ref{lem:regret-combined-1} into Lemma \ref{lem:error-fn-ub-1},
we obtain
\begin{align*}
T\cdot\err_{D}(\avx_{T}) & \leq\left(\frac{1}{2}\frac{D^{2}}{\eta}+6\eta\right)\sqrt{\sum_{t=1}^{T}\left\Vert \widehat{F(x_{t})}-\widehat{F(x_{t-1})}\right\Vert ^{2}}\\
 & -\sum_{t=1}^{T}\frac{1}{2}\gamma_{t-2}\left(\left\Vert x_{t-1}-z_{t-1}\right\Vert ^{2}+\left\Vert z_{t-1}-x_{t}\right\Vert ^{2}\right)\\
 & +\left\Vert \sum_{t=1}^{T}\xi_{t}\right\Vert D+\sum_{t=1}^{T}\left\langle \xi_{t},x_{t}-x_{0}\right\rangle +\frac{8}{\gamma_{0}}\left(\max_{0\leq t\leq T}\left\Vert \widehat{F(x_{t})}\right\Vert ^{2}\right)\\
 & +\left(\frac{1}{2}D^{2}+3\eta^{2}\right)\gamma_{0}
\end{align*}
 We further upper bound $\left\Vert \widehat{F(x_{t})}\right\Vert ^{2}$
as follows:
\begin{align*}
\left\Vert \widehat{F(x_{t})}\right\Vert ^{2} & =\left\Vert F(x_{t})-\xi_{t}\right\Vert ^{2}\\
 & \leq2\left\Vert F(x_{t})\right\Vert ^{2}+2\left\Vert \xi_{t}\right\Vert ^{2}\\
 & \leq2G^{2}+2\left\Vert \xi_{t}\right\Vert ^{2}
\end{align*}
 which gives the desired upper bound.
\end{proof}

\subsection{Analysis of the loss}

Here we analyze the loss and gain terms in the upper bound provided
by Lemma \ref{lem:error-fn-ub-refined-1} above:
\[
\underbrace{\left(\frac{1}{2}\frac{D^{2}}{\eta}+6\eta\right)\sqrt{\sum_{t=1}^{T}\left\Vert \widehat{F(x_{t})}-\widehat{F(x_{t-1})}\right\Vert ^{2}}}_{\text{loss}}-\underbrace{\sum_{t=1}^{T}\frac{1}{2}\gamma_{t-2}\left(\left\Vert x_{t-1}-z_{t-1}\right\Vert ^{2}+\left\Vert z_{t-1}-x_{t}\right\Vert ^{2}\right)}_{\text{gain}}
\]
For non-smooth operators, we ignore the gain term and bound the loss
term using an upper bound $G$ on the norm of the operator, leading
to an overall upper bound of $O\left(G\sqrt{T}\right)$ on the loss
(for $\eta=\Theta(D)$) plus an additional stochastic error. In contrast,
for smooth operators, we crucially use the gain term to balance the
loss term, leading to a\emph{ $O\left(\beta D^{2}+GD\right)$} upper
bound on the net loss (for $\eta=\Theta(D)$) plus an additional stochastic
error.
\begin{lem}
\label{lem:net-loss-nonsmooth-1} Suppose that $F$ is non-smooth.
Let $G=\max_{x\in\dom}\left\Vert F(x)\right\Vert $ and $\xi_{t}=F(x_{t})-\widehat{F(x_{t})}$.
We have
\begin{align*}
\sqrt{\sum_{t=1}^{T}\left\Vert \widehat{F(x_{t})}-\widehat{F(x_{t-1})}\right\Vert ^{2}} & \leq2\sqrt{2}G\sqrt{T}+2\sqrt{2}\sqrt{\sum_{t=0}^{T}\left\Vert \xi_{t}\right\Vert ^{2}}\\
 & =O\left(G\sqrt{T}+\sqrt{\sum_{t=0}^{T}\left\Vert \xi_{t}\right\Vert ^{2}}\right)
\end{align*}
\end{lem}
\begin{proof}
We have
\begin{align*}
\left\Vert \widehat{F(x_{t})}-\widehat{F(x_{t-1})}\right\Vert ^{2} & =\left\Vert F(x_{t})-\xi_{t}-F(x_{t-1})+\xi_{t-1}\right\Vert ^{2}\\
 & \leq4\left\Vert F(x_{t})\right\Vert ^{2}+4\left\Vert F(x_{t-1})\right\Vert ^{2}+4\left\Vert \xi_{t-1}\right\Vert ^{2}+4\left\Vert \xi_{t}\right\Vert ^{2}\\
 & \leq8G^{2}+4\left\Vert \xi_{t-1}\right\Vert ^{2}+4\left\Vert \xi_{t}\right\Vert ^{2}
\end{align*}
Therefore
\begin{align*}
\sqrt{\sum_{t=1}^{T}\left\Vert \widehat{F(x_{t})}-\widehat{F(x_{t-1})}\right\Vert ^{2}} & \leq\sqrt{8G^{2}T+8\sum_{t=0}^{T}\left\Vert \xi_{t}\right\Vert ^{2}}\\
 & \leq2\sqrt{2}G\sqrt{T}+2\sqrt{2}\sqrt{\sum_{t=0}^{T}\left\Vert \xi_{t}\right\Vert ^{2}}
\end{align*}
 as needed. 
\end{proof}

\begin{lem}
\label{lem:net-loss-smooth-1} Suppose that $F$ is $\beta$-smooth
with respect to the $\ell_{2}$-norm. Let $G=\max_{x\in\dom}\left\Vert F(x)\right\Vert $
and $\xi_{t}=F(x_{t})-\widehat{F(x_{t})}$. We have
\begin{align*}
 & \left(\frac{1}{2}\frac{D^{2}}{\eta}+6\eta\right)\sqrt{\sum_{t=1}^{T}\left\Vert \widehat{F(x_{t})}-\widehat{F(x_{t-1})}\right\Vert ^{2}}-\sum_{t=1}^{T}\frac{1}{2}\gamma_{t-2}\left(\left\Vert x_{t-1}-z_{t-1}\right\Vert ^{2}+\left\Vert z_{t-1}-x_{t}\right\Vert ^{2}\right)\\
 & \leq\left(\frac{1}{2}\frac{D^{2}}{\eta}+6\eta\right)\left(2\sqrt{2}\beta\sqrt{\frac{1}{2}D^{2}+6\eta^{2}}+4\sqrt{2}G\right)+\left(\sqrt{2}\frac{D^{2}}{\eta}+12\sqrt{2}\eta\right)\sqrt{\sum_{t=0}^{T}\left\Vert \xi_{t}\right\Vert ^{2}}\\
 & =O\left(\left(\frac{D^{2}}{\eta}+\eta\right)\left(\beta\left(D+\eta\right)+G\right)+\left(\frac{D^{2}}{\eta}+\eta\right)\sqrt{\sum_{t=0}^{T}\left\Vert \xi_{t}\right\Vert ^{2}}\right)
\end{align*}
\end{lem}
\begin{proof}
Recall that we want to upper bound the net loss:
\[
\underbrace{r\sqrt{\sum_{t=1}^{T}\left\Vert \widehat{F(x_{t})}-\widehat{F(x_{t-1})}\right\Vert ^{2}}}_{\text{loss}}-\underbrace{\frac{1}{2}\sum_{t=1}^{T}\gamma_{t-2}\left(\left\Vert x_{t}-z_{t-1}\right\Vert ^{2}+\left\Vert x_{t-1}-z_{t-1}\right\Vert ^{2}\right)}_{\text{gain}}
\]
 where we let $r:=\left(\frac{1}{2}\frac{D^{2}}{\eta}+6\eta\right)$.

The loss is proportional to the stochastic operator value differences,
whereas the gain is proportional to the iterate movement. Our main
approach is to relate the gain to the loss, and show that we can use
the gain to offset the loss. We start by using smoothness and the
inequality $(a+b)^{2}\leq2a^{2}+2b^{2}$, and relate the gain to the
deterministic operator value differences:
\begin{align*}
\left\Vert F(x_{t})-F(x_{t-1})\right\Vert ^{2} & \leq\beta^{2}\left\Vert x_{t}-x_{t-1}\right\Vert ^{2}\\
 & =\beta^{2}\left\Vert x_{t}-z_{t-1}+z_{t-1}-x_{t-1}\right\Vert ^{2}\\
 & \leq2\beta^{2}\left(\left\Vert x_{t}-z_{t-1}\right\Vert ^{2}+\left\Vert x_{t-1}-z_{t-1}\right\Vert ^{2}\right)
\end{align*}
Therefore
\[
\left\Vert x_{t}-z_{t-1}\right\Vert ^{2}+\left\Vert x_{t-1}-z_{t-1}\right\Vert ^{2}\geq\frac{1}{2\beta^{2}}\left\Vert F(x_{t})-F(x_{t-1})\right\Vert ^{2}
\]
 Thus we can upper bound the net loss as follows:

\begin{align}
 & r\sqrt{\sum_{t=1}^{T}\left\Vert \widehat{F(x_{t})}-\widehat{F(x_{t-1})}\right\Vert ^{2}}-\frac{1}{2}\sum_{t=1}^{T}\gamma_{t-2}\left(\left\Vert x_{t}-z_{t-1}\right\Vert ^{2}+\left\Vert x_{t-1}-z_{t-1}\right\Vert ^{2}\right)\nonumber \\
 & \leq r\sqrt{\sum_{t=1}^{T}\left\Vert \widehat{F(x_{t})}-\widehat{F(x_{t-1})}\right\Vert ^{2}}-\frac{1}{2}\sum_{t=1}^{T}\gamma_{t-2}\frac{1}{2\beta^{2}}\left\Vert F(x_{t})-F(x_{t-1})\right\Vert ^{2}\label{eq:net-loss1-1}
\end{align}
 We now show that after an initial number of iterations, the gain
offsets the loss up to a stochastic error term. Recall that the step
sizes $\gamma_{t}$ are increasing with $t$. Let $\Gamma\geq0$ be
a value that we will determine later. Let $\tau$ be the last iteration
$t$ such that $\gamma_{t-2}\leq\Gamma$. Note that the definition
of $\tau$ implies that
\begin{align*}
\sqrt{\sum_{t=1}^{\tau-2}\left\Vert \widehat{F(x_{t})}-\widehat{F(x_{t-1})}\right\Vert ^{2}} & \leq\sqrt{\eta^{2}\gamma_{0}^{2}+\sum_{t=1}^{\tau-2}\left\Vert \widehat{F(x_{t})}-\widehat{F(x_{t-1})}\right\Vert ^{2}}\\
 & =\eta\gamma_{\tau-2}\\
 & \leq\eta\Gamma
\end{align*}
 and
\[
\gamma_{t-2}\geq\Gamma\quad\forall t\geq\tau+1
\]
Using the above inequalities, we obtain
\begin{align}
 & r\sqrt{\sum_{t=1}^{T}\left\Vert \widehat{F(x_{t})}-\widehat{F(x_{t-1})}\right\Vert ^{2}}-\frac{1}{2}\sum_{t=1}^{T}\gamma_{t-2}\frac{1}{2\beta^{2}}\left\Vert F(x_{t})-F(x_{t-1})\right\Vert ^{2}\nonumber \\
 & \leq r\underbrace{\sqrt{\sum_{t=1}^{\tau-2}\left\Vert \widehat{F(x_{t})}-\widehat{F(x_{t-1})}\right\Vert ^{2}}}_{\leq\eta\Gamma}+r\sqrt{\sum_{t=\tau-1}^{T}\left\Vert \widehat{F(x_{t})}-\widehat{F(x_{t-1})}\right\Vert ^{2}}\nonumber \\
 & -\frac{1}{2}\sum_{t=\tau+1}^{T}\underbrace{\gamma_{t-2}}_{\geq\Gamma}\frac{1}{2\beta^{2}}\left\Vert F(x_{t})-F(x_{t-1})\right\Vert ^{2}\nonumber \\
 & \leq r\eta\Gamma+r\sqrt{\sum_{t=\tau-1}^{T}\left\Vert \widehat{F(x_{t})}-\widehat{F(x_{t-1})}\right\Vert ^{2}}-\frac{\Gamma}{4\beta^{2}}\sum_{t=\tau+1}^{T}\left\Vert F(x_{t})-F(x_{t-1})\right\Vert ^{2}\label{eq:net-loss2-1}
\end{align}
Next, we upper bound the second term above. Using the definition of
$\xi_{t}:=F(x_{t})-\widehat{F(x_{t})}$ and the inequality $(a+b)^{2}\leq2a^{2}+2b^{2}$,
we obtain
\begin{align*}
\left\Vert \widehat{F(x_{t})}-\widehat{F(x_{t-1})}\right\Vert ^{2} & =\left\Vert \xi_{t-1}-\xi_{t}+F(x_{t})-F(x_{t-1})\right\Vert ^{2}\\
 & \leq2\left\Vert \xi_{t-1}-\xi_{t}\right\Vert ^{2}+2\left\Vert F(x_{t})-F(x_{t-1})\right\Vert ^{2}\\
 & \leq4\left\Vert \xi_{t-1}\right\Vert ^{2}+4\left\Vert \xi_{t}\right\Vert ^{2}+2\left\Vert F(x_{t})-F(x_{t-1})\right\Vert ^{2}
\end{align*}
Therefore
\begin{align*}
\sum_{t=\tau-1}^{T}\left\Vert \widehat{F(x_{t})}-\widehat{F(x_{t-1})}\right\Vert ^{2} & \leq\sum_{t=\tau-1}^{T}\left(4\left\Vert \xi_{t-1}\right\Vert ^{2}+4\left\Vert \xi_{t}\right\Vert ^{2}+2\left\Vert F(x_{t})-F(x_{t-1})\right\Vert ^{2}\right)\\
 & \leq8\sum_{t=0}^{T}\left\Vert \xi_{t}\right\Vert ^{2}+2\sum_{t=\tau-1}^{T}\left\Vert F(x_{t})-F(x_{t-1})\right\Vert ^{2}
\end{align*}
 and hence
\begin{align}
 & \sqrt{\sum_{t=\tau-1}^{T}\left\Vert \widehat{F(x_{t})}-\widehat{F(x_{t-1})}\right\Vert ^{2}}\nonumber \\
 & \leq2\sqrt{2}\sqrt{\sum_{t=0}^{T}\left\Vert \xi_{t}\right\Vert ^{2}}+\sqrt{2}\sqrt{\sum_{t=\tau-1}^{T}\left\Vert F(x_{t})-F(x_{t-1})\right\Vert ^{2}}\nonumber \\
 & \leq2\sqrt{2}\sqrt{\sum_{t=0}^{T}\left\Vert \xi_{t}\right\Vert ^{2}}+\sqrt{2}\underbrace{\left\Vert F(x_{\tau-1})-F(x_{\tau-2})\right\Vert }_{\leq\left\Vert F(x_{\tau-1})\right\Vert +\left\Vert F(x_{\tau-2})\right\Vert \leq2G}+\sqrt{2}\underbrace{\left\Vert F(x_{\tau})-F(x_{\tau-1})\right\Vert }_{\leq\left\Vert F(x_{\tau})\right\Vert +\left\Vert F(x_{\tau-1})\right\Vert \leq2G}+\sqrt{2}\sqrt{\sum_{t=\tau+1}^{T}\left\Vert F(x_{t})-F(x_{t-1})\right\Vert ^{2}}\nonumber \\
 & \leq2\sqrt{2}\sqrt{\sum_{t=0}^{T}\left\Vert \xi_{t}\right\Vert ^{2}}+4\sqrt{2}G+\sqrt{2}\sqrt{\sum_{t=\tau+1}^{T}\left\Vert F(x_{t})-F(x_{t-1})\right\Vert ^{2}}\label{eq:net-loss3-1}
\end{align}
Combining (\ref{eq:net-loss1-1}), (\ref{eq:net-loss2-1}), (\ref{eq:net-loss3-1}),
we obtain
\begin{align}
 & r\sqrt{\sum_{t=1}^{T}\left\Vert \widehat{F(x_{t})}-\widehat{F(x_{t-1})}\right\Vert ^{2}}-\frac{1}{2}\sum_{t=1}^{T}\gamma_{t-2}\left(\left\Vert x_{t}-z_{t-1}\right\Vert ^{2}+\left\Vert x_{t-1}-z_{t-1}\right\Vert ^{2}\right)\nonumber \\
 & \leq r\eta\Gamma+2\sqrt{2}r\sqrt{\sum_{t=0}^{T}\left\Vert \xi_{t}\right\Vert ^{2}}+4\sqrt{2}rG\nonumber \\
 & +\sqrt{2}r\sqrt{\sum_{t=\tau+1}^{T}\left\Vert F(x_{t})-F(x_{t-1})\right\Vert ^{2}}-\frac{\Gamma}{4\beta^{2}}\sum_{t=\tau+1}^{T}\left\Vert F(x_{t})-F(x_{t-1})\right\Vert ^{2}\nonumber \\
 & \leq r\eta\Gamma+2\sqrt{2}r\sqrt{\sum_{t=0}^{T}\left\Vert \xi_{t}\right\Vert ^{2}}+4\sqrt{2}rG+\max_{y\geq0}\left\{ \sqrt{2}ry-\frac{\Gamma}{4\beta^{2}}y^{2}\right\} \nonumber \\
 & =r\eta\Gamma+2\sqrt{2}r\sqrt{\sum_{t=0}^{T}\left\Vert \xi_{t}\right\Vert ^{2}}+4\sqrt{2}rG+\frac{2r^{2}\beta^{2}}{\Gamma}\label{eq:net-loss4-1}
\end{align}
On the last line, we used the fact that the function $\phi(y)=ay-by^{2}$,
where $a,b>0$ are positive constants, is a concave function and it
is maximized at $y^{*}=\frac{a}{2b}$ and $\phi(y^{*})=\frac{a^{2}}{4b}$.

Finally, we choose $\Gamma$ in order to balance the terms:
\[
\Gamma=\beta\sqrt{\frac{2r}{\eta}}
\]
We now plug in the above choice of $\Gamma$ into (\ref{eq:net-loss4-1})
and recall that $r=\frac{1}{2}\frac{D^{2}}{\eta}+6\eta$. We obtain
\begin{align*}
 & r\sqrt{\sum_{t=1}^{T}\left\Vert \widehat{F(x_{t})}-\widehat{F(x_{t-1})}\right\Vert ^{2}}-\frac{1}{2}\sum_{t=1}^{T}\gamma_{t-2}\left(\left\Vert x_{t}-z_{t-1}\right\Vert ^{2}+\left\Vert x_{t-1}-z_{t-1}\right\Vert ^{2}\right)\\
 & \leq2\sqrt{2}\beta r\sqrt{r\eta}+4\sqrt{2}rG+2\sqrt{2}r\sqrt{\sum_{t=0}^{T}\left\Vert \xi_{t}\right\Vert ^{2}}\\
 & =\left(\frac{1}{2}\frac{D^{2}}{\eta}+6\eta\right)\left(2\sqrt{2}\beta\sqrt{\frac{1}{2}D^{2}+6\eta^{2}}+4\sqrt{2}G\right)+\left(\sqrt{2}\frac{D^{2}}{\eta}+12\sqrt{2}\eta\right)\sqrt{\sum_{t=0}^{T}\left\Vert \xi_{t}\right\Vert ^{2}}
\end{align*}
 as needed.
\end{proof}

\subsection{Analysis of the expected stochastic error}

\label{subsec:stoch-error-unbounded-domains}

Next, we analyze the stochastic error terms in Lemmas \ref{lem:error-fn-ub-refined-1},
\ref{lem:net-loss-nonsmooth-1}, \ref{lem:net-loss-smooth-1}:
\[
\underbrace{\sum_{t=1}^{T}\left\langle \xi_{t},x_{t}-x_{0}\right\rangle \quad\sqrt{\sum_{t=0}^{T}\left\Vert \xi_{t}\right\Vert ^{2}}\quad\left\Vert \sum_{t=1}^{T}\xi_{t}\right\Vert \quad\max_{0\leq t\leq T}\left\Vert \xi_{t}\right\Vert ^{2}}_{\text{stochastic error terms}}
\]
We can analyze the first three terms using Lemmas \ref{lem:stoch-err1},
\ref{lem:stoch-err2}, \ref{lem:stoch-err3}. Thus it only remains
to analyze the fourth term $\max_{0\leq t\leq T}\left\Vert \xi_{t}\right\Vert ^{2}$.
We do so via the following lemma:
\begin{lem}
Let $X_{1},X_{2},\dots,X_{k}$ be non-negative random variables such
that $\V\left[X_{i}\right]\leq\sigma^{2}$ for all $1\leq i\leq k$.
We have
\[
\E\left[\max_{1\leq i\leq k}X_{i}\right]\leq2\sqrt{k}\sigma
\]
\end{lem}
\begin{proof}
Since $\max_{1\leq i\leq k}X_{i}$ is a non-negative random variable,
we can write its expectation as
\[
\E\left[\max_{1\leq i\leq k}X_{i}\right]=\int_{0}^{\infty}\Pr\left[\max_{1\leq i\leq k}X_{i}\geq\nu\right]d\nu
\]
Next, we split the integral as follows:
\[
\int_{0}^{\infty}\Pr\left[\max_{1\leq i\leq k}X_{i}\geq\nu\right]d\nu=\int_{0}^{\sqrt{k}\sigma}\Pr\left[\max_{1\leq i\leq k}X_{i}\geq\nu\right]d\nu+\int_{\sqrt{k}\sigma}^{\infty}\Pr\left[\max_{1\leq i\leq k}X_{i}\geq\nu\right]d\nu
\]
For $\nu\leq\sqrt{k}\sigma$, we use the naive upper bound $\Pr\left[\max_{1\leq i\leq k}X_{i}\geq\nu\right]\leq1$.
Thus
\[
\int_{0}^{\sqrt{k}\sigma}\Pr\left[\max_{1\leq i\leq k}X_{i}>\nu\right]d\nu\leq\int_{0}^{\sqrt{k}\sigma}d\nu=\sqrt{k}\sigma
\]
 For $\nu\geq\sqrt{k}\sigma$, we upper bound $\Pr\left[\max_{1\leq i\leq k}X_{i}\geq\nu\right]$
using the union-bound and Chebyshev's inequalities:
\begin{align*}
\Pr\left[\max_{1\leq i\leq k}X_{i}\geq\nu\right] & \leq\sum_{i=1}^{k}\Pr\left[X_{i}\geq\nu\right]\leq\sum_{i=1}^{k}\frac{\V\left[X_{i}\right]}{\nu^{2}}\leq\frac{k\sigma^{2}}{\nu^{2}}
\end{align*}
 Thus
\begin{align*}
\int_{\sqrt{k}\sigma}^{\infty}\Pr\left[\max_{1\leq i\leq k}X_{i}\geq\nu\right]d\nu & \leq k\sigma^{2}\int_{\sqrt{k}\sigma}^{\infty}\frac{1}{\nu^{2}}d\nu=k\sigma^{2}\left(-\frac{1}{\nu}\vert_{\sqrt{k}\sigma}^{\infty}\right)=\sqrt{k}\sigma
\end{align*}
\end{proof}

Using the above lemma, the martingale assumption (\ref{eq:stoch-assumption-unbiased}),
and variance assumption (\ref{eq:stoch-assumption-variance}), we
obtain:
\begin{lem}
\label{lem:stoch-err4} We have
\[
\E\left[\max_{0\leq t\leq T}\left\Vert \xi_{t}\right\Vert ^{2}\right]\leq2\sqrt{T+1}\sigma
\]
 
\end{lem}

\subsection{Putting everything together}

We now put everything together and complete the proof of Theorem \ref{thm:adapeg-unbounded-convergence}.
\begin{lem}
Let $D>0$ be any fixed positive value. Suppose that $F$ is non-smooth.
Let $G:=\max_{0\leq t\leq T}\left\Vert F(x_{t})\right\Vert $. We
have
\begin{align*}
T\cdot\E\left[\err_{D}(\avx_{T})\right] & \leq\left(\sqrt{2}\frac{D^{2}}{\eta}+12\eta\right)G\sqrt{T}\\
 & \left(\sqrt{2}\frac{D^{2}}{\eta}+12\eta+D+\frac{32}{\gamma_{0}}\right)\sigma\sqrt{T+1}\\
 & +\left(\frac{1}{2}D^{2}+3\eta^{2}\right)\gamma_{0}+\frac{16G^{2}}{\gamma_{0}}
\end{align*}
 Setting $\eta=\Theta\left(D\right)$ gives
\[
\E\left[\err_{D}(\avx_{T})\right]\leq O\left(\frac{\gamma_{0}D^{2}+\gamma_{0}^{-1}G^{2}}{T}+\frac{DG+\left(D+\gamma_{0}^{-1}\right)\sigma}{\sqrt{T}}\right)
\]
\end{lem}
\begin{proof}
The lemma follows by combining Lemmas \ref{lem:error-fn-ub-refined-1}
, \ref{lem:net-loss-nonsmooth-1}, \ref{lem:stoch-err1}, \ref{lem:stoch-err2},
\ref{lem:stoch-err3}, \ref{lem:stoch-err4}.
\end{proof}
\begin{lem}
Let $D>0$ be any fixed positive value. Suppose that $F$ is $\beta$-smooth
with respect to the $\ell_{2}$-norm. Let $G:=\max_{0\leq t\leq T}\left\Vert F(x_{t})\right\Vert $.
We have
\begin{align*}
T\cdot\E\left[\err_{D}(\avx_{T})\right] & \leq\left(\frac{1}{2}\frac{D^{2}}{\eta}+6\eta\right)\left(2\sqrt{2}\beta\sqrt{\frac{1}{2}D^{2}+6\eta^{2}}+4\sqrt{2}G\right)\\
 & +\left(\sqrt{2}\frac{D^{2}}{\eta}+12\sqrt{2}\eta+D+\frac{32}{\gamma_{0}}\right)\sigma\sqrt{T+1}\\
 & +\left(\frac{1}{2}D^{2}+3\eta^{2}\right)\gamma_{0}+\frac{16G^{2}}{\gamma_{0}}
\end{align*}
 Setting $\eta=\Theta\left(D\right)$ gives
\begin{align*}
\E\left[\err_{D}(\avx_{T})\right] & \leq O\left(\frac{\left(\beta+\gamma_{0}\right)D^{2}+GD+G^{2}\gamma_{0}^{-1}}{T}+\frac{\left(D+\gamma_{0}^{-1}\right)\sigma}{\sqrt{T}}\right)
\end{align*}
\end{lem}
\begin{proof}
The lemma follows by combining Lemmas \ref{lem:error-fn-ub-refined-1}
, \ref{lem:net-loss-smooth-1}, \ref{lem:stoch-err1}, \ref{lem:stoch-err2},
\ref{lem:stoch-err3}, \ref{lem:stoch-err4}.
\end{proof}

\subsection{A guarantee for smooth operators that does not depend on $G$}

\label{subsec:smooth-without-G}

Here we adapt the analysis to show a guarantee in the spirit of \citet{antonakopoulos2021adaptive}.
As in \citet{antonakopoulos2021adaptive}, we consider the deterministic
setting $(\sigma=0)$. Since our main goal is to compare with the
work of \citet{antonakopoulos2021adaptive} which has sub-optimal
dependency on problem parameters such as the smoothness parameter
$\beta$, we state a guarantee with a sub-optimal dependency of $\beta^{2}$
on the smoothness. Aiming for a sub-optimal dependence of $\beta^{2}$
simplifies the analysis, as the careful accounting of the error in
Lemma \ref{lem:net-loss-smooth-1} that obtains the optimal dependency
of $\beta$ is no longer needed.
\begin{lem}
\label{lem:regret-combined-1-1}Suppose that $F$ is $\beta$-smooth
with respect to the $\ell_{2}$-norm. Consider the deterministic setting
($\sigma=0$). For any $y\in\dom$, we have 
\begin{align*}
\sum_{t=1}^{T}\left\langle F(x_{t}),x_{t}-y\right\rangle  & \leq\frac{1}{2\gamma_{0}}\frac{\left\Vert x_{0}-y\right\Vert ^{4}}{\eta^{2}}\beta^{2}+\frac{\gamma_{0}}{2}\left\Vert x_{0}-y\right\Vert ^{2}+\frac{8}{\gamma_{0}}\beta^{2}\eta^{2}\\
 & +\frac{1}{\gamma_{0}}\left\Vert F(x_{\tau})-F(x_{\tau-1})\right\Vert ^{2}+\frac{1}{\gamma_{0}}\left\Vert F(x_{\tau-1})-F(x_{\tau-2})\right\Vert ^{2}
\end{align*}
where $\tau$ is the last iteration $t$ such that $\gamma_{t-2}\leq2\sqrt{2}\beta$.

\end{lem}
\begin{proof}
Our starting point is the following guarantee shown in the proof of
Lemma \ref{lem:regret-combined-1}:
\begin{align*}
\sum_{t=1}^{T}\left\langle F(x_{t}),x_{t}-y\right\rangle  & \leq\frac{1}{2}\gamma_{T-1}\left\Vert x_{0}-y\right\Vert ^{2}+\sum_{t=1}^{T}\frac{1}{\gamma_{t-1}}\left\Vert F(x_{t})-F(x_{t-1})\right\Vert ^{2}\\
 & -\sum_{t=1}^{T}\frac{1}{2}\gamma_{t-2}\left(\left\Vert x_{t-1}-z_{t-1}\right\Vert ^{2}+\left\Vert z_{t-1}-x_{t}\right\Vert ^{2}\right)
\end{align*}
 As before, we use smoothness to relate the loss terms to the gain
terms and obtain:
\begin{align*}
\left\Vert F(x_{t})-F(x_{t-1})\right\Vert ^{2} & =\left\Vert F(x_{t})-F(z_{t-1})+F(z_{t-1})-F(x_{t-1})\right\Vert ^{2}\\
 & \leq2\left(\left\Vert F(x_{t})-F(z_{t-1})\right\Vert ^{2}+\left\Vert F(z_{t-1})-F(x_{t-1})\right\Vert ^{2}\right)\\
 & \leq2\beta^{2}\left(\left\Vert x_{t}-z_{t-1}\right\Vert ^{2}+\left\Vert x_{t-1}-z_{t-1}\right\Vert ^{2}\right)
\end{align*}
 Plugging into the previous inequality,
\begin{align*}
\sum_{t=1}^{T}\left\langle F(x_{t}),x_{t}-y\right\rangle  & \leq\frac{1}{2}\gamma_{T-1}\left\Vert x_{0}-y\right\Vert ^{2}+\sum_{t=1}^{T}\left(\frac{1}{\gamma_{t-1}}-\frac{1}{4\beta^{2}}\gamma_{t-2}\right)\left\Vert F(x_{t})-F(x_{t-1})\right\Vert ^{2}
\end{align*}
Recall that the step sizes $\gamma_{t}$ are increasing with $t$.
Let $\tau$ be the last iteration $t$ such that $\gamma_{t-2}\leq2\sqrt{2}\beta$.
For $t\geq\tau+1$, we have
\[
\frac{1}{\gamma_{t-1}}-\frac{1}{8\beta^{2}}\gamma_{t-2}\leq\frac{1}{\gamma_{t-2}}-\frac{1}{8\beta^{2}}\gamma_{t-2}\leq0
\]
 Therefore
\begin{align*}
\sum_{t=1}^{T}\left\langle F(x_{t}),x_{t}-y\right\rangle  & \leq\frac{1}{2}\gamma_{T-1}\left\Vert x_{0}-y\right\Vert ^{2}+\sum_{t=1}^{\tau}\frac{1}{\gamma_{t-1}}\left\Vert F(x_{t})-F(x_{t-1})\right\Vert ^{2}-\frac{1}{8\beta^{2}}\sum_{t=1}^{T}\gamma_{t-2}\left\Vert F(x_{t})-F(x_{t-1})\right\Vert ^{2}
\end{align*}
As we noted above, we will aim for a weaker guarantee than in Lemma
\ref{lem:net-loss-smooth-1}. Instead of leveraging that the two sums
above depend on the scalings $\gamma_{t}$, we will simply lower bound
$\gamma_{t}$ by $\gamma_{0}$ and forgo the possible gains coming
from increasing $\gamma_{t}$. This makes the analysis considerably
simpler at the cost of increasing the dependency on the smoothness
from $\beta$ to $\beta^{2}$.

Using that $\gamma_{t}\geq\gamma_{0}$ for all $t$, we obtain
\begin{align}
\sum_{t=1}^{T}\left\langle F(x_{t}),x_{t}-y\right\rangle  & \leq\underbrace{\frac{1}{2}\gamma_{T-1}\left\Vert x_{0}-y\right\Vert ^{2}+\frac{1}{\gamma_{0}}\sum_{t=1}^{\tau}\left\Vert F(x_{t})-F(x_{t-1})\right\Vert ^{2}}_{\text{losses}}-\underbrace{\frac{\gamma_{0}}{8\beta^{2}}\sum_{t=1}^{T}\left\Vert F(x_{t})-F(x_{t-1})\right\Vert ^{2}}_{\text{gain}}\label{eq:1}
\end{align}
 We now use the definition of the step sizes and the definition of
$\tau$ to upper bound the first two terms.

By the definition of the step sizes, we have
\begin{align*}
\frac{1}{2}\gamma_{T-1}\left\Vert x_{0}-y\right\Vert ^{2} & =\frac{1}{2}\frac{\left\Vert x_{0}-y\right\Vert ^{2}}{\eta}\sqrt{\eta^{2}\gamma_{0}^{2}+\sum_{t=1}^{T-1}\left\Vert F(x_{t})-F(x_{t-1})\right\Vert ^{2}}\\
 & \leq\frac{1}{2}\gamma_{0}\left\Vert x_{0}-y\right\Vert ^{2}+\frac{1}{2}\frac{\left\Vert x_{0}-y\right\Vert ^{2}}{\eta}\sqrt{\sum_{t=1}^{T-1}\left\Vert F(x_{t})-F(x_{t-1})\right\Vert ^{2}}
\end{align*}
 By the definition of the step sizes and the definition of $\tau$,
we have
\[
\frac{1}{\eta}\sqrt{\eta^{2}\gamma_{0}^{2}+\sum_{t=1}^{\tau-2}\left\Vert F(x_{t})-F(x_{t-1})\right\Vert ^{2}}=\gamma_{\tau-2}\leq2\sqrt{2}\beta
\]
which implies
\[
\sum_{t=1}^{\tau-2}\left\Vert F(x_{t})-F(x_{t-1})\right\Vert ^{2}\leq\eta^{2}\left(8\beta^{2}-\gamma_{0}^{2}\right)
\]
 Plugging into (\ref{eq:1}), we obtain
\begin{align*}
\sum_{t=1}^{T}\left\langle F(x_{t}),x_{t}-y\right\rangle  & \leq\frac{1}{2}\gamma_{0}\left\Vert x_{0}-y\right\Vert ^{2}+8\frac{1}{\gamma_{0}}\beta^{2}\eta^{2}+\frac{1}{\gamma_{0}}\left\Vert F(x_{\tau})-F(x_{\tau-1})\right\Vert ^{2}+\frac{1}{\gamma_{0}}\left\Vert F(x_{\tau-1})-F(x_{\tau-2})\right\Vert ^{2}\\
 & +\frac{1}{2}\frac{\left\Vert x_{0}-y\right\Vert ^{2}}{\eta}\sqrt{\sum_{t=1}^{T-1}\left\Vert F(x_{t})-F(x_{t-1})\right\Vert ^{2}}-\frac{\gamma_{0}}{8\beta^{2}}\sum_{t=1}^{T}\left\Vert F(x_{t})-F(x_{t-1})\right\Vert ^{2}\\
 & \leq\frac{1}{2}\gamma_{0}\left\Vert x_{0}-y\right\Vert ^{2}+8\frac{1}{\gamma_{0}}\beta^{2}\eta^{2}+\frac{1}{\gamma_{0}}\left\Vert F(x_{\tau})-F(x_{\tau-1})\right\Vert ^{2}+\frac{1}{\gamma_{0}}\left\Vert F(x_{\tau-1})-F(x_{\tau-2})\right\Vert ^{2}\\
 & +\max_{y\geq0}\left\{ \frac{1}{2}\frac{\left\Vert x_{0}-y\right\Vert ^{2}}{\eta}y-\frac{\gamma_{0}}{8\beta^{2}}y^{2}\right\} \\
 & =\frac{1}{2}\gamma_{0}\left\Vert x_{0}-y\right\Vert ^{2}+8\frac{1}{\gamma_{0}}\beta^{2}\eta^{2}+\frac{1}{\gamma_{0}}\left\Vert F(x_{\tau})-F(x_{\tau-1})\right\Vert ^{2}+\frac{1}{\gamma_{0}}\left\Vert F(x_{\tau-1})-F(x_{\tau-2})\right\Vert ^{2}\\
 & +\frac{1}{2\gamma_{0}}\frac{\left\Vert x_{0}-y\right\Vert ^{4}}{\eta^{2}}\beta^{2}
\end{align*}
On the last line, we used the fact that the function $\phi(y)=ay-by^{2}$,
where $a,b>0$ are positive constants, is a concave function and it
is maximized at $y^{*}=\frac{a}{2b}$ and $\phi(y^{*})=\frac{a^{2}}{4b}$.
\end{proof}

Combining with Lemma \ref{lem:error-fn-ub-refined-1}, we obtain:
\begin{lem}
\label{lem:without-G}Let $D>0$ be any fixed positive value. Suppose
that $F$ is $\beta$-smooth with respect to the $\ell_{2}$-norm.
Consider the deterministic setting ($\sigma=0$). We have
\[
\err_{D}(\avx_{T})\leq\frac{1}{T}\left(\frac{1}{2\gamma_{0}}\frac{D^{4}}{\eta^{2}}\beta^{2}+\frac{\gamma_{0}}{2}D^{2}+\frac{8}{\gamma_{0}}\beta^{2}\eta^{2}+\frac{1}{\gamma_{0}}\left\Vert F(x_{\tau})-F(x_{\tau-1})\right\Vert ^{2}+\frac{1}{\gamma_{0}}\left\Vert F(x_{\tau-1})-F(x_{\tau-2})\right\Vert ^{2}\right)
\]
 where $\tau$ is the last iteration $t$ such that $\gamma_{t-2}\leq2\sqrt{2}\beta$.

Setting $\eta=\Theta(D)$ and $\gamma_{0}=1$, we obtain
\[
\err_{D}(\avx_{T})\leq O\left(\frac{\beta^{2}D^{2}+\left\Vert F(x_{\tau})-F(x_{\tau-1})\right\Vert ^{2}+\left\Vert F(x_{\tau-1})-F(x_{\tau-2})\right\Vert ^{2}}{T}\right)
\]
\end{lem}

\section{Adaptive extra-gradient algorithms}

\label{sec:extensions}

\begin{algorithm}
\caption{AdaEG algorithm for bounded domains $\protect\dom$.}
\label{alg:adaeg}

Let $x_{0}=z_{0}\in\dom$, $\gamma_{0}\geq0$, $\eta>0$.

For $t=1,\dots,T$, update:
\begin{align*}
x_{t} & =\arg\min_{u\in\dom}\left\{ \left\langle \widehat{F(z_{t-1})},u\right\rangle +\frac{1}{2}\gamma_{t-1}\left\Vert u-z_{t-1}\right\Vert ^{2}\right\} \\
z_{t} & =\arg\min_{u\in\dom}\left\{ \left\langle \widehat{F(x_{t})},u\right\rangle +\frac{1}{2}\gamma_{t-1}\left\Vert u-z_{t-1}\right\Vert ^{2}+\frac{1}{2}\left(\gamma_{t}-\gamma_{t-1}\right)\left\Vert u-x_{t}\right\Vert ^{2}\right\} \\
\gamma_{t} & =\frac{1}{\eta}\sqrt{\eta^{2}\gamma_{0}^{2}+\sum_{s=1}^{t}\left\Vert \widehat{F(x_{s})}-\widehat{F(z_{s-1})}\right\Vert ^{2}}
\end{align*}

Return $\overline{x}_{T}=\frac{1}{T}\sum_{t=1}^{T}x_{t}$.
\end{algorithm}

\begin{algorithm}
\caption{AdaEG algorithm for unbounded domains $\protect\dom$.}
\label{alg:adaeg-unbounded}

Let $x_{0}=z_{0}\in\dom$, $\gamma_{0}\geq0$,$\gamma_{-1}=0$, $\eta>0$.

For $t=1,\dots,T$, update:
\begin{align*}
x_{t} & =\arg\min_{u\in\dom}\left\{ \left\langle \widehat{F(z_{t-1})},u\right\rangle +\frac{1}{2}\gamma_{t-2}\left\Vert u-z_{t-1}\right\Vert ^{2}+\frac{1}{2}\left(\gamma_{t-1}-\gamma_{t-2}\right)\left\Vert u-x_{0}\right\Vert ^{2}\right\} \\
z_{t} & =\arg\min_{u\in\dom}\left\{ \left\langle \widehat{F(x_{t})},u\right\rangle +\frac{1}{2}\gamma_{t-2}\left\Vert u-z_{t-1}\right\Vert ^{2}+\frac{1}{2}\left(\gamma_{t-1}-\gamma_{t-2}\right)\left\Vert u-x_{0}\right\Vert ^{2}\right\} \\
\gamma_{t} & =\frac{1}{\eta}\sqrt{\eta^{2}\gamma_{0}^{2}+\sum_{s=1}^{t}\left\Vert \widehat{F(x_{s})}-\widehat{F(z_{s-1})}\right\Vert ^{2}}
\end{align*}

Return $\overline{x}_{T}=\frac{1}{T}\sum_{t=1}^{T}x_{t}$.
\end{algorithm}

In this section, we discuss the 2-call variants of our algorithms
based on the Extra-Gradient algorithm \citep{Korpelevich76}. The
algorithms are shown in Algorithms \ref{alg:adaeg} and \ref{alg:adaeg-unbounded}.
The analysis is analogous to the analysis of Algorithms \ref{alg:adapeg}
and \ref{alg:adapeg-unbounded}, given in Sections \ref{sec:adapeg-analysis}
and \ref{sec:adapeg-unbounded-analysis}, respectively. For concreteness,
we consider Algorithm \ref{alg:adapeg}, and the analysis of Algorithm
\ref{alg:adapeg-unbounded} can be modified analogously. 

The starting point is the following upper bound on the error function,
given by Lemma \ref{lem:error-fn-ub}:
\begin{align*}
T\cdot\err(\avx_{T}) & \leq\underbrace{\sup_{y\in\dom}\left(\sum_{t=1}^{T}\left\langle \widehat{F(x_{t})},x_{t}-y\right\rangle \right)}_{\text{stochastic regret}}\\
 & +\underbrace{R\left\Vert \sum_{t=1}^{T}\left(\widehat{F(x_{t})}-F(x_{t})\right)\right\Vert +\sum_{t=1}^{T}\left\langle \widehat{F(x_{t})}-F(x_{t}),x_{t}-x_{0}\right\rangle }_{\text{stochastic error}}
\end{align*}
 We analyze the stochastic regret similarly to Section \ref{sec:adapeg-stoch-regret-analysis}.
We split the regret as follows:
\begin{align*}
\left\langle \widehat{F(x_{t})},x_{t}-y\right\rangle  & =\left\langle \widehat{F(x_{t})},z_{t}-y\right\rangle +\left\langle \widehat{F(x_{t})}-\widehat{F(z_{t-1})},x_{t}-z_{t}\right\rangle +\left\langle \widehat{F(z_{t-1})},x_{t}-z_{t}\right\rangle 
\end{align*}
We analyze each term in turn, using arguments analogous to Lemmas
\ref{lem:regret-first-term}, \ref{lem:regret-third-term}, \ref{lem:regret-second-term}.
By the optimality condition for $z_{t}$, we have
\begin{align*}
\left\langle \widehat{F(x_{t})},z_{t}-y\right\rangle  & \leq\frac{1}{2}\left(\gamma_{t}-\gamma_{t-1}\right)\left\Vert x_{t}-y\right\Vert ^{2}+\frac{1}{2}\gamma_{t-1}\left\Vert z_{t-1}-y\right\Vert ^{2}-\frac{1}{2}\gamma_{t}\left\Vert z_{t}-y\right\Vert ^{2}\\
 & -\frac{1}{2}\left(\gamma_{t}-\gamma_{t-1}\right)\left\Vert x_{t}-z_{t}\right\Vert ^{2}-\frac{1}{2}\gamma_{t-1}\left\Vert z_{t-1}-z_{t}\right\Vert ^{2}
\end{align*}
By the optimality condition for $x_{t}$, we have
\[
\left\langle \widehat{F(z_{t-1})},x_{t}-z_{t}\right\rangle \leq\frac{1}{2}\gamma_{t-1}\left\Vert z_{t-1}-z_{t}\right\Vert ^{2}-\frac{1}{2}\gamma_{t-1}\left\Vert z_{t-1}-x_{t}\right\Vert ^{2}-\frac{1}{2}\gamma_{t-1}\left\Vert x_{t}-z_{t}\right\Vert ^{2}
\]
Using Cauchy-Schwartz and an argument analogous to Lemma \ref{lem:regret-second-term},
we obtain
\begin{align*}
\left\langle \widehat{F(x_{t})}-\widehat{F(z_{t-1})},x_{t}-z_{t}\right\rangle  & \leq\left\Vert \widehat{F(x_{t})}-\widehat{F(z_{t-1})}\right\Vert \left\Vert x_{t}-z_{t}\right\Vert \\
 & \leq\frac{1}{\gamma_{t}}\left\Vert \widehat{F(x_{t})}-\widehat{F(z_{t-1})}\right\Vert ^{2}
\end{align*}
Combining, we obtain
\begin{align*}
\left\langle \widehat{F(x_{t})},x_{t}-y\right\rangle  & \leq\frac{1}{2}\left(\gamma_{t}-\gamma_{t-1}\right)\left\Vert x_{t}-y\right\Vert ^{2}+\frac{1}{2}\gamma_{t-1}\left\Vert z_{t-1}-y\right\Vert ^{2}-\frac{1}{2}\gamma_{t}\left\Vert z_{t}-y\right\Vert ^{2}\\
 & +\frac{1}{\gamma_{t}}\left\Vert \widehat{F(x_{t})}-\widehat{F(z_{t-1})}\right\Vert ^{2}-\frac{1}{2}\gamma_{t}\left\Vert x_{t}-z_{t}\right\Vert ^{2}-\frac{1}{2}\gamma_{t-1}\left\Vert x_{t}-z_{t-1}\right\Vert ^{2}
\end{align*}
We sum up over all iterations and telescope the sums analogously to
Lemma \ref{lem:regret-combined}, and obtain
\begin{align*}
\sum_{t=1}^{T}\left\langle \widehat{F(x_{t})},x_{t}-y\right\rangle  & \leq\frac{1}{2}R^{2}\gamma_{0}+\left(\frac{1}{2}\frac{R^{2}}{\eta}+2\eta\right)\sqrt{\sum_{t=1}^{T}\left\Vert \widehat{F(x_{t})}-\widehat{F(z_{t-1})}\right\Vert ^{2}}\\
 & -\frac{1}{2}\sum_{t=1}^{T}\left(\gamma_{t}\left\Vert x_{t}-z_{t}\right\Vert ^{2}+\gamma_{t-1}\left\Vert x_{t}-z_{t-1}\right\Vert ^{2}\right)
\end{align*}
Thus we obtain
\begin{align*}
T\cdot\err(\avx_{T}) & \leq\underbrace{\left(\frac{1}{2}\frac{R^{2}}{\eta}+2\eta\right)\sqrt{\sum_{t=1}^{T}\left\Vert \widehat{F(x_{t})}-\widehat{F(z_{t-1})}\right\Vert ^{2}}}_{\text{loss}}\\
 & -\underbrace{\frac{1}{2}\sum_{t=1}^{T}\left(\gamma_{t}\left\Vert x_{t}-z_{t}\right\Vert ^{2}+\gamma_{t-1}\left\Vert x_{t}-z_{t-1}\right\Vert ^{2}\right)}_{\text{gain}}\\
 & +\underbrace{R\left\Vert \sum_{t=1}^{T}\left(\widehat{F(x_{t})}-F(x_{t})\right)\right\Vert +\sum_{t=1}^{T}\left\langle \widehat{F(x_{t})}-F(x_{t}),x_{t}-x_{0}\right\rangle }_{\text{stochastic error}}\\
 & +\frac{1}{2}R^{2}\gamma_{0}
\end{align*}
The analysis of the net loss is simpler than the one in Section \ref{sec:adapeg-loss-analysis}.
For non-smooth operators, we upper bound the loss as in Lemma \ref{lem:net-loss-nonsmooth}:
\begin{align*}
 & \sqrt{\sum_{t=1}^{T}\left\Vert \widehat{F(x_{t})}-\widehat{F(z_{t-1})}\right\Vert ^{2}}\\
 & =\sqrt{\sum_{t=1}^{T}\left\Vert F(x_{t})+\left(\widehat{F(x_{t})}-F(x_{t})\right)-F(z_{t-1})-\left(\widehat{F(z_{t-1})}-F(z_{t-1})\right)\right\Vert ^{2}}\\
 & \leq\sqrt{\sum_{t=1}^{T}\left(4\left\Vert F(x_{t})\right\Vert ^{2}+4\left\Vert \widehat{F(x_{t})}-F(x_{t})\right\Vert ^{2}+4\left\Vert F(z_{t-1})\right\Vert ^{2}+4\left\Vert \widehat{F(z_{t-1})}-F(z_{t-1})\right\Vert ^{2}\right)}\\
 & \leq2\sqrt{2}G\sqrt{T}+2\sqrt{\sum_{t=1}^{T}\left\Vert \widehat{F(x_{t})}-F(x_{t})\right\Vert ^{2}}+2\sqrt{\sum_{t=0}^{T-1}\left\Vert \widehat{F(z_{t})}-F(z_{t})\right\Vert ^{2}}
\end{align*}
 where we let $G=\max_{x\in\dom}\left\Vert F(x)\right\Vert $.

For smooth operators, we use the second gain term to offset the loss.
By smoothness, we have
\[
\left\Vert x_{t}-z_{t-1}\right\Vert ^{2}\geq\frac{1}{\beta^{2}}\left\Vert F(x_{t})-F(z_{t-1})\right\Vert ^{2}
\]
 and thus
\begin{align*}
 & r\sqrt{\sum_{t=1}^{T}\left\Vert F(x_{t})-F(z_{t-1})\right\Vert ^{2}}-\frac{1}{2}\sum_{t=1}^{T}\gamma_{t-1}\left\Vert x_{t}-z_{t-1}\right\Vert ^{2}\\
 & \leq r\sqrt{\sum_{t=1}^{T}\left\Vert F(x_{t})-F(z_{t-1})\right\Vert ^{2}}-\frac{1}{2}\sum_{t=1}^{T}\frac{\gamma_{t-1}}{\beta^{2}}\left\Vert F(x_{t})-F(z_{t-1})\right\Vert ^{2}
\end{align*}
 where we let $r:=\frac{1}{2}\frac{R^{2}}{\eta}+2\eta$. We proceed
as in the proof of Lemma \ref{lem:net-loss-smooth}, and obtain
\begin{align*}
 & r\sqrt{\sum_{t=1}^{T}\left\Vert \widehat{F(x_{t})}-\widehat{F(z_{t-1})}\right\Vert ^{2}}-\frac{1}{2}\sum_{t=1}^{T}\gamma_{t-1}\frac{1}{\beta^{2}}\left\Vert F(x_{t})-F(z_{t-1})\right\Vert ^{2}\\
 & \leq\beta r\left(2\sqrt{r\eta}+\sqrt{2}R\right)+2r\left(\sqrt{\sum_{t=1}^{T}\left\Vert \widehat{F(x_{t})}-F(x_{t})\right\Vert ^{2}}+\sqrt{\sum_{t=0}^{T-1}\left\Vert \widehat{F(z_{t})}-F(z_{t})\right\Vert ^{2}}\right)
\end{align*}
Finally, we analyze the stochastic error analogously to Section \ref{sec:adapeg-stoch-error-analysis}.

\section{Algorithms for Bregman distances}

\label{sec:adapeg-bregman}

\begin{algorithm}[H]
\caption{AdaPEG-Bregman algorithm for bounded domains $\protect\dom$. $D_{\psi}(x,y)=\psi(x)-\psi(y)-\left\langle \nabla\psi(y),x-y\right\rangle $
is the Bregman divergence of a strongly convex function $\psi$.}
\label{alg:adapeg-bregman}

Let $x_{0}=z_{0}\in\dom$, $\gamma_{0}\geq0$, $\eta>0$.

For $t=1,\dots,T$, update:
\begin{align*}
x_{t} & =\arg\min_{u\in\dom}\left\{ \left\langle \widehat{F(x_{t-1})},u\right\rangle +\gamma_{t-1}D_{\psi}(u,z_{t-1})\right\} \\
z_{t} & =\arg\min_{u\in\dom}\left\{ \left\langle \widehat{F(x_{t})},u\right\rangle +\gamma_{t-1}D_{\psi}(u,z_{t-1})+\left(\gamma_{t}-\gamma_{t-1}\right)D_{\psi}(u,x_{t})\right\} \\
\gamma_{t} & =\frac{1}{\eta}\sqrt{\eta^{2}\gamma_{0}^{2}+\sum_{s=1}^{t}\left\Vert \widehat{F(x_{s})}-\widehat{F(x_{s-1})}\right\Vert ^{2}}
\end{align*}

Return $\overline{x}_{T}=\frac{1}{T}\sum_{t=1}^{T}x_{t}$.
\end{algorithm}

\begin{algorithm}
\caption{AdaPEG-Bregman algorithm for unbounded domains $\protect\dom$. $D_{\psi}(x,y)=\psi(x)-\psi(y)-\left\langle \nabla\psi(y),x-y\right\rangle $
is the Bregman divergence of a strongly convex function $\psi$.}

\label{alg:adapeg-unbounded-bregman}

Let $x_{0}=z_{0}\in\dom$, $\gamma_{0}\geq0$, $\gamma_{-1}=0$, $\eta>0$.

For $t=1,\dots,T$, update:
\begin{align*}
x_{t} & =\arg\min_{u\in\dom}\left\{ \left\langle \widehat{F(x_{t-1})},u\right\rangle +\gamma_{t-2}D_{\psi}(u,z_{t-1})+\left(\gamma_{t-1}-\gamma_{t-2}\right)D_{\psi}(u,x_{0})\right\} \\
z_{t} & =\arg\min_{u\in\dom}\left\{ \left\langle \widehat{F(x_{t})},u\right\rangle +\gamma_{t-2}D_{\psi}(u,z_{t-1})+\left(\gamma_{t-1}-\gamma_{t-2}\right)D_{\psi}(u,x_{0})\right\} \\
\gamma_{t} & =\frac{1}{\eta}\sqrt{\eta^{2}\gamma_{0}^{2}+\sum_{s=1}^{t}\left\Vert \widehat{F(x_{s})}-\widehat{F(x_{s-1})}\right\Vert ^{2}}
\end{align*}

Return $\overline{x}_{T}=\frac{1}{T}\sum_{t=1}^{T}x_{t}$.
\end{algorithm}

In this section, we extend our main algorithms, Algorithm \ref{alg:adapeg}
and Algorithm \ref{alg:adapeg-unbounded}, to general Bregman distances.
Let $\psi$ be a strongly convex function. The Bregman divergence
of $\psi$ is defined as follows:
\[
D_{\psi}(x,y)=\psi(x)-\psi(y)-\left\langle \nabla\psi(y),x-y\right\rangle 
\]
Following \citep{Nesterov07}, we define the restricted error function
as follows. Let $x_{0}\in\dom$ be an arbitrary point. For any fixed
positive value $D$, we define
\begin{equation}
\err_{D}(x)=\sup_{y\in\dom}\left\{ \left\langle F(y),x-y\right\rangle \colon D_{\psi}\left(y,x_{0}\right)\leq D\right\} \label{eq:restricted-error-fn-bregman}
\end{equation}
Lemma \ref{lem:error-function} holds for Bregman distances and it
justifies the use of the error function to analyze convergence.

The extensions of our main algorithms to Bregman divergences are shown
in Algorithm \ref{alg:adapeg-bregman} and Algorithm \ref{alg:adapeg-unbounded-bregman}.
The following theorems state their convergence guarantees. Their analysis
is a straightforward extension of the analyses from Sections \ref{sec:adapeg-analysis}
and \ref{sec:adapeg-unbounded-analysis}. We give the analysis for
Algorithm \ref{alg:adapeg-bregman} below. The analysis of Algorithm
\ref{alg:adapeg-unbounded-bregman} follows similarly, and we omit
it.
\begin{thm}
\label{thm:adapeg-bregman-convergence} Let $F$ be a monotone operator.
Let $\avx_{T}$ be the solution returned by Algorithm \ref{alg:adapeg-bregman}.
Let $R^{2}\geq2\max_{x,y\in\dom}D_{\psi}(x,y)$ and suppose we set
$\eta=R$. If $F$ is non-smooth, we have
\[
\E\left[\err(\avx_{T})\right]\leq O\left(\frac{\gamma_{0}R^{2}}{T}+\frac{R\left(G+\sigma\right)}{\sqrt{T}}\right)
\]
where $G=\max_{x\in\dom}\left\Vert F(x)\right\Vert $ and $\sigma^{2}$
is the variance parameter from assumption (\ref{eq:stoch-assumption-variance}).

If $F$ is $\beta$-smooth with respect to the $\ell_{2}$-norm, we
have
\[
\E\left[\err(\avx_{T})\right]\leq O\left(\frac{\left(\beta+\gamma_{0}\right)R^{2}}{T}+\frac{R\sigma}{\sqrt{T}}\right)
\]
\end{thm}
\begin{thm}
\label{thm:adapeg-unbounded-bregman-convergence} Let $F$ be a monotone
operator. Let $D>0$ be any fixed positive value. Let $\eta=\Theta(D)$.
Let $\avx_{T}$ be the solution returned by Algorithm \ref{alg:adapeg-unbounded-bregman}.
If $F$ is non-smooth, we have
\[
\E\left[\err_{D}(\avx_{T})\right]\leq O\left(\frac{\gamma_{0}D^{2}+\gamma_{0}^{-1}G^{2}}{T}+\frac{DG+\left(D+\gamma_{0}^{-1}\right)\sigma}{\sqrt{T}}\right)
\]
where $G=\max_{x\in\dom}\left\Vert F(x)\right\Vert $ and $\sigma^{2}$
is the variance parameter from assumption (\ref{eq:stoch-assumption-variance}).

If $F$ is $\beta$-smooth with respect to the $\ell_{2}$-norm, we
have
\[
\E\left[\err_{D}(\avx_{T})\right]\leq O\left(\frac{\left(\beta+\gamma_{0}\right)D^{2}+DG+\gamma_{0}^{-1}G^{2}}{T}+\frac{\left(D+\gamma_{0}^{-1}\right)\sigma}{\sqrt{T}}\right)
\]
\end{thm}

\subsection{Analysis of algorithm \ref{alg:adapeg-bregman}}

The analysis is an extension of the analysis of Algorithm \ref{alg:adapeg}
given in Section \ref{sec:adapeg-analysis}. In the following, we
show that we can obtain a lemma that is identical to Lemma \ref{lem:regret-combined}.
Given this, the rest of the analysis follows as in Section \ref{sec:adapeg-analysis}.

As before, we split the regret as follows:

\begin{equation}
\left\langle \widehat{F(x_{t})},x_{t}-y\right\rangle =\left\langle \widehat{F(x_{t})},z_{t}-y\right\rangle +\left\langle \widehat{F(x_{t})}-\widehat{F(x_{t-1})},x_{t}-z_{t}\right\rangle +\left\langle \widehat{F(x_{t-1})},x_{t}-z_{t}\right\rangle \label{eq:regret-split-A}
\end{equation}
We analyze each term in turn. The argument is a straightforward extension
of Lemmas \ref{lem:regret-first-term}, \ref{lem:regret-third-term},
\ref{lem:regret-second-term}. We will use the following well-known
identity, which follows from the definition of the Bregman divergence:
\begin{equation}
\left\langle \nabla_{x}D_{\psi}(x,y),z-x\right\rangle =D_{\psi}(z,y)-D_{\psi}(x,y)-D_{\psi}(z,x)\label{eq:bregman-identity}
\end{equation}
Note also that, in this section, $R$ is a value satisfying $D_{\psi}(x,y)\leq\frac{1}{2}R^{2}$
for all $x,y\in\dom$. Since $\psi$ is strongly convex, we have $\frac{1}{2}\left\Vert x-y\right\Vert ^{2}\leq D_{\psi}(x,y)$
and thus $\left\Vert x-y\right\Vert ^{2}\leq R^{2}$. 
\begin{lem}
\label{lem:regret-term1-A}For any $y\in\dom$, we have
\begin{align*}
\left\langle \widehat{F(x_{t})},z_{t}-y\right\rangle  & \leq\gamma_{t-1}\left(D_{\psi}(y,z_{t-1})-D_{\psi}(z_{t},z_{t-1})-D_{\psi}(y,z_{t})\right)\\
 & +\left(\gamma_{t}-\gamma_{t-1}\right)\left(D_{\psi}(y,x_{t})-D_{\psi}(z_{t},x_{t})-D_{\psi}(y,z_{t})\right)
\end{align*}
\end{lem}
\begin{proof}
By the optimality condition for $z_{t}$, we have
\[
\left\langle \widehat{F(x_{t})}+\gamma_{t-1}\nabla_{x}D_{\psi}(z_{t},z_{t-1})+\left(\gamma_{t}-\gamma_{t-1}\right)\nabla_{x}D_{\psi}(z_{t},x_{t}),z_{t}-y\right\rangle \leq0
\]
By rearranging and using (\ref{eq:bregman-identity}), we obtain
\begin{align*}
\left\langle \widehat{F(x_{t})},z_{t}-y\right\rangle  & \leq\gamma_{t-1}\left\langle \nabla_{x}D_{\psi}(z_{t},z_{t-1}),y-z_{t}\right\rangle +\left(\gamma_{t}-\gamma_{t-1}\right)\left\langle \nabla_{x}D_{\psi}(z_{t},x_{t}),y-z_{t}\right\rangle \\
 & =\gamma_{t-1}\left(D_{\psi}(y,z_{t-1})-D_{\psi}(z_{t},z_{t-1})-D_{\psi}(y,z_{t})\right)\\
 & +\left(\gamma_{t}-\gamma_{t-1}\right)\left(D_{\psi}(y,x_{t})-D_{\psi}(z_{t},x_{t})-D_{\psi}(y,z_{t})\right)
\end{align*}
as needed.
\end{proof}
\begin{lem}
\label{lem:regret-term3-A} We have 
\[
\left\langle \widehat{F(x_{t-1})},x_{t}-z_{t}\right\rangle \leq\gamma_{t-1}\left(D_{\psi}(z_{t},z_{t-1})-D_{\psi}(x_{t},z_{t-1})-D_{\psi}(z_{t},x_{t})\right)
\]
\end{lem}
\begin{proof}
By the optimality condition for $x_{t}$, we have
\[
\left\langle \widehat{F(x_{t-1})}+\gamma_{t-1}\nabla_{x}D_{\psi}(x_{t},z_{t-1}),x_{t}-z_{t}\right\rangle \leq0
\]
By rearranging and using (\ref{eq:bregman-identity}), we obtain
\begin{align*}
\left\langle \widehat{F(x_{t-1})},x_{t}-z_{t}\right\rangle  & \leq\gamma_{t-1}\left\langle \nabla_{x}D_{\psi}(x_{t},z_{t-1}),z_{t}-x_{t}\right\rangle \\
 & =\gamma_{t-1}\left(D_{\psi}(z_{t},z_{t-1})-D_{\psi}(x_{t},z_{t-1})-D_{\psi}(z_{t},x_{t})\right)
\end{align*}
as needed.
\end{proof}
\begin{lem}
\label{lem:regret-term2-A}We have
\[
\left\langle \widehat{F(x_{t})}-\widehat{F(x_{t-1})},x_{t}-z_{t}\right\rangle \leq\frac{1}{\gamma_{t}}\left\Vert \widehat{F(x_{t})}-\widehat{F(x_{t-1})}\right\Vert ^{2}
\]
\end{lem}
\begin{proof}
Let
\[
\phi_{t}(u)=\left\langle \widehat{F(x_{t-1})},u\right\rangle +\gamma_{t-1}D_{\psi}(u,z_{t-1})+\left(\gamma_{t}-\gamma_{t-1}\right)D_{\psi}(u,x_{t})
\]
Note that $x_{t}$ is the minimizer of $D_{\psi}(u,x_{t})$, since
$D_{\psi}(x_{t},x_{t})=0$ and the Bregman divergence is non-negative
by convexity of $\psi$. Since $x_{t}$ is the minimizer of both $\left\langle \widehat{F(x_{t-1})},u\right\rangle +\gamma_{t-1}D_{\psi}(u,z_{t-1})$
and $\left(\gamma_{t}-\gamma_{t-1}\right)D_{\psi}(u,x_{t})$, we have
\[
x_{t}=\arg\min_{u\in X}\phi_{t}(u)
\]
Moreover
\[
z_{t}=\arg\min_{u\in X}\left\{ \phi_{t}(u)+\left\langle \widehat{F(x_{t})}-\widehat{F(x_{t-1})},u\right\rangle \right\} 
\]
By Lemma \ref{lem:danskin}, for all $v$, we have
\[
\nabla\phi_{t}^{*}(v)=\arg\min_{u\in X}\left\{ \phi_{t}(u)-\left\langle u,v\right\rangle \right\} 
\]
Thus
\begin{align*}
x_{t} & =\nabla\phi_{t}^{*}(0)\\
z_{t} & =\nabla\phi_{t}^{*}\left(-\left(\widehat{F(x_{t})}-\widehat{F(x_{t-1})}\right)\right)
\end{align*}
Since $\psi$ is strongly convex, $\phi_{t}$ is $\gamma_{t}$-strongly
convex and thus Lemma \ref{lem:duality} implies that $\phi_{t}^{*}$
is $\frac{1}{\gamma_{t}}$-smooth. Thus
\begin{align*}
\left\Vert x_{t}-z_{t}\right\Vert  & =\left\Vert \nabla\phi_{t}^{*}(0)-\nabla\phi_{t}^{*}\left(-\left(\widehat{F(x_{t})}-\widehat{F(x_{t-1})}\right)\right)\right\Vert \\
 & \leq\frac{1}{\gamma_{t}}\left\Vert \widehat{F(x_{t})}-\widehat{F(x_{t-1})}\right\Vert 
\end{align*}
Using Cauchy-Schwartz and the above inequality, we obtain
\begin{align*}
\left\langle \widehat{F(x_{t})}-\widehat{F(x_{t-1})},x_{t}-z_{t}\right\rangle  & \leq\left\Vert \widehat{F(x_{t})}-\widehat{F(x_{t-1})}\right\Vert \left\Vert x_{t}-z_{t}\right\Vert \\
 & \leq\frac{1}{\gamma_{t}}\left\Vert \widehat{F(x_{t})}-\widehat{F(x_{t-1})}\right\Vert ^{2}
\end{align*}
as needed.
\end{proof}
By combining (\ref{eq:regret-split-A}) with Lemmas \ref{lem:regret-term1-A},
\ref{lem:regret-term3-A}, \ref{lem:regret-term2-A} and summing up
over all iterations, we obtain the following bound.
\begin{lem}
\label{lem:regret-combined-common-A}For any $y\in\dom$, we have
\begin{align*}
\sum_{t=1}^{T}\left\langle \widehat{F(x_{t})},x_{t}-y\right\rangle  & \leq\frac{1}{2}\gamma_{0}R^{2}+\frac{5}{2}R\sqrt{\sum_{t=1}^{T}\left\Vert \widehat{F(x_{t})}-\widehat{F(x_{t-1})}\right\Vert ^{2}}\\
 & -\frac{1}{2}\sum_{t=1}^{T}\gamma_{t-1}\left(\left\Vert x_{t}-z_{t-1}\right\Vert ^{2}+\left\Vert x_{t-1}-z_{t-1}\right\Vert ^{2}\right)
\end{align*}
\end{lem}
\begin{proof}
By combining (\ref{eq:regret-split-A}) with Lemmas \ref{lem:regret-term1-A},
\ref{lem:regret-term3-A}, \ref{lem:regret-term2-A}, we obtain
\begin{align*}
\left\langle \widehat{F(x_{t})},x_{t}-y\right\rangle  & \leq\left(\gamma_{t}-\gamma_{t-1}\right)D_{\psi}(y,x_{t})+\gamma_{t-1}D_{\psi}(y,z_{t-1})-\gamma_{t}D_{\psi}(y,z_{t})\\
 & -\gamma_{t}D_{\psi}(z_{t},x_{t})-\gamma_{t-1}D_{\psi}(x_{t},z_{t-1})+\frac{1}{\gamma_{t}}\left\Vert \widehat{F(x_{t})}-\widehat{F(x_{t-1})}\right\Vert ^{2}
\end{align*}
Summing up over all iterations and using $0\leq D_{\psi}(x,y)\leq\frac{1}{2}R^{2}$
for $x,y\in\dom$, we obtain
\begin{align*}
\sum_{t=1}^{T}\left\langle \widehat{F(x_{t})},x_{t}-y\right\rangle  & \leq\frac{1}{2}R^{2}\left(\gamma_{T}-\gamma_{0}\right)+\gamma_{0}\underbrace{D_{\psi}(y,z_{0})}_{\le\frac{1}{2}R^{2}}-\gamma_{T}D_{\psi}(y,z_{T})\\
 & +\sum_{t=1}^{T}\frac{1}{\gamma_{t}}\left\Vert \widehat{F(x_{t})}-\widehat{F(x_{t-1})}\right\Vert ^{2}-\sum_{t=1}^{T}\gamma_{t}D_{\psi}(z_{t},x_{t})-\sum_{t=1}^{T}\gamma_{t-1}D_{\psi}(x_{t},z_{t-1})\\
 & \leq\frac{1}{2}R^{2}\gamma_{T}+\sum_{t=1}^{T}\frac{1}{\gamma_{t}}\left\Vert \widehat{F(x_{t})}-\widehat{F(x_{t-1})}\right\Vert ^{2}\\
 & -\sum_{t=1}^{T}\gamma_{t}D_{\psi}(z_{t},x_{t})-\sum_{t=1}^{T}\gamma_{t-1}D_{\psi}(x_{t},z_{t-1})
\end{align*}
Since $\psi$ is strongly convex, we have $D_{\psi}(x,y)\geq\frac{1}{2}\left\Vert x-y\right\Vert ^{2}$
for all $x,y\in\dom$. Thus
\begin{align}
 & \sum_{t=1}^{T}\left\langle \widehat{F(x_{t})},x_{t}-y\right\rangle \nonumber \\
 & \leq\frac{1}{2}R^{2}\gamma_{T}+\sum_{t=1}^{T}\frac{1}{\gamma_{t}}\left\Vert \widehat{F(x_{t})}-\widehat{F(x_{t-1})}\right\Vert ^{2}-\frac{1}{2}\sum_{t=1}^{T}\gamma_{t}\left\Vert x_{t}-z_{t}\right\Vert ^{2}-\sum_{t=1}^{T}\gamma_{t-1}\left\Vert x_{t}-z_{t-1}\right\Vert ^{2}\nonumber \\
 & =\frac{1}{2}R^{2}\gamma_{T}+\sum_{t=1}^{T}\frac{1}{\gamma_{t}}\left\Vert \widehat{F(x_{t})}-\widehat{F(x_{t-1})}\right\Vert ^{2}-\frac{1}{2}\sum_{t=1}^{T}\gamma_{t-1}\left(\left\Vert x_{t}-z_{t-1}\right\Vert ^{2}+\left\Vert x_{t-1}-z_{t-1}\right\Vert ^{2}\right)\nonumber \\
 & +\gamma_{0}\underbrace{\left\Vert x_{0}-z_{0}\right\Vert ^{2}}_{=0}-\gamma_{T}\left\Vert x_{T}-z_{T}\right\Vert \nonumber \\
 & \leq\frac{1}{2}R^{2}\gamma_{T}+\sum_{t=1}^{T}\frac{1}{\gamma_{t}}\left\Vert \widehat{F(x_{t})}-\widehat{F(x_{t-1})}\right\Vert ^{2}-\frac{1}{2}\sum_{t=1}^{T}\gamma_{t-1}\left(\left\Vert x_{t}-z_{t-1}\right\Vert ^{2}+\left\Vert x_{t-1}-z_{t-1}\right\Vert ^{2}\right)\label{eq:combined1-A}
\end{align}
Using Lemma \ref{lem:ineq}, we obtain
\begin{align}
\sum_{t=1}^{T}\frac{1}{\gamma_{t}}\left\Vert \widehat{F(x_{t})}-\widehat{F(x_{t-1})}\right\Vert ^{2} & =R\sum_{t=1}^{T}\frac{\left\Vert \widehat{F(x_{t})}-\widehat{F(x_{t-1})}\right\Vert ^{2}}{\sqrt{R^{2}\gamma_{0}^{2}+\sum_{s=1}^{t}\left\Vert \widehat{F(x_{s})}-\widehat{F(x_{s-1})}\right\Vert ^{2}}}\nonumber \\
 & \leq R\sum_{t=1}^{T}\frac{\left\Vert \widehat{F(x_{t})}-\widehat{F(x_{t-1})}\right\Vert ^{2}}{\sqrt{\sum_{s=1}^{t}\left\Vert \widehat{F(x_{s})}-\widehat{F(x_{s-1})}\right\Vert ^{2}}}\nonumber \\
 & \le2R\sqrt{\sum_{t=1}^{T}\left\Vert \widehat{F(x_{t})}-\widehat{F(x_{t-1})}\right\Vert ^{2}}\label{eq:combined2-A}
\end{align}
Additionally,
\begin{align}
\frac{1}{2}R^{2}\gamma_{T} & =\frac{1}{2}R\sqrt{R^{2}\gamma_{0}^{2}+\sum_{t=1}^{T}\left\Vert \widehat{F(x_{t})}-\widehat{F(x_{t-1})}\right\Vert ^{2}}\leq\frac{1}{2}R^{2}\gamma_{0}+\frac{1}{2}R\sqrt{\sum_{t=1}^{T}\left\Vert \widehat{F(x_{t})}-\widehat{F(x_{t-1})}\right\Vert ^{2}}\label{eq:combined3-A}
\end{align}
Plugging (\ref{eq:combined2-A}) and (\ref{eq:combined3-A}) into
(\ref{eq:combined1-A}), we obtain
\begin{align*}
\sum_{t=1}^{T}\left\langle \widehat{F(x_{t})},x_{t}-y\right\rangle  & \leq\frac{1}{2}\gamma_{0}R^{2}+\frac{5}{2}R\sqrt{\sum_{t=1}^{T}\left\Vert \widehat{F(x_{t})}-\widehat{F(x_{t-1})}\right\Vert ^{2}}\\
 & -\frac{1}{2}\sum_{t=1}^{T}\gamma_{t-1}\left(\left\Vert x_{t}-z_{t-1}\right\Vert ^{2}+\left\Vert x_{t-1}-z_{t-1}\right\Vert ^{2}\right)
\end{align*}
as needed.
\end{proof}
Thus we have obtained the same result as Lemma \ref{lem:regret-combined}.
Thus we can proceed as in Section \ref{sec:adapeg-analysis}. 

\section{Algorithms with per-coordinate step sizes}

\label{sec:vector-algorithms}

\begin{algorithm}
\caption{AdaPEG-Vector algorithm for bounded domains $\protect\dom$.}

\label{alg:adapeg-vector}

Let $x_{0}=z_{0}\in\dom$, $\D_{0}=\gamma_{0}I$ for some $\gamma_{0}\geq0$,
$\eta>0$.

For $t=1,\dots,T$, update:
\begin{align*}
x_{t} & =\arg\min_{u\in\dom}\left\{ \left\langle \widehat{F(x_{t-1})},u\right\rangle +\frac{1}{2}\left\Vert u-z_{t-1}\right\Vert _{\D_{t-1}}^{2}\right\} \\
z_{t} & =\arg\min_{u\in\dom}\left\{ \left\langle \widehat{F(x_{t})},u\right\rangle +\frac{1}{2}\left\Vert u-z_{t-1}\right\Vert _{\D_{t-1}}^{2}+\frac{1}{2}\left\Vert u-x_{t}\right\Vert _{\D_{t}-\D_{t-1}}^{2}\right\} \\
\D_{t,i} & =\frac{1}{\eta}\sqrt{\eta^{2}\D_{0,i}^{2}+\sum_{s=1}^{t}\left(\left(\widehat{F(x_{s})}\right)_{i}-\left(\widehat{F(x_{s-1})}\right)_{i}\right)^{2}} & \forall i\in[d]
\end{align*}

Return $\avx_{T}:=\frac{1}{T}\sum_{t=1}^{T}x_{t}$
\end{algorithm}

\begin{algorithm}
\caption{AdaPEG-Vector algorithm for unbounded domains $\protect\dom$.}
\label{alg:adapeg-vector-unbounded}

Let $x_{0}=z_{0}\in\dom$, $\D_{0}=\gamma_{0}I$ for some $\gamma_{0}\geq0$,
$\D_{-1}=0$, $\eta>0$.

For $t=1,\dots,T$, update:
\begin{align*}
x_{t} & =\arg\min_{u\in\dom}\left\{ \left\langle \widehat{F(x_{t-1})},u\right\rangle +\frac{1}{2}\left\Vert u-z_{t-1}\right\Vert _{\D_{t-2}}^{2}+\frac{1}{2}\left\Vert u-x_{0}\right\Vert _{\D_{t-1}-\D_{t-2}}^{2}\right\} \\
z_{t} & =\arg\min_{u\in\dom}\left\{ \left\langle \widehat{F(x_{t})},u\right\rangle +\frac{1}{2}\left\Vert u-z_{t-1}\right\Vert _{\D_{t-2}}^{2}+\frac{1}{2}\left\Vert u-x_{0}\right\Vert _{\D_{t-1}-\D_{t-2}}^{2}\right\} \\
\D_{t,i} & =\frac{1}{\eta}\sqrt{\eta^{2}\D_{0,i}^{2}+\sum_{s=1}^{t}\left(\left(\widehat{F(x_{s})}\right)_{i}-\left(\widehat{F(x_{s-1})}\right)_{i}\right)^{2}} & \forall i\in[d]
\end{align*}

Return $\overline{x}_{T}=\frac{1}{T}\sum_{t=1}^{T}x_{t}$.
\end{algorithm}

In this section, we extend our main algorithms, Algorithm \ref{alg:adapeg}
and Algorithm \ref{alg:adapeg-unbounded}, to the setting where we
use per-coordinate step sizes. The algorithms are shown in Algorithms
\ref{alg:adapeg-vector} and \ref{alg:adapeg-vector-unbounded}. We
give the analysis for Algorithm \ref{alg:adapeg-vector} below, which
extends the analysis from Section \ref{sec:adapeg-analysis}. The
analysis of Algorithm \ref{alg:adapeg-vector-unbounded} follows similarly
from the analysis given in Section \ref{sec:adapeg-unbounded-analysis},
and we omit it.

\subsection{Analysis of algorithm \ref{alg:adapeg-vector}}

\label{sec:adapeg-vector-analysis}

In this section, we analyze Algorithm \ref{alg:adapeg-vector}. The
analysis builds on the analysis from Section \ref{sec:adapeg-analysis}.
Throughout this section, the norm $\left\Vert \cdot\right\Vert $
without a subscript denotes the $\ell_{2}$-norm and $R\geq\max_{x,y}\left\Vert x-y\right\Vert $.
The following theorem states the  convergence guarantee. We note that
the convergence for smooth operators has a sub-optimal dependence
on the smoothness parameter ($\beta^{2}$ instead of the optimal $\beta$).
In Section \ref{sec:adapeg-vector-cocoercive}, we provide a stronger
guarantee with optimal dependence on the smoothness for operators
that are cocoercive.
\begin{thm}
\label{thm:adapeg-vector-convergence} Let $F$ be a monotone operator.
Let $\avx_{T}$ be the solution returned by Algorithm \ref{alg:adapeg-vector}.
Let $R_{\infty}\geq\max_{x,y\in\dom}\left\Vert x-y\right\Vert _{\infty}$,
and suppose that we set $\eta=R_{\infty}$. If $F$ is non-smooth,
we have
\[
\E\left[\err(\avx_{T})\right]\le O\left(\frac{\gamma_{0}dR_{\infty}^{2}}{T}+\frac{\sqrt{d}R_{\infty}G+\left(\sqrt{d}R_{\infty}+R\right)\sigma}{\sqrt{T}}\right)
\]
where $G=\max_{x\in\dom}\left\Vert F(x)\right\Vert $ and $\sigma^{2}$
is the variance parameter from assumption (\ref{eq:stoch-assumption-variance}).

If $F$ is $\beta$-smooth with respect to the $\ell_{2}$-norm, we
have
\[
\E\left[\err(\avx_{T})\right]\leq O\left(\frac{\gamma_{0}dR_{\infty}^{2}+\gamma_{0}^{-1}dR_{\infty}^{2}\beta^{2}}{T}+\frac{\left(\sqrt{d}R_{\infty}+R\right)\sigma}{\sqrt{T}}\right)
\]
\end{thm}
As before, for notational convenience, we let $\xi_{t}=F(x_{t})-\widehat{F(x_{t})}$.
By Lemma \ref{lem:error-fn-ub}, we have
\begin{equation}
T\cdot\err(\avx_{T})\leq\underbrace{\sup_{y\in\dom}\left(\sum_{t=1}^{T}\left\langle \widehat{F(x_{t})},x_{t}-y\right\rangle \right)}_{\text{stochastic regret}}+\underbrace{R\left\Vert \sum_{t=1}^{T}\xi_{t}\right\Vert +\sum_{t=1}^{T}\left\langle \xi_{t},x_{t}-x_{0}\right\rangle }_{\text{stochastic error}}\label{eq:vector-error-fn-ub-B}
\end{equation}
We analyze each of the two terms in turn.

\subsubsection{Analysis of the stochastic regret}

Here we analyze the stochastic regret in (\ref{eq:vector-error-fn-ub-B}):
\[
\underbrace{\sup_{y\in\dom}\left(\sum_{t=1}^{T}\left\langle \widehat{F(x_{t})},x_{t}-y\right\rangle \right)}_{\text{stochastic regret}}
\]
We fix an arbitrary $y\in\dom$, and we analyze the stochastic regret
$\sum_{t=1}^{T}\left\langle \widehat{F(x_{t})},x_{t}-y\right\rangle $.
We split the inner product $\left\langle \widehat{F(x_{t})},x_{t}-y\right\rangle $
as follows:
\begin{equation}
\left\langle \widehat{F(x_{t})},x_{t}-y\right\rangle =\left\langle \widehat{F(x_{t})},z_{t}-y\right\rangle +\left\langle \widehat{F(x_{t})}-\widehat{F(x_{t-1})},x_{t}-z_{t}\right\rangle +\left\langle \widehat{F(x_{t-1})},x_{t}-z_{t}\right\rangle \label{eq:vector-regret-split-B}
\end{equation}
We upper bound each term in turn. For the first term, we apply the
optimality condition for $z_{t}$ and obtain:
\begin{lem}
\label{lem:vector-regret-term1-B}For all $y\in\dom$, we have
\begin{align*}
\left\langle \widehat{F(x_{t})},z_{t}-y\right\rangle  & \leq\frac{1}{2}\left\Vert x_{t}-y\right\Vert _{\D_{t}-\D_{t-1}}^{2}+\frac{1}{2}\left\Vert z_{t-1}-y\right\Vert _{\D_{t-1}}^{2}\\
 & -\frac{1}{2}\left\Vert z_{t}-y\right\Vert _{\D_{t}}^{2}-\frac{1}{2}\left\Vert z_{t-1}-z_{t}\right\Vert _{\D_{t-1}}^{2}-\frac{1}{2}\left\Vert x_{t}-z_{t}\right\Vert _{\D_{t}-\D_{t-1}}^{2}
\end{align*}
\end{lem}
\begin{proof}
By the optimality condition for $z_{t}$, for all $u\in\dom$, we
have
\begin{align*}
\left\langle \widehat{F(x_{t})}+\D_{t-1}(z_{t}-z_{t-1})+(\D_{t}-\D_{t-1})(z_{t}-x_{t}),z_{t}-u\right\rangle  & \le0
\end{align*}
We apply the above inequality with $u=y$ and obtain
\[
\left\langle \widehat{F(x_{t})}+\D_{t-1}(z_{t}-z_{t-1})+(\D_{t}-\D_{t-1})(z_{t}-x_{t}),z_{t}-y\right\rangle \le0
\]
By rearranging the above inequality and using the identity $ab=\frac{1}{2}\left(\left(a+b\right)^{2}-a^{2}-b^{2}\right)$,
we obtain
\begin{align*}
\left\langle \widehat{F(x_{t})},z_{t}-y\right\rangle  & \le\left\langle (\D_{t}-\D_{t-1})(x_{t}-z_{t}),z_{t}-y\right\rangle +\left\langle \D_{t-1}(z_{t-1}-z_{t}),z_{t}-y\right\rangle \\
 & =\frac{1}{2}\left(\left\Vert x_{t}-y\right\Vert _{\D_{t}-\D_{t-1}}^{2}-\left\Vert x_{t}-z_{t}\right\Vert _{\D_{t}-\D_{t-1}}^{2}-\left\Vert z_{t}-y\right\Vert _{\D_{t}-\D_{t-1}}^{2}\right)\\
 & +\frac{1}{2}\left(\left\Vert z_{t-1}-y\right\Vert _{\D_{t-1}}^{2}-\left\Vert z_{t-1}-z_{t}\right\Vert _{\D_{t-1}}^{2}-\left\Vert z_{t}-y\right\Vert _{\D_{t-1}}^{2}\right)\\
 & =\frac{1}{2}\left\Vert x_{t}-y\right\Vert _{\D_{t}-\D_{t-1}}^{2}+\frac{1}{2}\left\Vert z_{t-1}-y\right\Vert _{\D_{t-1}}^{2}-\frac{1}{2}\left\Vert z_{t}-y\right\Vert _{\D_{t}}^{2}\\
 & -\frac{1}{2}\left\Vert z_{t-1}-z_{t}\right\Vert _{\D_{t-1}}^{2}-\frac{1}{2}\left\Vert x_{t}-z_{t}\right\Vert _{\D_{t}-\D_{t-1}}^{2}
\end{align*}
as needed.
\end{proof}

For the third term, we apply the optimality condition for $x_{t}$
and obtain:
\begin{lem}
\label{lem:vector-regret-term3-B}We have
\[
\left\langle \widehat{F(x_{t-1})},x_{t}-z_{t}\right\rangle \leq\frac{1}{2}\left\Vert z_{t-1}-z_{t}\right\Vert _{\D_{t-1}}^{2}-\frac{1}{2}\left\Vert z_{t-1}-x_{t}\right\Vert _{\D_{t-1}}^{2}-\frac{1}{2}\left\Vert x_{t}-z_{t}\right\Vert _{\D_{t-1}}^{2}
\]
\end{lem}
\begin{proof}
By the optimality condition for $x_{t}$, for all $u\in\dom$, we
have
\[
\left\langle \widehat{F(x_{t-1})}+\D_{t-1}\left(x_{t}-z_{t-1}\right),x_{t}-u\right\rangle \leq0
\]
We apply the above inequality with $u=z_{t}$ and obtain
\[
\left\langle \widehat{F(x_{t-1})}+\D_{t-1}\left(x_{t}-z_{t-1}\right),x_{t}-z_{t}\right\rangle \leq0
\]
By rearranging the above inequality and using the identity $ab=\frac{1}{2}\left(\left(a+b\right)^{2}-a^{2}-b^{2}\right)$,
we obtain
\begin{align*}
\left\langle \widehat{F(x_{t-1})},x_{t}-z_{t}\right\rangle  & \leq\left\langle \D_{t-1}(z_{t-1}-x_{t}),x_{t}-z_{t}\right\rangle \\
 & =\frac{1}{2}\left(\left\Vert z_{t-1}-z_{t}\right\Vert _{\D_{t-1}}^{2}-\left\Vert z_{t-1}-x_{t}\right\Vert _{\D_{t-1}}^{2}-\left\Vert x_{t}-z_{t}\right\Vert _{\D_{t-1}}^{2}\right)
\end{align*}
\end{proof}

We now analyze the second term. The argument is inspired by the work
of \citet{MohriYang16} for online convex minimization. We make careful
use of the definition of $z_{t}$ and duality and obtain the following
guarantee:
\begin{lem}
\label{lem:vector-regret-term2-B}We have
\[
\left\langle \widehat{F(x_{t})}-\widehat{F(x_{t-1})},x_{t}-z_{t}\right\rangle \leq\left\Vert \widehat{F(x_{t})}-\widehat{F(x_{t-1})}\right\Vert _{\D_{t}^{-1}}^{2}
\]
\end{lem}
\begin{proof}
Let
\[
\phi_{t}(u)=\left\langle \widehat{F(x_{t-1})},u\right\rangle +\frac{1}{2}\left\Vert u-z_{t-1}\right\Vert _{\D_{t-1}}^{2}+\frac{1}{2}\left\Vert u-x_{t}\right\Vert _{\D_{t}-\D_{t-1}}^{2}
\]
Since $x_{t}$ is the minimizer of both $\left\langle \widehat{F(x_{t-1})},u\right\rangle +\frac{1}{2}\left\Vert u-z_{t-1}\right\Vert _{\D_{t-1}}^{2}$
and $\frac{1}{2}\left\Vert u-x_{t}\right\Vert _{\D_{t}-\D_{t-1}}^{2}$,
we have
\[
x_{t}=\arg\min_{u\in\dom}\phi_{t}(u)
\]
Moreover
\[
z_{t}=\arg\min_{u\in\dom}\left\{ \phi_{t}(u)+\left\langle \widehat{F(x_{t})}-\widehat{F(x_{t-1})},u\right\rangle \right\} 
\]
By Lemma \ref{lem:danskin}, for all $v$, we have
\[
\nabla\phi_{t}^{*}(v)=\arg\min_{u\in\dom}\left\{ \phi_{t}(u)-\left\langle u,v\right\rangle \right\} 
\]
Thus
\begin{align*}
x_{t} & =\nabla\phi_{t}^{*}(0)\\
z_{t} & =\nabla\phi_{t}^{*}\left(-\left(\widehat{F(x_{t})}-\widehat{F(x_{t-1})}\right)\right)
\end{align*}
Since $\phi_{t}$ is strongly convex with respect to $\left\Vert \cdot\right\Vert _{\D_{t}}$,
$\phi_{t}^{*}$ is smooth with respect to $\left\Vert \cdot\right\Vert _{\D_{t}^{-1}}$
and thus
\begin{align*}
\left\Vert x_{t}-z_{t}\right\Vert _{\D_{t}} & =\left\Vert \nabla\phi_{t}^{*}(0)-\nabla\phi_{t}^{*}\left(-\left(\widehat{F(x_{t})}-\widehat{F(x_{t-1})}\right)\right)\right\Vert _{\D_{t}}\\
 & \leq\left\Vert \widehat{F(x_{t})}-\widehat{F(x_{t-1})}\right\Vert _{\D_{t}^{-1}}
\end{align*}
Using Holder's inequality and the above inequality, we obtain
\begin{align*}
\left\langle \widehat{F(x_{t})}-\widehat{F(x_{t-1})},x_{t}-z_{t}\right\rangle  & \leq\left\Vert \widehat{F(x_{t})}-\widehat{F(x_{t-1})}\right\Vert _{\D_{t}^{-1}}\left\Vert x_{t}-z_{t}\right\Vert _{\D_{t}}\\
 & \leq\left\Vert \widehat{F(x_{t})}-\widehat{F(x_{t-1})}\right\Vert _{\D_{t}^{-1}}^{2}
\end{align*}
as needed.
\end{proof}

We now combine (\ref{eq:vector-regret-split-B}) with Lemmas \ref{lem:vector-regret-term1-B},
\ref{lem:vector-regret-term3-B}, \ref{lem:vector-regret-term2-B}.
By summing up over all iterations and telescoping the sums appropriately,
we obtain:
\begin{lem}
\label{lem:vector-regret-combined-B}For all $y\in\dom$, we have
\begin{align*}
\sum_{t=1}^{T}\left\langle \widehat{F(x_{t})},x_{t}-y\right\rangle  & \leq R_{\infty}^{2}\tr(\D_{0})+3R_{\infty}\sum_{i=1}^{d}\sqrt{\sum_{t=1}^{T}\left(\left(\widehat{F(x_{t})}\right)_{i}-\left(\widehat{F(x_{t-1})}\right)_{i}\right)^{2}}\\
 & -\frac{1}{2}\sum_{t=1}^{T}\left(\left\Vert x_{t}-z_{t}\right\Vert _{\D_{t}}^{2}+\left\Vert x_{t}-z_{t-1}\right\Vert _{\D_{t}}^{2}\right)
\end{align*}
\end{lem}
\begin{proof}
By plugging in the guarantees provided Lemmas \ref{lem:vector-regret-term1-B},
\ref{lem:vector-regret-term3-B}, \ref{lem:vector-regret-term2-B}
into (\ref{eq:vector-regret-split-B}), we obtain
\begin{align*}
 & \left\langle \widehat{F(x_{t})},x_{t}-y\right\rangle \\
 & \leq\frac{1}{2}\left\Vert x_{t}-y\right\Vert _{\D_{t}-\D_{t-1}}^{2}+\frac{1}{2}\left\Vert z_{t-1}-x_{t}\right\Vert _{\D_{t}-\D_{t-1}}^{2}+\frac{1}{2}\left\Vert z_{t-1}-y\right\Vert _{\D_{t-1}}^{2}-\frac{1}{2}\left\Vert z_{t}-y\right\Vert _{\D_{t}}^{2}\\
 & +\left\Vert \widehat{F(x_{t})}-\widehat{F(x_{t-1})}\right\Vert _{\D_{t}^{-1}}^{2}-\frac{1}{2}\left\Vert x_{t}-z_{t}\right\Vert _{\D_{t}}^{2}-\frac{1}{2}\left\Vert x_{t}-z_{t-1}\right\Vert _{\D_{t}}^{2}
\end{align*}
Summing up over all iterations and using the inequality $\left\Vert x-y\right\Vert _{\D}^{2}\leq R_{\infty}^{2}\tr(\D)$
for $x,y\in\dom$, we obtain
\begin{align}
 & \sum_{t=1}^{T}\left\langle \widehat{F(x_{t})},x_{t}-y\right\rangle \nonumber \\
 & \leq R_{\infty}^{2}\left(\tr(\D_{T})-\tr(\D_{0})\right)+\frac{1}{2}\left\Vert z_{0}-y\right\Vert _{\D_{0}}^{2}-\frac{1}{2}\left\Vert z_{T}-y\right\Vert _{\D_{T}}^{2}\nonumber \\
 & +\sum_{t=1}^{T}\left\Vert \widehat{F(x_{t})}-\widehat{F(x_{t-1})}\right\Vert _{\D_{t}^{-1}}^{2}-\frac{1}{2}\sum_{t=1}^{T}\left(\left\Vert x_{t}-z_{t}\right\Vert _{\D_{t}}^{2}+\left\Vert x_{t}-z_{t-1}\right\Vert _{\D_{t}}^{2}\right)\nonumber \\
 & \leq R_{\infty}^{2}\tr(\D_{T})+\sum_{t=1}^{T}\left\Vert \widehat{F(x_{t})}-\widehat{F(x_{t-1})}\right\Vert _{\D_{t}^{-1}}^{2}-\frac{1}{2}\sum_{t=1}^{T}\left(\left\Vert x_{t}-z_{t}\right\Vert _{\D_{t}}^{2}+\left\Vert x_{t}-z_{t-1}\right\Vert _{\D_{t}}^{2}\right)\label{eq:stoch-combined1}
\end{align}
Recall that $\eta=R_{\infty}$. We apply Lemma \ref{lem:ineq} to
each coordinate separately, and obtain
\begin{align}
 & \sum_{t=1}^{T}\left\Vert \widehat{F(x_{t})}-\widehat{F(x_{t-1})}\right\Vert _{\D_{t}^{-1}}^{2}\nonumber \\
 & =\sum_{i=1}^{d}\sum_{t=1}^{T}\frac{\left(\left(\widehat{F(x_{t})}\right)_{i}-\left(\widehat{F(x_{t-1})}\right)_{i}\right)^{2}}{\D_{t,i}}\nonumber \\
 & =R_{\infty}\sum_{i=1}^{d}\sum_{t=1}^{T}\frac{\left(\left(\widehat{F(x_{t})}\right)_{i}-\left(\widehat{F(x_{t-1})}\right)_{i}\right)^{2}}{\sqrt{R_{\infty}^{2}\D_{0,i}^{2}+\sum_{s=1}^{t}\left(\left(\widehat{F(x_{s})}\right)_{i}-\left(\widehat{F(x_{s-1})}\right)_{i}\right)^{2}}}\nonumber \\
 & \leq R_{\infty}\sum_{i=1}^{d}\sum_{t=1}^{T}\frac{\left(\left(\widehat{F(x_{t})}\right)_{i}-\left(\widehat{F(x_{t-1})}\right)_{i}\right)^{2}}{\sqrt{\sum_{s=1}^{t}\left(\left(\widehat{F(x_{s})}\right)_{i}-\left(\widehat{F(x_{s-1})}\right)_{i}\right)^{2}}}\nonumber \\
 & \leq2R_{\infty}\sum_{i=1}^{d}\sqrt{\sum_{t=1}^{T}\left(\left(\widehat{F(x_{t})}\right)_{i}-\left(\widehat{F(x_{t-1})}\right)_{i}\right)^{2}}\label{eq:stoch-combined2}
\end{align}
Additionally,
\begin{align}
R_{\infty}^{2}\tr(\D_{T}) & =R_{\infty}\sum_{i=1}^{d}\sqrt{R_{\infty}^{2}\D_{0,i}^{2}+\sum_{t=1}^{T}\left(\left(\widehat{F(x_{t})}\right)_{i}-\left(\widehat{F(x_{t-1})}\right)_{i}\right)^{2}}\nonumber \\
 & \leq R_{\infty}^{2}\tr(\D_{0})+R_{\infty}\sum_{i=1}^{d}\sqrt{\sum_{t=1}^{T}\left(\left(\widehat{F(x_{t})}\right)_{i}-\left(\widehat{F(x_{t-1})}\right)_{i}\right)^{2}}\label{eq:stoch-combined3}
\end{align}
Plugging (\ref{eq:stoch-combined2}) and (\ref{eq:stoch-combined3})
into (\ref{eq:stoch-combined1}), we obtain
\begin{align*}
\sum_{t=1}^{T}\left\langle \widehat{F(x_{t})},x_{t}-y\right\rangle  & \leq R_{\infty}^{2}\tr(\D_{0})+3R_{\infty}\sum_{i=1}^{d}\sqrt{\sum_{t=1}^{T}\left(\left(\widehat{F(x_{t})}\right)_{i}-\left(\widehat{F(x_{t-1})}\right)_{i}\right)^{2}}\\
 & -\frac{1}{2}\sum_{t=1}^{T}\left(\left\Vert x_{t}-z_{t}\right\Vert _{\D_{t}}^{2}+\left\Vert x_{t}-z_{t-1}\right\Vert _{\D_{t}}^{2}\right)
\end{align*}
as needed.
\end{proof}

We take the analysis one step further and further bound the main loss
term above as follows.
\begin{lem}
\label{lem:vector-main-error-B}We have
\begin{align*}
 & \sum_{i=1}^{d}\sqrt{\sum_{t=1}^{T}\left(\left(\widehat{F(x_{t})}\right)_{i}-\left(\widehat{F(x_{t-1})}\right)_{i}\right)^{2}}\\
 & \leq\sqrt{2}\sqrt{d}\sqrt{\sum_{t=1}^{T}\left\Vert F(x_{t})-F(x_{t-1})\right\Vert ^{2}}+2\sqrt{2}\sqrt{d}\sqrt{\sum_{t=0}^{T}\left\Vert \xi_{t}\right\Vert ^{2}}
\end{align*}
\end{lem}
\begin{proof}
Recall that $\xi_{t}=F(x_{t})-\widehat{F(x_{t})}$. Using concavity
of the square root, we obtain
\begin{align*}
 & \sum_{i=1}^{d}\sqrt{\sum_{t=1}^{T}\left(\left(\widehat{F(x_{t})}\right)_{i}-\left(\widehat{F(x_{t-1})}\right)_{i}\right)^{2}}\\
 & \leq\sqrt{d}\sqrt{\sum_{t=1}^{T}\left\Vert \widehat{F(x_{t})}-\widehat{F(x_{t-1})}\right\Vert ^{2}}\\
 & =\sqrt{d}\sqrt{\sum_{t=1}^{T}\left\Vert F(x_{t})-F(x_{t-1})+\xi_{t-1}-\xi_{t}\right\Vert ^{2}}\\
 & \leq\sqrt{d}\sqrt{\sum_{t=1}^{T}\left(2\left\Vert F(x_{t})-F(x_{t-1})\right\Vert ^{2}+4\left\Vert \xi_{t-1}\right\Vert ^{2}+4\left\Vert \xi_{t}\right\Vert ^{2}\right)}\\
 & \leq\sqrt{d}\sqrt{2\sum_{t=1}^{T}\left\Vert F(x_{t})-F(x_{t-1})\right\Vert ^{2}+8\sum_{t=0}^{T}\left\Vert \xi_{t}\right\Vert ^{2}}\\
 & \leq\sqrt{2}\sqrt{d}\sqrt{\sum_{t=1}^{T}\left\Vert F(x_{t})-F(x_{t-1})\right\Vert ^{2}}+2\sqrt{2}\sqrt{d}\sqrt{\sum_{t=0}^{T}\left\Vert \xi_{t}\right\Vert ^{2}}
\end{align*}
as needed.
\end{proof}

By combining Lemmas \ref{lem:vector-regret-combined-B} and \ref{lem:vector-main-error-B},
we obtain our upper bound on the stochastic regret.
\begin{lem}
\label{lem:vector-regret-main-B}For all $y\in\dom$, we have

\begin{align*}
\sum_{t=1}^{T}\left\langle \widehat{F(x_{t})},x_{t}-y\right\rangle  & \leq\underbrace{3\sqrt{2}\sqrt{d}R_{\infty}\sqrt{\sum_{t=1}^{T}\left\Vert F(x_{t})-F(x_{t-1})\right\Vert ^{2}}}_{\text{loss}}\\
 & -\underbrace{\frac{1}{2}\sum_{t=1}^{T}\left(\left\Vert x_{t}-z_{t}\right\Vert _{\D_{t}}^{2}+\left\Vert x_{t}-z_{t-1}\right\Vert _{\D_{t}}^{2}\right)}_{\text{gain}}\\
 & +R_{\infty}^{2}\tr(\D_{0})+6\sqrt{2}\sqrt{d}R_{\infty}\sqrt{\sum_{t=0}^{T}\left\Vert \xi_{t}\right\Vert ^{2}}
\end{align*}
\end{lem}
By plugging in Lemma \ref{lem:vector-regret-main-B} into (\ref{eq:vector-error-fn-ub-B}),
we obtain the following upper bound on the error function.
\begin{lem}
\label{lem:vector-error-fn-ub-refined-B}Let $\eta=R_{\infty}\geq\max_{x,y\in\dom}\left\Vert x-y\right\Vert _{\infty}$.
Let $\xi_{t}:=F(x_{t})-\widehat{F(x_{t})}$. We have
\begin{align*}
T\cdot\err(\avx_{T}) & \leq\underbrace{3\sqrt{2}\sqrt{d}R_{\infty}\sqrt{\sum_{t=1}^{T}\left\Vert F(x_{t})-F(x_{t-1})\right\Vert ^{2}}}_{\text{loss}}\\
 & -\underbrace{\frac{1}{2}\sum_{t=1}^{T}\left(\left\Vert x_{t}-z_{t}\right\Vert _{\D_{t}}^{2}+\left\Vert x_{t}-z_{t-1}\right\Vert _{\D_{t}}^{2}\right)}_{\text{gain}}\\
 & +\underbrace{6\sqrt{2}\sqrt{d}R_{\infty}\sqrt{\sum_{t=0}^{T}\left\Vert \xi_{t}\right\Vert ^{2}}+R\left\Vert \sum_{t=1}^{T}\xi_{t}\right\Vert +\sum_{t=1}^{T}\left\langle \xi_{t},x_{t}-x_{0}\right\rangle }_{\text{stochastic error}}\\
 & +R_{\infty}^{2}\tr(\D_{0})
\end{align*}
\end{lem}

\subsubsection{Analysis of the loss}

Here we analyze the loss and gain terms in the upper bound provided
by Lemma \ref{lem:vector-error-fn-ub-refined-B} above:
\[
\underbrace{3\sqrt{2}\sqrt{d}R_{\infty}\sqrt{\sum_{t=1}^{T}\left\Vert F(x_{t})-F(x_{t-1})\right\Vert ^{2}}}_{\text{loss}}-\underbrace{\frac{1}{2}\sum_{t=1}^{T}\left(\left\Vert x_{t}-z_{t}\right\Vert _{\D_{t}}^{2}+\left\Vert x_{t}-z_{t-1}\right\Vert _{\D_{t}}^{2}\right)}_{\text{gain}}
\]
For non-smooth operators, we ignore the gain term and bound the loss
term as in Lemma \ref{lem:net-loss-nonsmooth}. For smooth operators,
we use the gain term to balance the loss term.
\begin{lem}
\label{lem:vector-net-loss-nonsmooth-B}Suppose that $F$ is non-smooth
and let $G:=\max_{x\in\dom}\left\Vert F(x)\right\Vert $. We have
\[
\sqrt{\sum_{t=1}^{T}\left\Vert F(x_{t})-F(x_{t-1})\right\Vert ^{2}}\leq O\left(G\sqrt{T}\right)
\]
\end{lem}
\begin{proof}
We have
\begin{align*}
\sqrt{\sum_{t=1}^{T}\left\Vert F(x_{t})-F(x_{t-1})\right\Vert ^{2}} & \leq\sqrt{\sum_{t=1}^{T}\left(2\left\Vert F(x_{t})\right\Vert ^{2}+2\left\Vert F(x_{t-1})\right\Vert ^{2}\right)}\leq2G\sqrt{T}
\end{align*}
\end{proof}

\begin{lem}
\label{lem:vector-net-loss-smooth-B}Suppose that $F$ is $\beta$-smooth
with respect to the $\ell_{2}$-norm. We have
\begin{align*}
 & 3\sqrt{2}\sqrt{d}R_{\infty}\sqrt{\sum_{t=1}^{T}\left\Vert F(x_{t})-F(x_{t-1})\right\Vert ^{2}}-\frac{1}{2}\sum_{t=1}^{T}\left(\left\Vert x_{t}-z_{t}\right\Vert _{\D_{t}}^{2}+\left\Vert x_{t}-z_{t-1}\right\Vert _{\D_{t}}^{2}\right)\\
 & \leq O\left(\frac{dR_{\infty}^{2}\beta^{2}}{\gamma_{0}}\right)
\end{align*}
\end{lem}
\begin{proof}
Using smoothness and the inequality $(a+b)^{2}\leq2a^{2}+2b^{2}$,
we obtain
\begin{align*}
\left\Vert F(x_{t})-F(x_{t-1})\right\Vert ^{2} & \leq\beta^{2}\left\Vert x_{t}-x_{t-1}\right\Vert ^{2}\\
 & =\beta^{2}\left\Vert x_{t}-z_{t-1}+z_{t-1}-x_{t-1}\right\Vert ^{2}\\
 & \leq2\beta^{2}\left(\left\Vert x_{t}-z_{t-1}\right\Vert ^{2}+\left\Vert x_{t-1}-z_{t-1}\right\Vert ^{2}\right)
\end{align*}
Therefore
\[
\left\Vert x_{t}-z_{t-1}\right\Vert ^{2}+\left\Vert x_{t-1}-z_{t-1}\right\Vert ^{2}\geq\frac{1}{2\beta^{2}}\left\Vert F(x_{t})-F(x_{t-1})\right\Vert ^{2}
\]
Since $\D_{t}\succeq\D_{0}=\gamma_{0}I$, we have
\begin{equation}
\sum_{t=1}^{T}\left(\left\Vert x_{t}-z_{t}\right\Vert _{\D_{t}}^{2}+\left\Vert x_{t}-z_{t-1}\right\Vert _{\D_{t}}^{2}\right)\geq\gamma_{0}\sum_{t=1}^{T}\left(\left\Vert x_{t}-z_{t}\right\Vert ^{2}+\left\Vert x_{t}-z_{t-1}\right\Vert ^{2}\right)\label{eq:vector-smooth2}
\end{equation}
By combining the two inequalities, we obtain
\[
\sum_{t=1}^{T}\left(\left\Vert x_{t}-z_{t}\right\Vert _{\D_{t}}^{2}+\left\Vert x_{t}-z_{t-1}\right\Vert _{\D_{t}}^{2}\right)\geq\frac{\gamma_{0}}{2\beta^{2}}\sum_{t=1}^{T}\left\Vert F(x_{t})-F(x_{t-1})\right\Vert ^{2}
\]
Thus we can upper bound the net loss as follows:
\begin{align*}
 & 3\sqrt{2}\sqrt{d}R_{\infty}\sqrt{\sum_{t=1}^{T}\left\Vert F(x_{t})-F(x_{t-1})\right\Vert ^{2}}-\frac{1}{2}\sum_{t=1}^{T}\left(\left\Vert x_{t}-z_{t}\right\Vert _{\D_{t}}^{2}+\left\Vert x_{t}-z_{t-1}\right\Vert _{\D_{t}}^{2}\right)\\
 & \leq3\sqrt{2}\sqrt{d}R_{\infty}\sqrt{\sum_{t=1}^{T}\left\Vert F(x_{t})-F(x_{t-1})\right\Vert ^{2}}-\frac{\gamma_{0}}{4\beta^{2}}\sum_{t=1}^{T}\left\Vert F(x_{t})-F(x_{t-1})\right\Vert ^{2}\\
 & \leq\max_{z\geq0}\left\{ 3\sqrt{2}\sqrt{d}R_{\infty}z-\frac{\gamma_{0}}{4\beta^{2}}z^{2}\right\} \\
 & =\frac{18dR_{\infty}^{2}\beta^{2}}{\gamma_{0}}
\end{align*}
as needed.
\end{proof}

\subsubsection{Analysis of the stochastic error}

Using Lemmas \ref{lem:stoch-err1}, \ref{lem:stoch-err2}, \ref{lem:stoch-err3},
we obtain:
\begin{lem}
\label{lem:vector-stoch-err-combined-B}We have
\[
\E\left[6\sqrt{2}\sqrt{d}R_{\infty}\sqrt{\sum_{t=0}^{T}\left\Vert \xi_{t}\right\Vert ^{2}}+R\left\Vert \sum_{t=1}^{T}\xi_{t}\right\Vert +\sum_{t=1}^{T}\left\langle \xi_{t},x_{t}-x_{0}\right\rangle \right]\leq O\left(R\sigma\sqrt{T}\right)
\]
\end{lem}

\subsubsection{Putting everything together}

By combining Lemma \ref{lem:vector-error-fn-ub-refined-B} with Lemmas
\ref{lem:vector-net-loss-nonsmooth-B}, \ref{lem:vector-net-loss-smooth-B},
\ref{lem:vector-stoch-err-combined-B}, we obtain our final convergence
guarantee and complete the proof of Theorem \ref{thm:adapeg-vector-convergence}.
\begin{lem}
Suppose $F$ is non-smooth and let $G=\max_{x\in\dom}\left\Vert F(x)\right\Vert $.
We have
\[
\E\left[\err(\avx_{T})\right]\le O\left(\frac{\gamma_{0}dR_{\infty}^{2}}{T}+\frac{\sqrt{d}R_{\infty}G+\left(\sqrt{d}R_{\infty}+R\right)\sigma}{\sqrt{T}}\right)
\]
\end{lem}
\begin{proof}
The lemma follows by combining Lemmas \ref{lem:vector-error-fn-ub-refined-B}
, \ref{lem:vector-net-loss-nonsmooth-B}, \ref{lem:vector-stoch-err-combined-B}.
\end{proof}
\begin{lem}
Suppose $F$ is $\beta$-smooth. We have
\[
\E\left[\err(\avx_{T})\right]\leq O\left(\frac{\gamma_{0}dR_{\infty}^{2}+\gamma_{0}^{-1}dR_{\infty}^{2}\beta^{2}}{T}+\frac{\left(\sqrt{d}R_{\infty}+R\right)\sigma}{\sqrt{T}}\right)
\]
\end{lem}
\begin{proof}
The lemma follows by combining Lemmas \ref{lem:vector-error-fn-ub-refined-B}
, \ref{lem:vector-net-loss-smooth-B}, \ref{lem:vector-stoch-err-combined-B}.
\end{proof}

\subsection{Analysis of algorithm \ref{alg:adapeg-vector} for cocoercive operators}

\label{sec:adapeg-vector-cocoercive}

In this section, we provide a stronger convergence guarantee for Algorithm
\ref{alg:adapeg-vector} for smooth operators that are cocoercive.
The following theorem states the convergence guarantee. We note that
it has optimal dependence on the smoothness parameters.

Throughout this section, the norm $\left\Vert \cdot\right\Vert $
without a subscript denotes the $\ell_{2}$-norm and $R=\max_{x,y}\left\Vert x-y\right\Vert $.
As before, for notational convenience, we let $\xi_{t}=F(x_{t})-\widehat{F(x_{t})}$. 
\begin{thm}
\label{thm:adapeg-vector-cocoercive} Let $F$ be an operator that
is $1$-cocoercive with respect to a norm $\left\Vert \cdot\right\Vert _{\sm}$,
where $\sm=\diag\left(\beta_{1},\dots,\beta_{d}\right)$ is an unknown
diagonal matrix with $\beta_{1},\dots,\beta_{d}>0$. Let $\avx_{T}$
be the solution returned by Algorithm \ref{alg:adapeg-vector}. Let
$R_{\infty}\geq\max_{x,y\in\dom}\left\Vert x-y\right\Vert _{\infty}$,
and suppose that we set $\eta=R_{\infty}$. We have
\[
\E\left[\err(\avx_{T})\right]\leq O\left(\frac{\left(\tr(\sm)+\gamma_{0}d\right)R_{\infty}^{2}}{T}+\frac{\left(\sqrt{d}R_{\infty}+R\right)\sigma}{\sqrt{T}}\right)
\]
where $\sigma^{2}$ is the variance parameter from assumption (\ref{eq:stoch-assumption-variance}).
\end{thm}
The analysis is a strengthening of the analysis from Section \ref{sec:adapeg-vector-analysis}.
The key observation is that the cocoercive assumption allows us to
obtain the following stronger bound on the error function. The proof
is analogous to the proof of Lemma \ref{lem:error-fn-ub}, and it
uses the stronger cocoercivity assumption.
\begin{lem}
\label{lem:error-fn-ub-cocoercive}Let $R\geq\max_{x,y\in\dom}\left\Vert x-y\right\Vert $.
Let $\xi_{t}:=F(x_{t})-\widehat{F(x_{t})}$. We have
\begin{align*}
T\cdot\err(\avx_{T}) & \leq\sup_{y\in\dom}\left(\sum_{t=1}^{T}\left(\left\langle \widehat{F(x_{t})},x_{t}-y\right\rangle -\left\Vert F(x_{t})-F(y)\right\Vert _{\sm^{-1}}^{2}\right)\right)\\
 & +R\left\Vert \sum_{t=1}^{T}\xi_{t}\right\Vert +\sum_{t=1}^{T}\left\langle \xi_{t},x_{t}-x_{0}\right\rangle 
\end{align*}
\end{lem}
\begin{proof}
Using the definition of the error function (\ref{eq:restricted-error-fn})
and the cocoercivity property (\ref{eq:cocoercive-operator}), we
obtain
\begin{align*}
\err(\avx_{T}) & =\sup_{y\in\dom}\left\langle F(y),\avx_{T}-y\right\rangle \\
 & =\frac{1}{T}\sup_{y\in\dom}\left(\sum_{t=1}^{T}\left\langle F(y),x_{t}-y\right\rangle \right)\\
 & \leq\frac{1}{T}\sup_{y\in\dom}\left(\sum_{t=1}^{T}\left(\left\langle F(x_{t}),x_{t}-y\right\rangle -\left\Vert F(x_{t})-F(y)\right\Vert _{\sm^{-1}}^{2}\right)\right)
\end{align*}
We further write
\begin{align*}
\left\langle F(x_{t}),x_{t}-y\right\rangle  & =\left\langle \widehat{F(x_{t})},x_{t}-y\right\rangle +\left\langle F(x_{t})-\widehat{F(x_{t})},x_{t}-y\right\rangle \\
 & =\left\langle \widehat{F(x_{t})},x_{t}-y\right\rangle +\left\langle F(x_{t})-\widehat{F(x_{t})},x_{0}-y\right\rangle +\left\langle F(x_{t})-\widehat{F(x_{t})},x_{t}-x_{0}\right\rangle \\
 & =\left\langle \widehat{F(x_{t})},x_{t}-y\right\rangle +\left\langle \xi_{t},x_{0}-y\right\rangle +\left\langle \xi_{t},x_{t}-x_{0}\right\rangle 
\end{align*}
Thus we obtain
\begin{align*}
 & T\cdot\err(\avx_{T})\\
 & \leq\sup_{y\in\dom}\left(\sum_{t=1}^{T}\left(\left\langle \widehat{F(x_{t})},x_{t}-y\right\rangle -\left\Vert F(x_{t})-F(y)\right\Vert _{\sm^{-1}}^{2}\right)+\sum_{t=1}^{T}\left\langle \xi_{t},x_{0}-y\right\rangle \right)+\sum_{t=1}^{T}\left\langle \xi_{t},x_{t}-x_{0}\right\rangle \\
 & \leq\sup_{y\in\dom}\left(\sum_{t=1}^{T}\left(\left\langle \widehat{F(x_{t})},x_{t}-y\right\rangle -\left\Vert F(x_{t})-F(y)\right\Vert _{\sm^{-1}}^{2}\right)\right)\\
 & +\sup_{y\in\dom}\left(\sum_{t=1}^{T}\left\langle \xi_{t},x_{0}-y\right\rangle \right)+\sum_{t=1}^{T}\left\langle \xi_{t},x_{t}-x_{0}\right\rangle 
\end{align*}
Using the Cauchy-Schwartz inequality, we obtain the following upper
bound on the second term above:

\[
\left\langle \sum_{t=1}^{T}\xi_{t},x_{0}-y\right\rangle \leq\left\Vert \sum_{t=1}^{T}\xi_{t}\right\Vert \left\Vert x_{0}-y\right\Vert \leq\left\Vert \sum_{t=1}^{T}\xi_{t}\right\Vert R
\]
Therefore
\begin{align*}
T\cdot\err(\avx_{T}) & \leq\sup_{y\in\dom}\left(\sum_{t=1}^{T}\left(\left\langle \widehat{F(x_{t})},x_{t}-y\right\rangle -\left\Vert F(x_{t})-F(y)\right\Vert _{\sm^{-1}}^{2}\right)\right)\\
 & +\left\Vert \sum_{t=1}^{T}\xi_{t}\right\Vert R+\sum_{t=1}^{T}\left\langle \xi_{t},x_{t}-x_{0}\right\rangle 
\end{align*}
as needed.
\end{proof}

Next, we analyze the stochastic regret.
\begin{lem}
\label{lem:regret-cocoercive} For all $y\in\dom$, we have
\begin{align*}
 & \sum_{t=1}^{T}\left(\left\langle \widehat{F(x_{t})},x_{t}-y\right\rangle -\left\Vert F(x_{t})-F(y)\right\Vert _{\sm^{-1}}^{2}\right)\\
 & \leq R_{\infty}^{2}\tr(\D_{0})+\left(2R_{\infty}^{2}+2\sqrt{2}R_{\infty}\right)\tr(\sm)+2\sqrt{2}\sqrt{d}R_{\infty}\sqrt{\sum_{t=0}^{T}\left\Vert \xi_{t}\right\Vert ^{2}}
\end{align*}
\end{lem}
\begin{proof}
As before, we split the inner product $\left\langle \widehat{F(x_{t})},x_{t}-y\right\rangle $
as follows:
\[
\left\langle \widehat{F(x_{t})},x_{t}-y\right\rangle =\left\langle \widehat{F(x_{t})},z_{t}-y\right\rangle +\left\langle \widehat{F(x_{t})}-\widehat{F(x_{t-1})},x_{t}-z_{t}\right\rangle +\left\langle \widehat{F(x_{t-1})},x_{t}-z_{t}\right\rangle 
\]
We bound each of the above terms as in Lemmas \ref{lem:vector-regret-term1-B},
\ref{lem:vector-regret-term3-B}, \ref{lem:vector-regret-term2-B}.
We obtain
\begin{align*}
\left\langle \widehat{F(x_{t})},x_{t}-y\right\rangle  & \leq\frac{1}{2}\left\Vert x_{t}-y\right\Vert _{\D_{t}-\D_{t-1}}^{2}+\frac{1}{2}\left\Vert x_{t}-z_{t-1}\right\Vert _{\D_{t}-\D_{t-1}}^{2}\\
 & +\frac{1}{2}\left\Vert z_{t-1}-y\right\Vert _{\D_{t-1}}^{2}-\frac{1}{2}\left\Vert z_{t}-y\right\Vert _{\D_{t}}^{2}\\
 & +\left\Vert \widehat{F(x_{t})}-\widehat{F(x_{t-1})}\right\Vert _{\D_{t}^{-1}}^{2}-\frac{1}{2}\left\Vert x_{t}-z_{t}\right\Vert _{\D_{t}}^{2}-\frac{1}{2}\left\Vert x_{t}-z_{t-1}\right\Vert _{\D_{t}}^{2}
\end{align*}
We now proceed as in the proof of Lemma \ref{lem:vector-regret-combined-B}.
Summing up over all iterations and using the inequality $\left\Vert x-y\right\Vert _{\D}^{2}\leq R_{\infty}^{2}\tr(\D)$
for $x,y\in\dom$, we obtain
\begin{align}
 & \sum_{t=1}^{T}\left\langle \widehat{F(x_{t})},x_{t}-y\right\rangle \nonumber \\
 & \leq R_{\infty}^{2}\left(\tr(\D_{T})-\tr(\D_{0})\right)+\frac{1}{2}\left\Vert z_{0}-y\right\Vert _{\D_{0}}^{2}-\frac{1}{2}\left\Vert z_{T}-y\right\Vert _{\D_{T}}^{2}\nonumber \\
 & +\sum_{t=1}^{T}\left\Vert \widehat{F(x_{t})}-\widehat{F(x_{t-1})}\right\Vert _{\D_{t}^{-1}}^{2}-\frac{1}{2}\sum_{t=1}^{T}\left(\left\Vert x_{t}-z_{t}\right\Vert _{\D_{t}}^{2}+\left\Vert x_{t}-z_{t-1}\right\Vert _{\D_{t}}^{2}\right)\nonumber \\
 & \leq R_{\infty}^{2}\tr(\D_{T})+\sum_{t=1}^{T}\left\Vert \widehat{F(x_{t})}-\widehat{F(x_{t-1})}\right\Vert _{\D_{t}^{-1}}^{2}-\frac{1}{2}\sum_{t=1}^{T}\left(\left\Vert x_{t}-z_{t}\right\Vert _{\D_{t}}^{2}+\left\Vert x_{t}-z_{t-1}\right\Vert _{\D_{t}}^{2}\right)\label{eq:stoch-combined1-1}
\end{align}
Recall that $\eta=R_{\infty}$. We apply Lemma \ref{lem:ineq} to
each coordinate separately, and obtain
\begin{align}
 & \sum_{t=1}^{T}\left\Vert \widehat{F(x_{t})}-\widehat{F(x_{t-1})}\right\Vert _{\D_{t}^{-1}}^{2}\nonumber \\
 & =\sum_{i=1}^{d}\sum_{t=1}^{T}\frac{\left(\left(\widehat{F(x_{t})}\right)_{i}-\left(\widehat{F(x_{t-1})}\right)_{i}\right)^{2}}{\D_{t,i}}\nonumber \\
 & =R_{\infty}\sum_{i=1}^{d}\sum_{t=1}^{T}\frac{\left(\left(\widehat{F(x_{t})}\right)_{i}-\left(\widehat{F(x_{t-1})}\right)_{i}\right)^{2}}{\sqrt{R_{\infty}^{2}\D_{0,i}^{2}+\sum_{s=1}^{t}\left(\left(\widehat{F(x_{s})}\right)_{i}-\left(\widehat{F(x_{s-1})}\right)_{i}\right)^{2}}}\nonumber \\
 & \leq R_{\infty}\sum_{i=1}^{d}\sum_{t=1}^{T}\frac{\left(\left(\widehat{F(x_{t})}\right)_{i}-\left(\widehat{F(x_{t-1})}\right)_{i}\right)^{2}}{\sqrt{\sum_{s=1}^{t}\left(\left(\widehat{F(x_{s})}\right)_{i}-\left(\widehat{F(x_{s-1})}\right)_{i}\right)^{2}}}\nonumber \\
 & \leq2R_{\infty}\sum_{i=1}^{d}\sqrt{\sum_{t=1}^{T}\left(\left(\widehat{F(x_{t})}\right)_{i}-\left(\widehat{F(x_{t-1})}\right)_{i}\right)^{2}}\label{eq:stoch-combined2-1}
\end{align}
Additionally,
\begin{align}
R_{\infty}^{2}\tr(\D_{T}) & =R_{\infty}\sum_{i=1}^{d}\sqrt{R_{\infty}^{2}\D_{0,i}^{2}+\sum_{t=1}^{T}\left(\left(\widehat{F(x_{t})}\right)_{i}-\left(\widehat{F(x_{t-1})}\right)_{i}\right)^{2}}\nonumber \\
 & \leq R_{\infty}^{2}\tr(\D_{0})+R_{\infty}\sum_{i=1}^{d}\sqrt{\sum_{t=1}^{T}\left(\left(\widehat{F(x_{t})}\right)_{i}-\left(\widehat{F(x_{t-1})}\right)_{i}\right)^{2}}\label{eq:stoch-combined3-1}
\end{align}
Plugging (\ref{eq:stoch-combined2-1}) and (\ref{eq:stoch-combined3-1})
into (\ref{eq:stoch-combined1-1}) and dropping non-positive terms,
we obtain
\begin{align}
\sum_{t=1}^{T}\left\langle \widehat{F(x_{t})},x_{t}-y\right\rangle  & \leq R_{\infty}^{2}\tr(\D_{0})+3R_{\infty}\sum_{i=1}^{d}\sqrt{\sum_{t=1}^{T}\left(\left(\widehat{F(x_{t})}\right)_{i}-\left(\widehat{F(x_{t-1})}\right)_{i}\right)^{2}}\label{eq:stoch-combined4-1}
\end{align}
Recall that $\xi_{t}=F(x_{t})-\widehat{F(x_{t})}$. Using the inequality
$(a+b)^{2}\leq2a^{2}+2b^{2}$ and concavity of the square root, we
further bound
\begin{align*}
 & \sum_{i=1}^{d}\sqrt{\sum_{t=1}^{T}\left(\left(\widehat{F(x_{t})}\right)_{i}-\left(\widehat{F(x_{t-1})}\right)_{i}\right)^{2}}\\
 & =\sum_{i=1}^{d}\sqrt{\sum_{t=1}^{T}\left(\left(F(x_{t})\right)_{i}-\left(F(y)\right)_{i}+\left(F(y)\right)_{i}-\left(F(x_{t-1})\right)_{i}+\left(\xi_{t-1}\right)_{i}-\left(\xi_{t}\right)_{i}\right)^{2}}\\
 & \leq\sum_{i=1}^{d}\sqrt{4\sum_{t=1}^{T}\left(\left(\left(F(x_{t})\right)_{i}-\left(F(y)\right)_{i}\right)^{2}+\left(\left(F(y)\right)_{i}-\left(F(x_{t-1})\right)_{i}\right)^{2}+\left(\left(\xi_{t-1}\right)_{i}\right)^{2}+\left(\left(\xi_{t}\right)_{i}\right)^{2}\right)}\\
 & \leq\sum_{i=1}^{d}\sqrt{8\sum_{t=0}^{T}\left(\left(\left(F(x_{t})\right)_{i}-\left(F(y)\right)_{i}\right)^{2}+\left(\left(\xi_{t}\right)_{i}\right)^{2}\right)}\\
 & \leq2\sqrt{2}\sum_{i=1}^{d}\sqrt{\sum_{t=1}^{T}\left(\left(F(x_{t})\right)_{i}-\left(F(y)\right)_{i}\right)^{2}}+2\sqrt{2}\sum_{i=1}^{d}\left|\left(F(x_{0})\right)_{i}-\left(F(y)\right)_{i}\right|\\
 & +2\sqrt{2}\sum_{i=1}^{d}\sqrt{\sum_{t=0}^{T}\left(\left(\xi_{t}\right)_{i}\right)^{2}}
\end{align*}
Using Cauchy-Schwartz, the fact that $F$ is $1$-smooth with respect
to the norm $\left\Vert \cdot\right\Vert _{\sm}$, and the inequality
$\left\Vert x-y\right\Vert _{\sm}^{2}\leq R_{\infty}^{2}\tr(\sm)$
for $x,y\in\dom$, we obtain
\begin{align*}
\sum_{i=1}^{d}\left|\left(F(x_{0})\right)_{i}-\left(F(y)\right)_{i}\right| & \leq\sqrt{\sum_{i=1}^{d}\frac{1}{\beta_{i}}\left(\left(F(x_{0})\right)_{i}-\left(F(y)\right)_{i}\right)^{2}}\sqrt{\sum_{i=1}^{d}\beta_{i}}\\
 & =\sqrt{\left\Vert F(x_{0})-F(y)\right\Vert _{\sm^{-1}}^{2}}\sqrt{\tr(\sm)}\\
 & \leq\sqrt{\left\Vert x_{0}-y\right\Vert _{\sm}^{2}}\sqrt{\tr(\sm)}\\
 & \leq R_{\infty}\tr(\sm)
\end{align*}
Plugging into (\ref{eq:stoch-combined4-1}), we obtain
\begin{align*}
\sum_{t=1}^{T}\left\langle \widehat{F(x_{t})},x_{t}-y\right\rangle  & \leq R_{\infty}^{2}\tr(\D_{0})+2\sqrt{2}R_{\infty}\sum_{i=1}^{d}\sqrt{\sum_{t=1}^{T}\left(\left(F(x_{t})\right)_{i}-\left(F(y)\right)_{i}\right)^{2}}\\
 & +2\sqrt{2}R_{\infty}\tr(\sm)+2\sqrt{2}\sqrt{d}R_{\infty}\sqrt{\sum_{t=0}^{T}\left\Vert \xi_{t}\right\Vert ^{2}}
\end{align*}
Therefore
\begin{align*}
 & \sum_{t=1}^{T}\left(\left\langle \widehat{F(x_{t})},x_{t}-y\right\rangle -\left\Vert F(x_{t})-F(y)\right\Vert _{\sm^{-1}}^{2}\right)\\
 & \leq\sum_{i=1}^{d}\left(2\sqrt{2}R_{\infty}\sqrt{\sum_{t=1}^{T}\left(\left(F(x_{t})\right)_{i}-\left(F(y)\right)_{i}\right)^{2}}-\frac{1}{\beta_{i}}\sum_{t=1}^{T}\left(\left(F(x_{t})\right)_{i}-\left(F(y)\right)_{i}\right)^{2}\right)\\
 & +R_{\infty}^{2}\tr(\D_{0})+2\sqrt{2}R_{\infty}\tr(\sm)+2\sqrt{2}\sqrt{d}R_{\infty}\sqrt{\sum_{t=0}^{T}\left\Vert \xi_{t}\right\Vert ^{2}}\\
 & \leq\sum_{i=1}^{d}\max_{z\geq0}\left(2\sqrt{2}R_{\infty}z-\frac{1}{\beta_{i}}z^{2}\right)+R_{\infty}^{2}\tr(\D_{0})+2\sqrt{2}R_{\infty}\tr(\sm)+2\sqrt{2}\sqrt{d}R_{\infty}\sqrt{\sum_{t=0}^{T}\left\Vert \xi_{t}\right\Vert ^{2}}\\
 & =2R_{\infty}^{2}\underbrace{\sum_{i=1}^{d}\beta_{i}}_{=\tr(\sm)}+R_{\infty}^{2}\tr(\D_{0})+2\sqrt{2}R_{\infty}\tr(\sm)+2\sqrt{2}\sqrt{d}R_{\infty}\sqrt{\sum_{t=0}^{T}\left\Vert \xi_{t}\right\Vert ^{2}}\\
 & =R_{\infty}^{2}\tr(\D_{0})+\left(2R_{\infty}^{2}+2\sqrt{2}R_{\infty}\right)\tr(\sm)+2\sqrt{2}\sqrt{d}R_{\infty}\sqrt{\sum_{t=0}^{T}\left\Vert \xi_{t}\right\Vert ^{2}}
\end{align*}
as needed.
\end{proof}

We now combine Lemmas \ref{lem:error-fn-ub-cocoercive} and \ref{lem:regret-cocoercive},
and obtain
\begin{align*}
T\cdot\err(\avx_{T}) & \leq R_{\infty}^{2}\tr(\D_{0})+\left(2R_{\infty}^{2}+2\sqrt{2}R_{\infty}\right)\tr(\sm)\\
 & +2\sqrt{2}\sqrt{d}R_{\infty}\sqrt{\sum_{t=0}^{T}\left\Vert \xi_{t}\right\Vert ^{2}}+R\left\Vert \sum_{t=1}^{T}\xi_{t}\right\Vert +\sum_{t=1}^{T}\left\langle \xi_{t},x_{t}-x_{0}\right\rangle 
\end{align*}
By taking expectation and using Lemmas \ref{lem:stoch-err1}, \ref{lem:stoch-err2},
\ref{lem:stoch-err3}, we obtain

\[
\E\left[\err(\avx_{T})\right]\leq O\left(\frac{R_{\infty}^{2}\left(\tr(\sm)+\tr(\D_{0})\right)}{T}+\frac{\left(\sqrt{d}R_{\infty}+R\right)\sigma}{\sqrt{T}}\right)
\]

\section{Single-call variant of the algorithm of \citet{ene2020adaptive}}

\label{sec:adapeg-vector-movement-analysis}

\begin{algorithm}
\caption{A single-call variant of the algorithm of \citet{ene2020adaptive}.}

\label{alg:adapeg-vector-movement}

Let $x_{0}=z_{0}\in\dom$, $\D_{0}=\gamma_{0}I$ for some $\gamma_{0}\geq0$,
$R_{\infty}\geq\max_{x,y\in\dom}\left\Vert x-y\right\Vert _{\infty}$.

For $t=1,\dots,T$, update:

\begin{align*}
x_{t} & =\arg\min_{u\in\dom}\left\{ \left\langle \widehat{F(x_{t-1})},u\right\rangle +\frac{1}{2}\left\Vert u-z_{t-1}\right\Vert _{\D_{t-1}}^{2}\right\} \\
z_{t} & =\arg\min_{u\in\dom}\left\{ \left\langle \widehat{F(x_{t})},u\right\rangle +\frac{1}{2}\left\Vert u-z_{t-1}\right\Vert _{\D_{t-1}}^{2}\right\} \\
\D_{t,i}^{2} & =\D_{t-1,i}^{2}\left(1+\frac{\left(x_{t,i}-z_{t-1,i}\right)^{2}+\left(x_{t,i}-z_{t,i}\right)^{2}}{2R_{\infty}^{2}}\right) & \forall i\in[d]
\end{align*}

Return $\avx_{T}:=\frac{1}{T}\sum_{t=1}^{T}x_{t}$
\end{algorithm}

In this section, we build on the analysis from Section \ref{sec:adapeg-vector-analysis}
and the work of \citet{ene2020adaptive} in order to obtain an analysis
of a single-call variant of the algorithm of \citet{ene2020adaptive}
that uses a single operator evaluation per iteration instead of two.
The algorithm is shown in Algorithm \ref{alg:adapeg-vector-movement}.
The following theorem states its convergence guarantee. Throughout
this section, we let $\left\Vert \cdot\right\Vert $ denote the $\ell_{2}$-norm
and $R\geq\max_{x,y}\left\Vert x-y\right\Vert $. 
\begin{thm}
\label{thm:adapeg-vector-movement-convergence} Let $F$ be a monotone
operator. Let $\avx_{T}$ be the solution returned by Algorithm \ref{alg:adapeg-vector-movement}.
If $F$ is a non-smooth operator, we have
\[
\E\left[\err(\avx_{T})\right]\leq O\left(\frac{\gamma_{0}dR_{\infty}^{2}}{T}+\frac{\sqrt{d}R_{\infty}G\left(\sqrt{\ln\left(\frac{GT}{R_{\infty}}\right)}+\sqrt{\ln\left(\gamma_{0}^{-1}\right)}\right)}{\sqrt{T}}+\frac{R\sigma}{\sqrt{T}}\right)
\]

where $G=\max_{x\in\dom}\left\Vert F(x)\right\Vert $ and $\sigma^{2}$
is the variance parameter from assumption (\ref{eq:stoch-assumption-variance}).

If $F$ is $1$-smooth with respect to a norm $\left\Vert \cdot\right\Vert _{\sm}$,
where $\sm=\diag\left(\beta_{1},\dots,\beta_{d}\right)$ is an unknown
diagonal matrix with $\beta_{1},\dots,\beta_{d}\geq1$, we have
\[
\E\left[\err(\avx_{T})\right]\leq O\left(\frac{R_{\infty}^{2}\sum_{i=1}^{d}\beta_{i}\left(\ln\left(2\beta_{i}\right)+\ln\left(\gamma_{0}^{-1}\right)\right)+\gamma_{0}dR_{\infty}^{2}}{T}+\frac{R\sigma}{\sqrt{T}}\right)
\]
\end{thm}
We will use the following lemmas from \citep{ene2020adaptive}:
\begin{lem}
\emph{\citep{ene2020adaptive}}\label{lem:inequalities}Let $d_{1}^{2},d_{2}^{2},\dots,d_{T}^{2}$
and $R^{2}$ be scalars. Let $D_{0}>0$ and let $D_{1},\dots,D_{T}$
be defined according to the following recurrence
\[
D_{t}^{2}=D_{t-1}^{2}\left(1+\frac{d_{t}^{2}}{R^{2}}\right)
\]
We have
\[
\sum_{t=a}^{b}D_{t-1}\cdot d_{t}^{2}\geq2R^{2}\left(D_{b}-D_{a-1}\right)
\]
If $d_{t}^{2}\leq R^{2}$ for all $t$, we have
\begin{align*}
\sum_{t=a}^{b}D_{t-1}\cdot d_{t}^{2} & \leq\left(\sqrt{2}+1\right)R^{2}\left(D_{b}-D_{a-1}\right)\\
\sum_{t=a}^{b}d_{t}^{2} & \leq4R^{2}\ln\left(\frac{D_{b}}{D_{a-1}}\right)
\end{align*}
\end{lem}
\begin{lem}
\emph{\citep{ene2020adaptive}} \label{lem:phiz}Let $\phi\colon\R^{d}\to\R$,
$\phi(z)=a\sqrt{\sum_{i=1}^{d}\ln\left(z_{i}\right)}-\sum_{i=1}^{d}z_{i}$,
where $a$ is a non-negative scalar. Let $z^{*}\in\arg\max_{z\geq1}\phi(z)$.
We have
\[
\phi(z^{*})\leq\sqrt{d}a\sqrt{\ln a}
\]
\end{lem}
As before, for notational convenience, we let $\xi_{t}=F(x_{t})-\widehat{F(x_{t})}$.
By Lemma \ref{lem:error-fn-ub}, we have

\begin{equation}
\err(\avx_{T})\leq\frac{1}{T}\left(\sup_{y\in\dom}\left(\sum_{t=1}^{T}\left\langle \widehat{F(x_{t})},x_{t}-y\right\rangle \right)+R\left\Vert \sum_{t=1}^{T}\xi_{t}\right\Vert +\sum_{t=1}^{T}\left\langle \xi_{t},x_{t}-x_{0}\right\rangle \right)\label{eq:vector-error-fn-ub-C}
\end{equation}
We split the inner product $\left\langle \widehat{F(x_{t})},x_{t}-y\right\rangle $
as follows:

\begin{equation}
\left\langle \widehat{F(x_{t})},x_{t}-y\right\rangle =\left\langle \widehat{F(x_{t})},z_{t}-y\right\rangle +\left\langle \widehat{F(x_{t})}-\widehat{F(x_{t-1})},x_{t}-z_{t}\right\rangle +\left\langle \widehat{F(x_{t-1})},x_{t}-z_{t}\right\rangle \label{eq:vector-regret-split-C}
\end{equation}
We upper bound each term in turn. Using the optimality condition for
$z_{t}$, we obtain:
\begin{lem}
\label{lem:vector-regret-term1-C}For all $y\in\dom$, we have
\[
\left\langle \widehat{F(x_{t})},z_{t}-y\right\rangle \leq\frac{1}{2}\left\Vert z_{t-1}-y\right\Vert _{\D_{t-1}}^{2}-\frac{1}{2}\left\Vert z_{t}-y\right\Vert _{\D_{t-1}}^{2}-\frac{1}{2}\left\Vert z_{t-1}-z_{t}\right\Vert _{\D_{t-1}}^{2}
\]
\end{lem}
\begin{proof}
By the optimality condition for $z_{t}$, for all $u\in\dom$, we
have
\begin{align*}
\left\langle \widehat{F(x_{t})}+\D_{t-1}(z_{t}-z_{t-1}),z_{t}-u\right\rangle  & \le0
\end{align*}
We apply the above inequality with $u=y$ and obtain
\[
\left\langle \widehat{F(x_{t})}+\D_{t-1}(z_{t}-z_{t-1}),z_{t}-y\right\rangle \le0
\]
By rearranging the above inequality and using the identity $ab=\frac{1}{2}\left(\left(a+b\right)^{2}-a^{2}-b^{2}\right)$,
we obtain
\begin{align*}
\left\langle \widehat{F(x_{t})},z_{t}-y\right\rangle  & \le\left\langle \D_{t-1}(z_{t-1}-z_{t}),z_{t}-y\right\rangle \\
 & =\frac{1}{2}\left(\left\Vert z_{t-1}-y\right\Vert _{\D_{t-1}}^{2}-\left\Vert z_{t-1}-z_{t}\right\Vert _{\D_{t-1}}^{2}-\left\Vert z_{t}-y\right\Vert _{\D_{t-1}}^{2}\right)
\end{align*}
as needed.
\end{proof}

Using the optimality condition for $x_{t}$, we obtain:
\begin{lem}
\label{lem:vector-regret-term3-C}For all $y\in\dom$, we have
\[
\left\langle \widehat{F(x_{t-1})},x_{t}-z_{t}\right\rangle \leq\frac{1}{2}\left\Vert z_{t-1}-z_{t}\right\Vert _{\D_{t-1}}^{2}-\frac{1}{2}\left\Vert x_{t}-z_{t-1}\right\Vert _{\D_{t-1}}^{2}-\frac{1}{2}\left\Vert x_{t}-z_{t}\right\Vert _{\D_{t-1}}^{2}
\]
\end{lem}
\begin{proof}
By the optimality condition for $x_{t}$, for all $u\in\dom$, we
have
\[
\left\langle \widehat{F(x_{t-1})}+\D_{t-1}\left(x_{t}-z_{t-1}\right),x_{t}-u\right\rangle \leq0
\]
We apply the above inequality with $u=z_{t}$ and obtain
\[
\left\langle \widehat{F(x_{t-1})}+\D_{t-1}\left(x_{t}-z_{t-1}\right),x_{t}-z_{t}\right\rangle \leq0
\]
By rearranging the above inequality and using the identity $ab=\frac{1}{2}\left(\left(a+b\right)^{2}-a^{2}-b^{2}\right)$,
we obtain
\begin{align*}
\left\langle \widehat{F(x_{t-1})},x_{t}-z_{t}\right\rangle  & \leq\left\langle \D_{t-1}(z_{t-1}-x_{t}),x_{t}-z_{t}\right\rangle \\
 & =\frac{1}{2}\left(\left\Vert z_{t-1}-z_{t}\right\Vert _{\D_{t-1}}^{2}-\left\Vert z_{t-1}-x_{t}\right\Vert _{\D_{t-1}}^{2}-\left\Vert x_{t}-z_{t}\right\Vert _{\D_{t-1}}^{2}\right)
\end{align*}
as needed.
\end{proof}

Finally, we analyze the term $\left\langle \widehat{F(x_{t})}-\widehat{F(x_{t-1})},x_{t}-z_{t}\right\rangle $.
We do so separately for non-smooth and smooth operators. 
\begin{lem}
\label{lem:vector-regret-term2-C-nonsmooth}Suppose $F$ is non-smooth
and let $G=\max_{x\in\dom}\left\Vert F(x)\right\Vert $. We have
\[
\left\langle \widehat{F(x_{t})}-\widehat{F(x_{t-1})},x_{t}-z_{t}\right\rangle \leq2G\left\Vert x_{t}-z_{t}\right\Vert +\left(\left\Vert \xi_{t}\right\Vert +\left\Vert \xi_{t-1}\right\Vert \right)R
\]
\end{lem}
\begin{proof}
We write
\[
\left\langle \widehat{F(x_{t})}-\widehat{F(x_{t-1})},x_{t}-z_{t}\right\rangle =\left\langle F(x_{t})-F(x_{t-1}),x_{t}-z_{t}\right\rangle +\left\langle \xi_{t-1}-\xi_{t},x_{t}-z_{t}\right\rangle 
\]
Using Cauchy-Schwartz, we obtain
\begin{align*}
\left\langle F(x_{t})-F(x_{t-1}),x_{t}-z_{t}\right\rangle  & \le\left\Vert F(x_{t})-F(x_{t-1})\right\Vert \left\Vert x_{t}-z_{t}\right\Vert \\
 & \leq\left(\left\Vert F(x_{t})\right\Vert +\left\Vert F(x_{t-1})\right\Vert \right)\left\Vert x_{t}-z_{t}\right\Vert \\
 & \leq2G\left\Vert x_{t}-z_{t}\right\Vert 
\end{align*}
Using Cauchy-Schwartz and the triangle inequality, we obtain
\[
\left\langle \xi_{t-1}-\xi_{t},x_{t}-z_{t}\right\rangle \leq\left\Vert \xi_{t-1}-\xi_{t}\right\Vert \left\Vert x_{t}-z_{t}\right\Vert \leq\left(\left\Vert \xi_{t}\right\Vert +\left\Vert \xi_{t-1}\right\Vert \right)R
\]
Thus
\[
\left\langle \widehat{F(x_{t})}-\widehat{F(x_{t-1})},x_{t}-z_{t}\right\rangle \leq2G\left\Vert x_{t}-z_{t}\right\Vert +\left(\left\Vert \xi_{t}\right\Vert +\left\Vert \xi_{t-1}\right\Vert \right)R
\]
as needed.
\end{proof}

\begin{lem}
\label{lem:vector-regret-term2-C-smooth}Suppose $F$ is smooth with
respect to the $\left\Vert \cdot\right\Vert _{\sm}$ norm, where $\sm=\diag\left(\beta_{1},\dots,\beta_{d}\right)$
for scalars $\beta_{1},\dots,\beta_{d}\geq1$. We have
\begin{align*}
\left\langle \widehat{F(x_{t})}-\widehat{F(x_{t-1})},x_{t}-z_{t}\right\rangle  & \leq\left\Vert x_{t}-z_{t-1}\right\Vert _{\sm}^{2}+\left\Vert x_{t-1}-z_{t-1}\right\Vert _{\sm}^{2}+\frac{1}{2}\left\Vert x_{t}-z_{t}\right\Vert _{\sm}^{2}\\
 & +\left(\left\Vert \xi_{t}\right\Vert +\left\Vert \xi_{t-1}\right\Vert \right)R
\end{align*}
\end{lem}
\begin{proof}
We write
\[
\left\langle \widehat{F(x_{t})}-\widehat{F(x_{t-1})},x_{t}-z_{t}\right\rangle =\left\langle F(x_{t})-F(x_{t-1}),x_{t}-z_{t}\right\rangle +\left\langle \xi_{t-1}-\xi_{t},x_{t}-z_{t}\right\rangle 
\]
Using Holder's inequality and smoothness, we obtain
\begin{align*}
\left\langle F(x_{t})-F(x_{t-1}),x_{t}-z_{t}\right\rangle  & \leq\left\Vert F(x_{t})-F(x_{t-1})\right\Vert _{\sm^{-1}}\left\Vert x_{t}-z_{t}\right\Vert _{\sm}\\
 & \leq\left\Vert x_{t}-x_{t-1}\right\Vert _{\sm}\left\Vert x_{t}-z_{t}\right\Vert _{\sm}\\
 & \leq\frac{1}{2}\left\Vert x_{t}-x_{t-1}\right\Vert _{\sm}^{2}+\frac{1}{2}\left\Vert x_{t}-z_{t}\right\Vert _{\sm}^{2}\\
 & =\frac{1}{2}\left\Vert x_{t}-z_{t-1}+z_{t-1}-x_{t-1}\right\Vert _{\sm}^{2}+\frac{1}{2}\left\Vert x_{t}-z_{t}\right\Vert _{\sm}^{2}\\
 & \leq\left\Vert x_{t}-z_{t-1}\right\Vert _{\sm}^{2}+\left\Vert z_{t-1}-x_{t-1}\right\Vert _{\sm}^{2}+\frac{1}{2}\left\Vert x_{t}-z_{t}\right\Vert _{\sm}^{2}
\end{align*}
Using Cauchy-Schwartz and the triangle inequality, we obtain
\[
\left\langle \xi_{t-1}-\xi_{t},x_{t}-z_{t}\right\rangle \leq\left\Vert \xi_{t-1}-\xi_{t}\right\Vert \left\Vert x_{t}-z_{t}\right\Vert \leq\left(\left\Vert \xi_{t}\right\Vert +\left\Vert \xi_{t-1}\right\Vert \right)R
\]
Thus
\begin{align*}
\left\langle \widehat{F(x_{t})}-\widehat{F(x_{t-1})},x_{t}-z_{t}\right\rangle  & \leq\left\Vert x_{t}-z_{t-1}\right\Vert _{\sm}^{2}+\left\Vert z_{t-1}-x_{t-1}\right\Vert _{\sm}^{2}+\frac{1}{2}\left\Vert x_{t}-z_{t}\right\Vert _{\sm}^{2}\\
 & +\left(\left\Vert \xi_{t}\right\Vert +\left\Vert \xi_{t-1}\right\Vert \right)R
\end{align*}
as needed.
\end{proof}

Next, we put everything together and obtain the final convergence
guarantees.
\begin{lem}
Suppose $F$ is non-smooth and let $G=\max_{x\in\dom}\left\Vert F(x)\right\Vert $.
We have
\[
\E\left[\err(\avx_{T})\right]\leq O\left(\frac{\gamma_{0}dR_{\infty}^{2}}{T}+\frac{\sqrt{d}R_{\infty}G\left(\sqrt{\ln\left(\frac{GT}{R_{\infty}}\right)}+\sqrt{\ln\left(\gamma_{0}^{-1}\right)}\right)}{\sqrt{T}}+\frac{R\sigma}{\sqrt{T}}\right)
\]
\end{lem}
\begin{proof}
By combining (\ref{eq:vector-regret-split-C}) and Lemmas \ref{lem:vector-regret-term1-C},
\ref{lem:vector-regret-term3-C}, \ref{lem:vector-regret-term2-C-nonsmooth},
we obtain
\begin{align*}
 & \left\langle \widehat{F(x_{t})},x_{t}-y\right\rangle \\
 & \leq\frac{1}{2}\left\Vert z_{t-1}-y\right\Vert _{\D_{t-1}}^{2}-\frac{1}{2}\left\Vert z_{t}-y\right\Vert _{\D_{t-1}}^{2}-\frac{1}{2}\left\Vert x_{t}-z_{t-1}\right\Vert _{\D_{t-1}}^{2}-\frac{1}{2}\left\Vert x_{t}-z_{t}\right\Vert _{\D_{t-1}}^{2}\\
 & +2G\left\Vert x_{t}-z_{t}\right\Vert +\left(\left\Vert \xi_{t}\right\Vert +\left\Vert \xi_{t-1}\right\Vert \right)R\\
 & =\frac{1}{2}\left\Vert z_{t}-y\right\Vert _{\D_{t}-\D_{t-1}}^{2}+\frac{1}{2}\left\Vert z_{t-1}-y\right\Vert _{\D_{t-1}}^{2}-\frac{1}{2}\left\Vert z_{t}-y\right\Vert _{\D_{t}}^{2}\\
 & -\frac{1}{2}\left\Vert x_{t}-z_{t-1}\right\Vert _{\D_{t-1}}^{2}-\frac{1}{2}\left\Vert x_{t}-z_{t}\right\Vert _{\D_{t-1}}^{2}+2G\left\Vert x_{t}-z_{t}\right\Vert +\left(\left\Vert \xi_{t}\right\Vert +\left\Vert \xi_{t-1}\right\Vert \right)R
\end{align*}
Summing up over all iterations and using the inequality $\left\Vert x-y\right\Vert _{\D}^{2}\leq\tr(\D)R_{\infty}^{2}$
for $x,y\in\dom$, we obtain
\begin{align*}
 & \left\langle \widehat{F(x_{t})},x_{t}-y\right\rangle \\
 & \leq\frac{1}{2}R_{\infty}^{2}\left(\tr(\D_{T})-\tr(\D_{0})\right)+\frac{1}{2}\underbrace{\left\Vert z_{0}-y\right\Vert _{\D_{0}}^{2}}_{\leq R_{\infty}^{2}\tr(\D_{0})}-\frac{1}{2}\left\Vert z_{T}-y\right\Vert _{\D_{T}}^{2}\\
 & -\frac{1}{2}\sum_{t=1}^{T}\left(\left\Vert x_{t}-z_{t-1}\right\Vert _{\D_{t-1}}^{2}+\left\Vert x_{t}-z_{t}\right\Vert _{\D_{t-1}}^{2}\right)+2G\sum_{t=1}^{T}\left\Vert x_{t}-z_{t}\right\Vert +2R\sum_{t=0}^{T}\left\Vert \xi_{t}\right\Vert \\
 & \leq\frac{1}{2}R_{\infty}^{2}\tr(\D_{T})-\underbrace{\frac{1}{2}\sum_{t=1}^{T}\left(\left\Vert x_{t}-z_{t-1}\right\Vert _{\D_{t-1}}^{2}+\left\Vert x_{t}-z_{t}\right\Vert _{\D_{t-1}}^{2}\right)}_{(\star)}+\underbrace{2G\sum_{t=1}^{T}\left\Vert x_{t}-z_{t}\right\Vert }_{(\star\star)}+2R\sum_{t=0}^{T}\left\Vert \xi_{t}\right\Vert 
\end{align*}
We now upper bound each of the terms $(\star)$ and $(\star\star)$
in turn. The argument is analogous to that used in \citep{ene2020adaptive}.

For each coordinate separately, we apply Lemma \ref{lem:inequalities}
with $d_{t}^{2}=\left(x_{t,i}-z_{t-1,i}\right)^{2}+\left(x_{t,i}-z_{t,i}\right)^{2}$
and $R^{2}=2R_{\infty}^{2}\geq d_{t}^{2}$, and obtain
\begin{align*}
(\star) & =\frac{1}{2}\sum_{t=1}^{T}\left(\left\Vert x_{t}-z_{t-1}\right\Vert _{\D_{t-1}}^{2}+\left\Vert x_{t}-z_{t}\right\Vert _{\D_{t-1}}^{2}\right)\\
 & =\frac{1}{2}\sum_{i=1}^{d}\sum_{t=1}^{T}\D_{t-1,i}\left(\left(x_{t,i}-z_{t-1,i}\right)^{2}+\left(x_{t,i}-z_{t,i}\right)^{2}\right)\\
 & \geq2R_{\infty}^{2}\left(\tr(\D_{T})-\tr(\D_{0})\right)\\
 & =2R_{\infty}^{2}\left(\tr(\D_{T})-d\gamma_{0}\right)
\end{align*}
Since $\sqrt{z}$ is concave, we have
\begin{align*}
(\star\star) & =2G\sum_{t=1}^{T}\sqrt{\left\Vert x_{t}-z_{t}\right\Vert ^{2}}\\
 & \leq2G\sqrt{T}\sqrt{\sum_{t=1}^{T}\left\Vert x_{t}-z_{t}\right\Vert ^{2}}\\
 & \leq2G\sqrt{T}\sqrt{\sum_{t=1}^{T}\left(\left\Vert x_{t}-z_{t-1}\right\Vert ^{2}+\left\Vert x_{t}-z_{t}\right\Vert ^{2}\right)}
\end{align*}
For each coordinate separately, we apply Lemma \ref{lem:inequalities}
with $d_{t}^{2}=\left(x_{t,i}-z_{t-1,i}\right)^{2}+\left(x_{t,i}-z_{t,i}\right)^{2}$
and $R^{2}=2R_{\infty}^{2}\geq d_{t}^{2}$, and obtain
\begin{align*}
\sum_{t=1}^{T}\left(\left\Vert x_{t}-z_{t-1}\right\Vert ^{2}+\left\Vert x_{t}-z_{t}\right\Vert ^{2}\right) & \leq8R_{\infty}^{2}\sum_{i=1}^{d}\ln\left(\frac{\D_{T,i}}{\D_{0,i}}\right)\\
 & =8R_{\infty}^{2}\left(\sum_{i=1}^{d}\ln\left(\D_{T,i}\right)+d\ln\left(\gamma_{0}^{-1}\right)\right)
\end{align*}
Therefore
\begin{align*}
(\star\star) & \leq4\sqrt{2}R_{\infty}G\sqrt{T}\sqrt{\sum_{i=1}^{d}\ln\left(\D_{T,i}\right)+d\ln\left(\gamma_{0}^{-1}\right)}\\
 & \leq4\sqrt{2}R_{\infty}G\sqrt{T}\sqrt{\sum_{i=1}^{d}\ln\left(\D_{T,i}\right)}+4\sqrt{2}\sqrt{d}R_{\infty}G\sqrt{T}\sqrt{\ln\left(\gamma_{0}^{-1}\right)}
\end{align*}
Plugging into the previous inequality and using Lemma \ref{lem:phiz},
we obtain
\begin{align*}
 & \sum_{t=1}^{T}\left\langle \widehat{F(x_{t})},x_{t}-y\right\rangle \\
 & \leq4\sqrt{2}R_{\infty}G\sqrt{T}\sqrt{\sum_{i=1}^{d}\ln\left(\D_{T,i}\right)}-\frac{3}{2}R_{\infty}^{2}\sum_{i=1}^{d}\D_{T,i}\\
 & +2\gamma_{0}dR_{\infty}^{2}+4\sqrt{2}\sqrt{d}R_{\infty}G\sqrt{T}\sqrt{\ln\left(\gamma_{0}^{-1}\right)}+2R\sum_{t=0}^{T}\left\Vert \xi_{t}\right\Vert \\
 & \leq O\left(\sqrt{d}R_{\infty}G\sqrt{T}\sqrt{\ln\left(\frac{GT}{R_{\infty}}\right)}\right)+2\gamma_{0}dR_{\infty}^{2}+4\sqrt{2}\sqrt{d}R_{\infty}G\sqrt{T}\sqrt{\ln\left(\gamma_{0}^{-1}\right)}+2R\sum_{t=0}^{T}\left\Vert \xi_{t}\right\Vert \\
 & =O\left(\sqrt{d}R_{\infty}G\sqrt{T}\left(\sqrt{\ln\left(\frac{GT}{R_{\infty}}\right)}+\sqrt{\ln\left(\gamma_{0}^{-1}\right)}\right)\right)+O\left(\gamma_{0}dR_{\infty}^{2}\right)+2R\sum_{t=0}^{T}\left\Vert \xi_{t}\right\Vert 
\end{align*}
By combining the above inequality with (\ref{eq:vector-error-fn-ub-C}),
we obtain
\begin{align*}
T\cdot\err(\avx_{T}) & \leq O\left(\sqrt{d}R_{\infty}G\sqrt{T}\left(\sqrt{\ln\left(\frac{GT}{R_{\infty}}\right)}+\sqrt{\ln\left(\gamma_{0}^{-1}\right)}\right)\right)+O\left(\gamma_{0}dR_{\infty}^{2}\right)\\
 & +3R\left\Vert \sum_{t=1}^{T}\xi_{t}\right\Vert +\sum_{t=1}^{T}\left\langle \xi_{t},x_{t}-x_{0}\right\rangle 
\end{align*}
By taking expectation and using Lemmas \ref{lem:stoch-err1} and \ref{lem:stoch-err3},
we obtain
\[
\E\left[\err(\avx_{T})\right]\leq O\left(\frac{\gamma_{0}dR_{\infty}^{2}}{T}+\frac{\sqrt{d}R_{\infty}G\left(\sqrt{\ln\left(\frac{GT}{R_{\infty}}\right)}+\sqrt{\ln\left(\gamma_{0}^{-1}\right)}\right)}{\sqrt{T}}+\frac{R\sigma}{\sqrt{T}}\right)
\]
as needed. 
\end{proof}

\begin{lem}
Suppose $F$ is smooth with respect to the $\left\Vert \cdot\right\Vert _{\sm}$
norm, where $\sm=\diag\left(\beta_{1},\dots,\beta_{d}\right)$ for
scalars $\beta_{1},\dots,\beta_{d}\geq1$. We have
\[
\E\left[\err(\avx_{T})\right]\leq O\left(\frac{R_{\infty}^{2}\sum_{i=1}^{d}\beta_{i}\left(\ln\left(2\beta_{i}\right)+\ln\left(\gamma_{0}^{-1}\right)\right)+\gamma_{0}dR_{\infty}^{2}}{T}+\frac{R\sigma}{\sqrt{T}}\right)
\]
\end{lem}
\begin{proof}
By combining (\ref{eq:vector-regret-split-C}) and Lemmas \ref{lem:vector-regret-term1-C},
\ref{lem:vector-regret-term3-C}, \ref{lem:vector-regret-term2-C-smooth},
we obtain
\begin{align*}
 & \left\langle \widehat{F(x_{t})},x_{t}-y\right\rangle \\
 & \leq\frac{1}{2}\left\Vert z_{t-1}-y\right\Vert _{\D_{t-1}}^{2}-\frac{1}{2}\left\Vert z_{t}-y\right\Vert _{\D_{t-1}}^{2}-\frac{1}{2}\left\Vert x_{t}-z_{t-1}\right\Vert _{\D_{t-1}}^{2}-\frac{1}{2}\left\Vert x_{t}-z_{t}\right\Vert _{\D_{t-1}}^{2}\\
 & +\left\Vert x_{t}-z_{t-1}\right\Vert _{\sm}^{2}+\left\Vert x_{t-1}-z_{t-1}\right\Vert _{\sm}^{2}+\frac{1}{2}\left\Vert x_{t}-z_{t}\right\Vert _{\sm}^{2}+\left(\left\Vert \xi_{t}\right\Vert +\left\Vert \xi_{t-1}\right\Vert \right)R\\
 & =\frac{1}{2}\left\Vert z_{t}-y\right\Vert _{\D_{t}-\D_{t-1}}^{2}+\frac{1}{2}\left\Vert z_{t-1}-y\right\Vert _{\D_{t-1}}^{2}\\
 & -\frac{1}{2}\left\Vert z_{t}-y\right\Vert _{\D_{t}}^{2}-\frac{1}{2}\left\Vert x_{t}-z_{t-1}\right\Vert _{\D_{t-1}}^{2}-\frac{1}{2}\left\Vert x_{t}-z_{t}\right\Vert _{\D_{t-1}}^{2}\\
 & +\left\Vert x_{t}-z_{t-1}\right\Vert _{\sm}^{2}+\left\Vert x_{t-1}-z_{t-1}\right\Vert _{\sm}^{2}+\frac{1}{2}\left\Vert x_{t}-z_{t}\right\Vert _{\sm}^{2}+\left(\left\Vert \xi_{t}\right\Vert +\left\Vert \xi_{t-1}\right\Vert \right)R
\end{align*}
Summing up over all iterations and using the inequality$\left\Vert x-y\right\Vert _{\D}^{2}\leq\tr(\D)R_{\infty}^{2}$
for $x,y\in\dom$, we obtain
\begin{align*}
 & \sum_{t=1}^{T}\left\langle \widehat{F(x_{t})},x_{t}-y\right\rangle \\
 & \leq\frac{1}{2}R_{\infty}^{2}\left(\tr(\D_{T})-\tr(\D_{0})\right)+\frac{1}{2}\underbrace{\left\Vert z_{0}-y\right\Vert _{\D_{0}}^{2}}_{\leq R_{\infty}^{2}\tr(\D_{0})}-\frac{1}{2}\left\Vert z_{T}-y\right\Vert _{\D_{T}}^{2}\\
 & -\frac{1}{2}\sum_{t=1}^{T}\left(\left\Vert x_{t}-z_{t-1}\right\Vert _{\D_{t-1}}^{2}+\left\Vert x_{t}-z_{t}\right\Vert _{\D_{t-1}}^{2}\right)+\frac{3}{2}\sum_{t=1}^{T}\left(\left\Vert x_{t}-z_{t-1}\right\Vert _{\sm}^{2}+\left\Vert x_{t}-z_{t}\right\Vert _{\sm}^{2}\right)+2R\sum_{t=1}^{T}\left\Vert \xi_{t}\right\Vert \\
 & \leq\frac{1}{2}R_{\infty}^{2}\tr(\D_{T})-\frac{1}{2}\sum_{t=1}^{T}\left(\left\Vert x_{t}-z_{t-1}\right\Vert _{\D_{t-1}}^{2}+\left\Vert x_{t}-z_{t}\right\Vert _{\D_{t-1}}^{2}\right)\\
 & +\frac{3}{2}\sum_{t=1}^{T}\left(\left\Vert x_{t}-z_{t-1}\right\Vert _{\sm}^{2}+\left\Vert x_{t}-z_{t}\right\Vert _{\sm}^{2}\right)+2R\sum_{t=1}^{T}\left\Vert \xi_{t}\right\Vert \\
 & =\underbrace{\frac{1}{2}R_{\infty}^{2}\tr(\D_{T})-\frac{1}{8}\sum_{t=1}^{T}\left(\left\Vert x_{t}-z_{t-1}\right\Vert _{\D_{t-1}}^{2}+\left\Vert x_{t}-z_{t}\right\Vert _{\D_{t-1}}^{2}\right)}_{(\star)}\\
 & +\underbrace{\frac{3}{2}\sum_{t=1}^{T}\left(\left\Vert x_{t}-z_{t-1}\right\Vert _{\sm}^{2}+\left\Vert x_{t}-z_{t}\right\Vert _{\sm}^{2}\right)-\frac{3}{8}\sum_{t=1}^{T}\left(\left\Vert x_{t}-z_{t-1}\right\Vert _{\D_{t-1}}^{2}+\left\Vert x_{t}-z_{t}\right\Vert _{\D_{t-1}}^{2}\right)}_{(\star\star)}\\
 & +2R\sum_{t=1}^{T}\left\Vert \xi_{t}\right\Vert 
\end{align*}
The key to the rest of the analysis is to show that the negative gain
terms can be used to absorb most of the positive loss terms. To this
end, we have split the gain terms across $(\star)$ and $(\star\star)$.
We upper bound each of the terms $(\star)$ and $(\star\star)$ in
turn. The argument is analogous to that used in \citep{ene2020adaptive}.

For each coordinate separately, we apply Lemma \ref{lem:inequalities}
with $d_{t}^{2}=\left(x_{t,i}-z_{t-1,i}\right)^{2}+\left(x_{t,i}-z_{t,i}\right)^{2}$
and $R^{2}=2R_{\infty}^{2}\geq d_{t}^{2}$, and obtain
\begin{align*}
\sum_{t=1}^{T}\left(\left\Vert x_{t}-z_{t-1}\right\Vert _{\D_{t-1}}^{2}+\left\Vert x_{t}-z_{t}\right\Vert _{\D_{t-1}}^{2}\right) & =\sum_{i=1}^{d}\sum_{t=1}^{T}\D_{t-1,i}\left(\left(x_{t,i}-z_{t-1,i}\right)^{2}+\left(x_{t,i}-z_{t,i}\right)^{2}\right)\\
 & \geq4R_{\infty}^{2}\left(\tr(\D_{T})-\tr(\D_{0})\right)
\end{align*}
Therefore
\begin{align*}
(\star) & =\frac{1}{2}R_{\infty}^{2}\tr(\D_{T})-\frac{1}{8}\sum_{t=1}^{T}\left(\left\Vert x_{t}-z_{t-1}\right\Vert _{\D_{t}}^{2}+\left\Vert x_{t}-z_{t}\right\Vert _{\D_{t}}^{2}\right)\leq\frac{1}{2}R_{\infty}^{2}\tr(\D_{0})=\frac{1}{2}\gamma_{0}dR_{\infty}^{2}
\end{align*}
Next, we consider $(\star\star)$. Note that $\D_{t,i}$ is increasing
with $t$. Let $\tau_{i}$ be the last iteration $t$ for which $\D_{t-1,i}\leq4\beta_{i}$.
We have
\begin{align*}
 & (\star\star)\\
 & =\frac{3}{2}\sum_{t=1}^{T}\left(\left(\left\Vert x_{t}-z_{t-1}\right\Vert _{\sm}^{2}+\left\Vert x_{t}-z_{t}\right\Vert _{\sm}^{2}\right)-\frac{1}{4}\left(\left\Vert x_{t}-z_{t-1}\right\Vert _{\D_{t-1}}^{2}+\left\Vert x_{t}-z_{t}\right\Vert _{\D_{t-1}}^{2}\right)\right)\\
 & =\frac{3}{2}\sum_{i=1}^{d}\sum_{t=1}^{T}\left(\beta_{i}\left(\left(x_{t,i}-z_{t-1,i}\right)^{2}+\left(x_{t,i}-z_{t,i}\right)^{2}\right)-\frac{1}{4}\D_{t-1,i}\left(\left(x_{t,i}-z_{t-1,i}\right)^{2}+\left(x_{t,i}-z_{t,i}\right)^{2}\right)\right)\\
 & \leq\frac{3}{2}\sum_{i=1}^{d}\sum_{t=1}^{\tau_{i}}\beta_{i}\left(\left(x_{t,i}-z_{t-1,i}\right)^{2}+\left(x_{t,i}-z_{t,i}\right)^{2}\right)
\end{align*}
For each coordinate separately, we apply Lemma \ref{lem:inequalities}
with $d_{t}^{2}=\left(x_{t,i}-y_{t-1,i}\right)^{2}+\left(x_{t,i}-y_{t,i}\right)^{2}$
and $R^{2}=2R_{\infty}^{2}\geq d_{t}^{2}$, and obtain
\begin{align*}
\sum_{t=1}^{\tau_{i}}\left(\left(x_{t,i}-z_{t-1,i}\right)^{2}+\left(x_{t,i}-z_{t,i}\right)^{2}\right) & \leq8R_{\infty}^{2}\ln\left(\frac{\D_{\tau_{i},i}}{\D_{0,i}}\right)\\
 & =8R_{\infty}^{2}\left(\ln\left(\D_{\tau_{i},i}\right)+\ln\left(\gamma_{0}^{-1}\right)\right)\\
 & \leq8R_{\infty}^{2}\left(\ln\left(4\beta_{i}\right)+\ln\left(\gamma_{0}^{-1}\right)\right)
\end{align*}
Therefore
\[
(\star\star)\leq O\left(R_{\infty}^{2}\sum_{i=1}^{d}\beta_{i}\left(\ln\left(2\beta_{i}\right)+\ln\left(\gamma_{0}^{-1}\right)\right)\right)
\]
Plugging into the previous inequality, we obtain
\[
\sum_{t=1}^{T}\left\langle \widehat{F(x_{t})},x_{t}-y\right\rangle \leq O\left(R_{\infty}^{2}\sum_{i=1}^{d}\beta_{i}\left(\ln\left(2\beta_{i}\right)+\ln\left(\gamma_{0}^{-1}\right)\right)+\gamma_{0}dR_{\infty}^{2}\right)+2R\sum_{t=1}^{T}\left\Vert \xi_{t}\right\Vert 
\]
By combining the above inequality with (\ref{eq:vector-error-fn-ub-C}),
we obtain
\[
T\cdot\err(\avx_{T})\leq O\left(R_{\infty}^{2}\sum_{i=1}^{d}\beta_{i}\left(\ln\left(2\beta_{i}\right)+\ln\left(\gamma_{0}^{-1}\right)\right)+\gamma_{0}dR_{\infty}^{2}\right)+3R\left\Vert \sum_{t=1}^{T}\xi_{t}\right\Vert +\sum_{t=1}^{T}\left\langle \xi_{t},x_{t}-x_{0}\right\rangle 
\]
By taking expectation and using Lemmas \ref{lem:stoch-err1} and \ref{lem:stoch-err3},
we obtain
\[
\E\left[\err(\avx_{T})\right]\leq O\left(\frac{R_{\infty}^{2}\sum_{i=1}^{d}\beta_{i}\left(\ln\left(2\beta_{i}\right)+\ln\left(\gamma_{0}^{-1}\right)\right)+\gamma_{0}dR_{\infty}^{2}}{T}+\frac{R\sigma}{\sqrt{T}}\right)
\]
as needed.
\end{proof}

\section{Additional experimental results}

\label{sec:experiments-extra}

\subsection{Additional bilinear experiments}

\begin{figure}
\label{fig:bilinear-2call-vs-1call}

\includegraphics[width=0.49\linewidth]{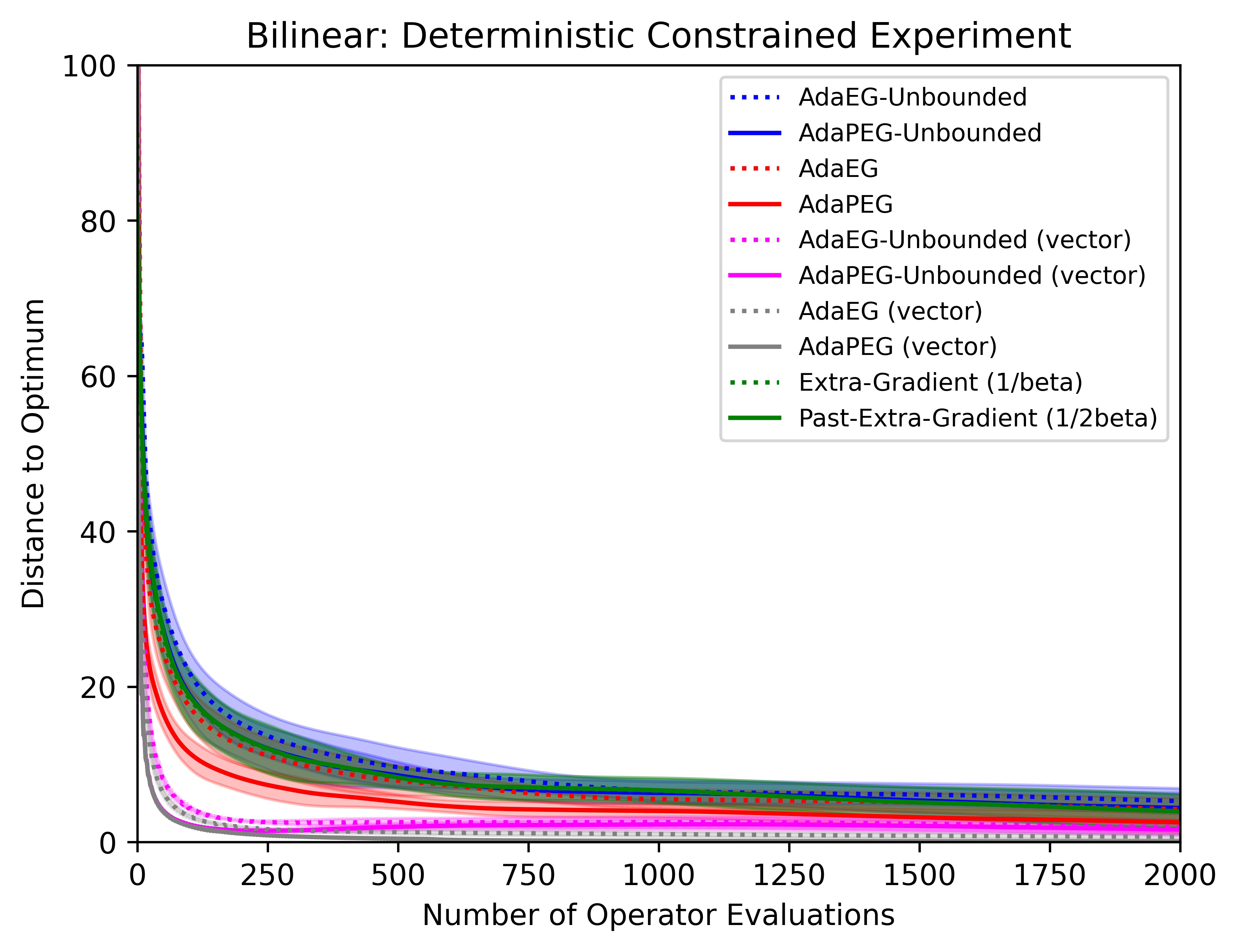}\includegraphics[width=0.49\linewidth]{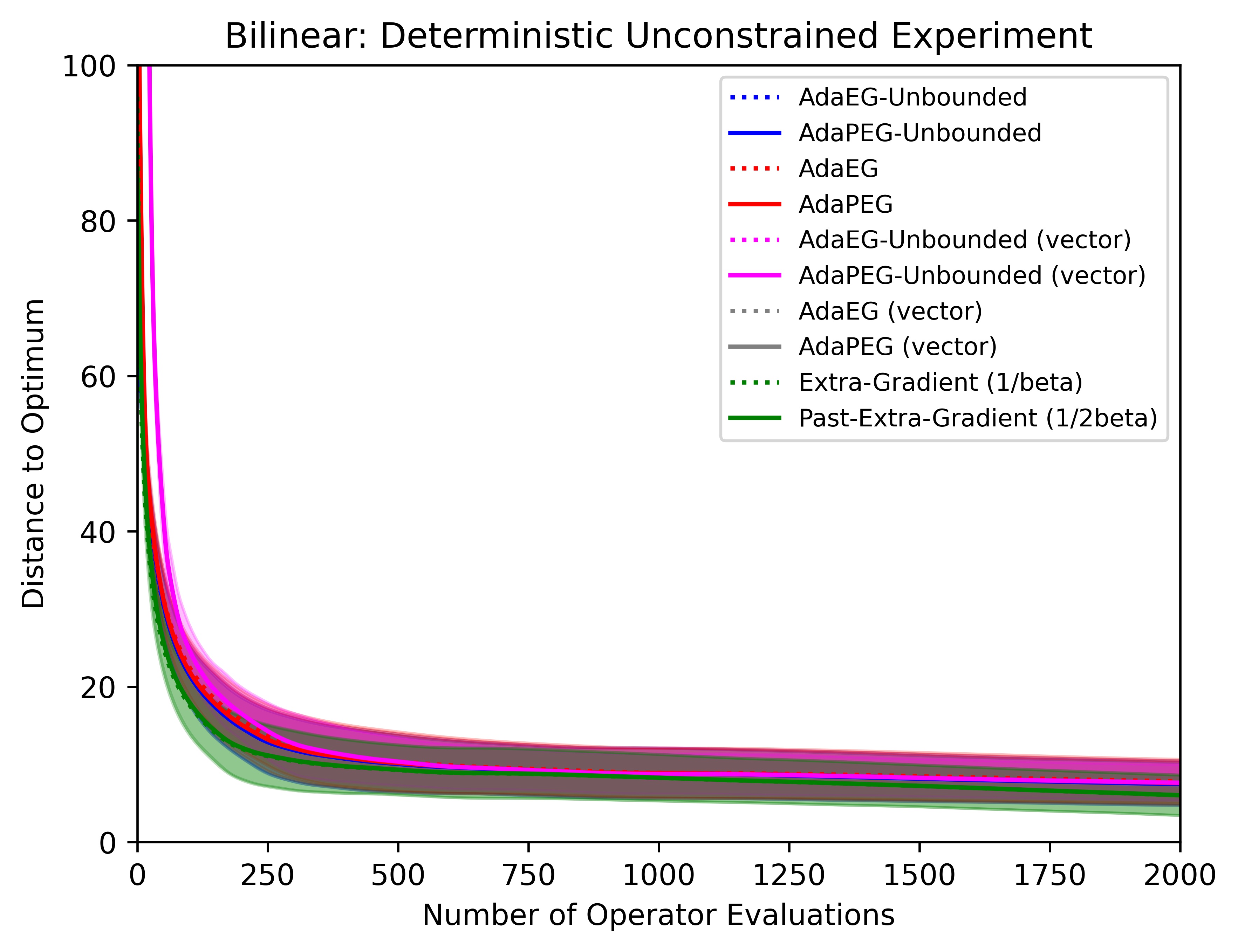}

\includegraphics[width=0.49\linewidth]{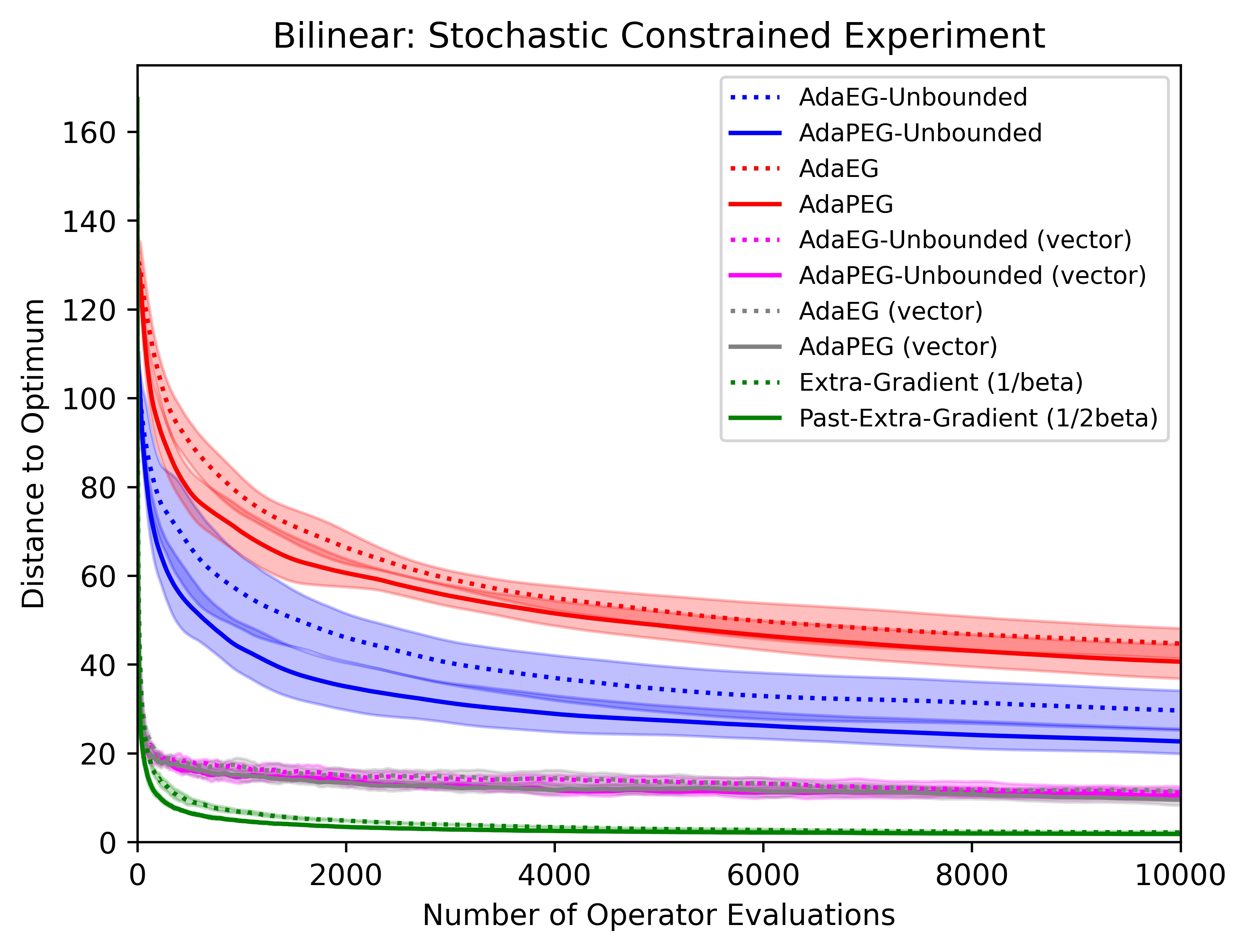}\includegraphics[width=0.49\linewidth]{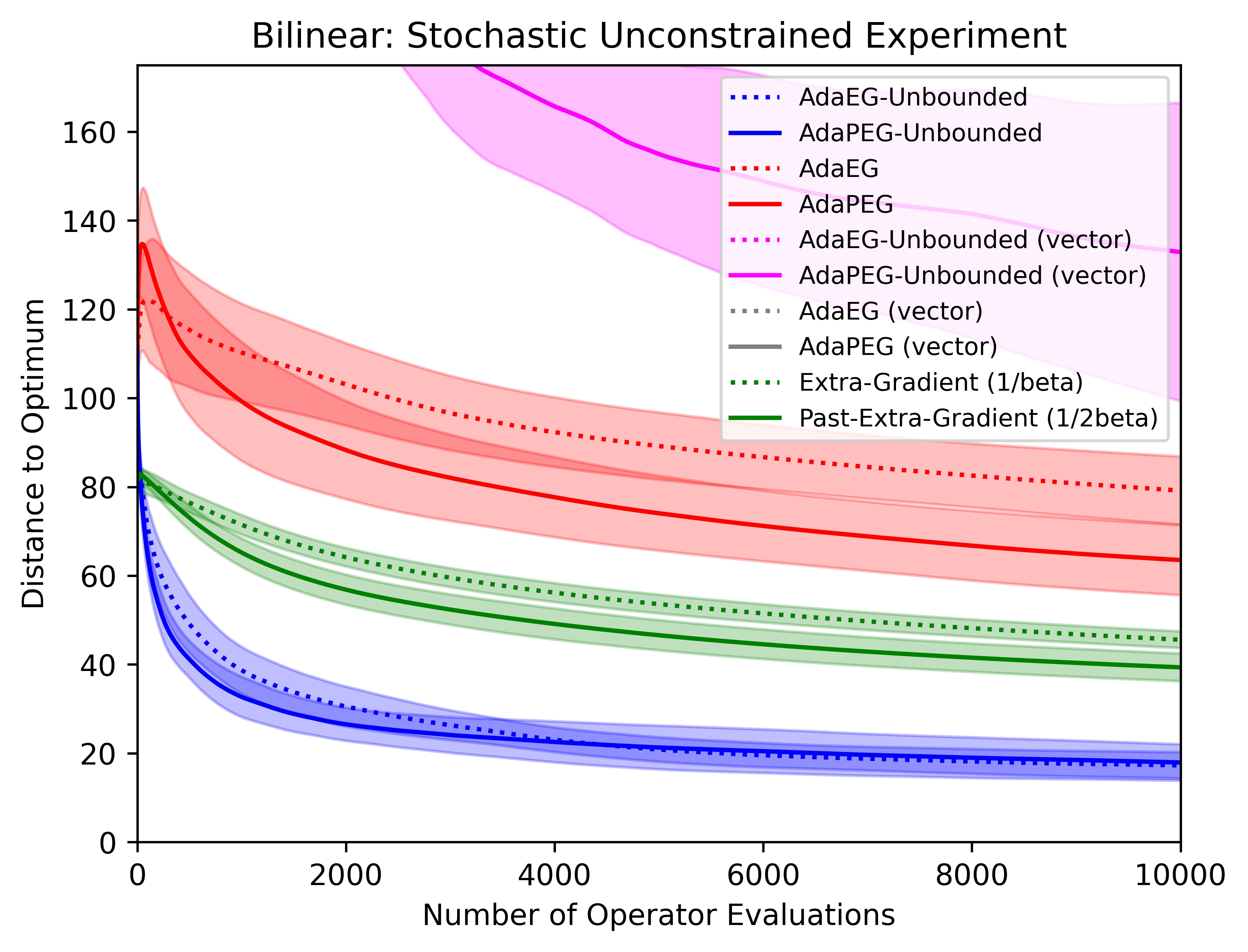}

\caption{Comparison of the 1-call and 2-call algorithms on bilinear instances.
We report the mean and standard deviation over $5$ runs.}
\end{figure}
\begin{figure}
\label{fig:bilinear-adapeg-variants}

\includegraphics[width=0.49\linewidth]{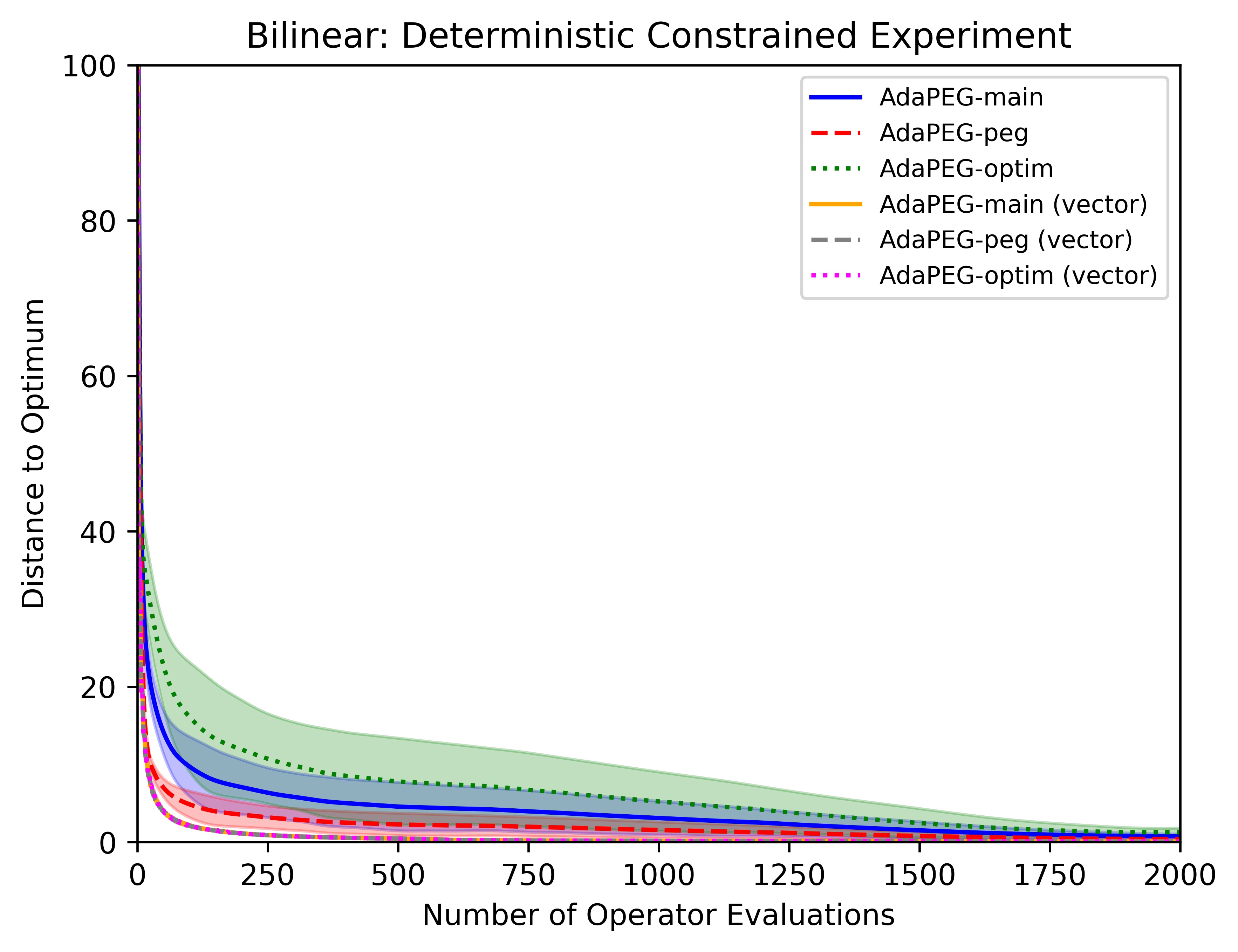}\includegraphics[width=0.49\linewidth]{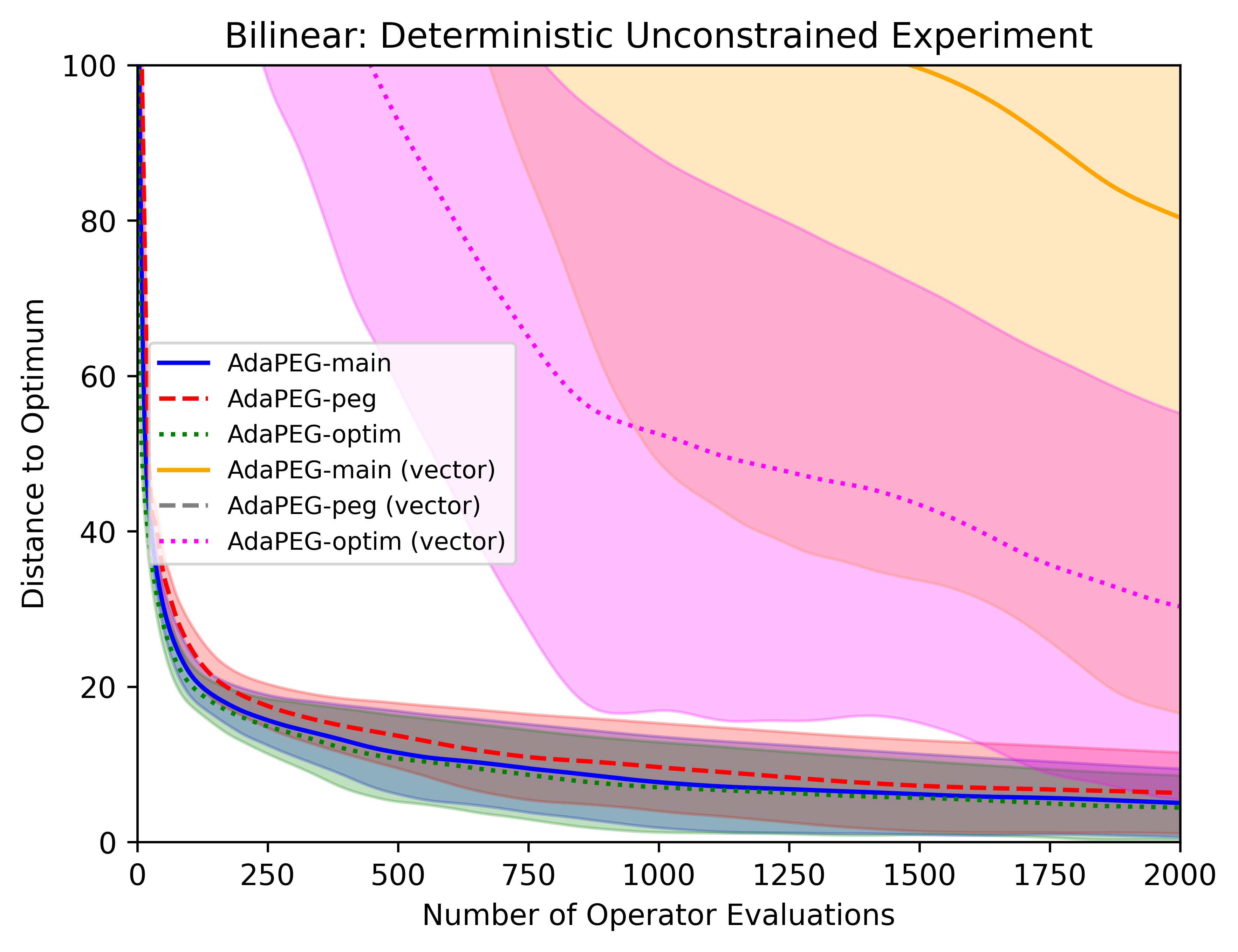}

\includegraphics[width=0.49\linewidth]{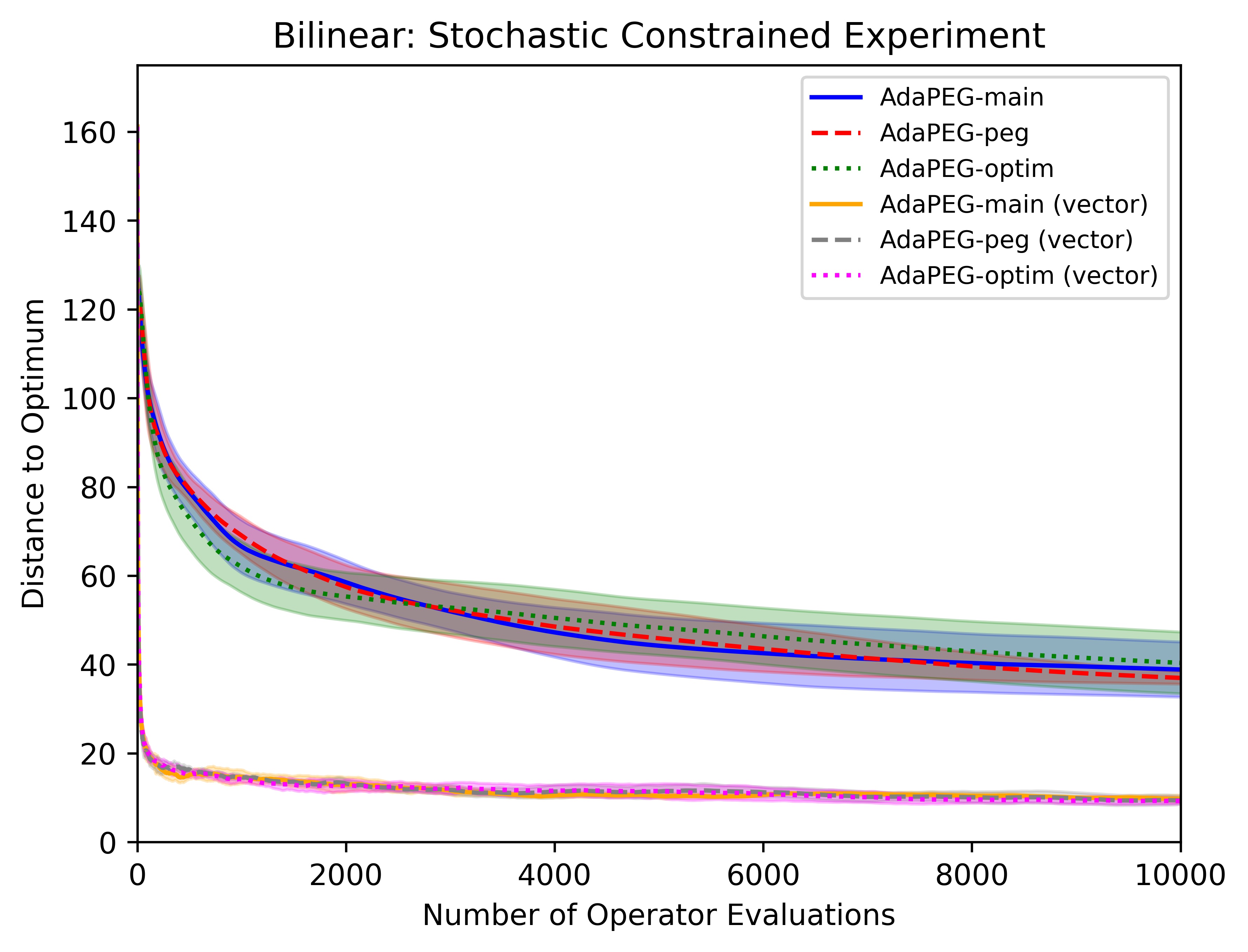}\includegraphics[width=0.49\linewidth]{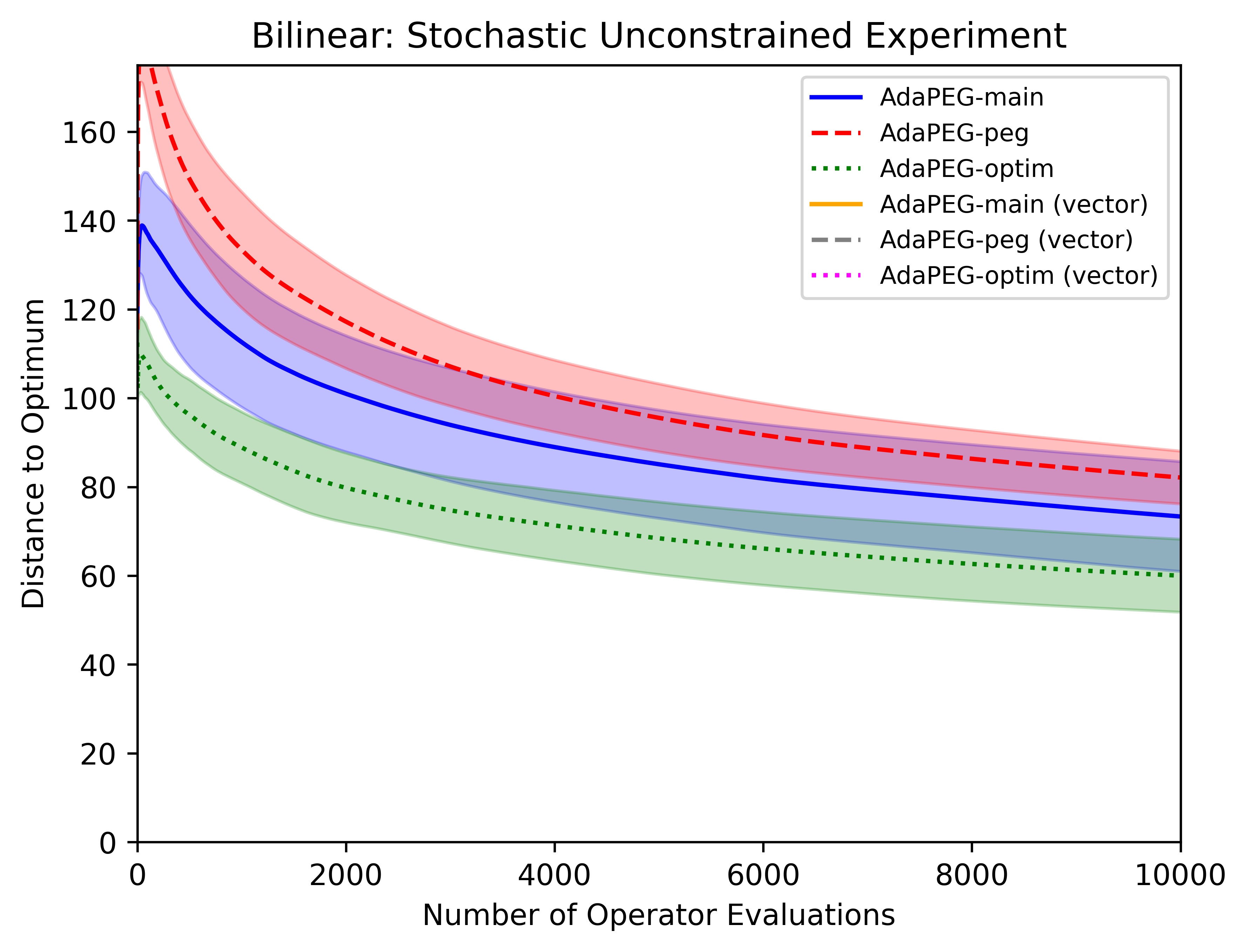}

\caption{Comparison of the AdaPEG variants on bilinear instances. We report
the mean and standard deviation over $5$ runs.}
\end{figure}

In this section, we give additional experimental results on bilinear
instances. The experimental setup is the same as in Section \ref{sec:experiments}.

\textbf{Comparison of 1-call and 2-call methods: }We experimentally
compared the 1-call methods based on Past Extra-Gradient with the
2-call methods based on Extra-Gradient on bilinear instances. The
results are shown in Figure \ref{fig:bilinear-2call-vs-1call}. We
observe that the 1-call methods perform as well or better than the
2-call methods on all of the instances.

\textbf{Comparison of variants of Algorithm \ref{alg:adapeg}: }We
experimentally compared variants of Algorithm \ref{alg:adapeg} that
do not include the additional term $\left\Vert u-x_{t}\right\Vert ^{2}$
in the update for $z_{t}$. The results are shown in Figure \ref{fig:bilinear-adapeg-variants}.
There are two versions that one can consider, depending on whether
we update the step size before updating $z_{t}$ or after. The updates
performed by the original algorithm and these two variants are as
follows:

AdaPEG-main (Algorithm \ref{alg:adapeg}) update:
\begin{align*}
x_{t} & =\arg\min_{u\in\dom}\left\{ \left\langle \widehat{F(x_{t-1})},u\right\rangle +\frac{1}{2}\gamma_{t-1}\left\Vert u-z_{t-1}\right\Vert ^{2}\right\} \\
z_{t} & =\arg\min_{u\in\dom}\left\{ \left\langle \widehat{F(x_{t})},u\right\rangle +\frac{1}{2}\gamma_{t-1}\left\Vert u-z_{t-1}\right\Vert ^{2}+\frac{1}{2}\left(\gamma_{t}-\gamma_{t-1}\right)\left\Vert u-x_{t}\right\Vert ^{2}\right\} 
\end{align*}
AdaPEG-peg update:
\begin{align*}
x_{t} & =\arg\min_{u\in\dom}\left\{ \left\langle \widehat{F(x_{t-1})},u\right\rangle +\frac{1}{2}\gamma_{t-1}\left\Vert u-z_{t-1}\right\Vert ^{2}\right\} \\
z_{t} & =\arg\min_{u\in\dom}\left\{ \left\langle \widehat{F(x_{t})},u\right\rangle +\frac{1}{2}\gamma_{t-1}\left\Vert u-z_{t-1}\right\Vert ^{2}\right\} 
\end{align*}
AdaPEG-optim update:
\begin{align*}
x_{t} & =\arg\min_{u\in\dom}\left\{ \left\langle \widehat{F(x_{t-1})},u\right\rangle +\frac{1}{2}\gamma_{t-1}\left\Vert u-z_{t-1}\right\Vert ^{2}\right\} \\
z_{t} & =\arg\min_{u\in\dom}\left\{ \left\langle \widehat{F(x_{t})},u\right\rangle +\frac{1}{2}\gamma_{t}\left\Vert u-z_{t-1}\right\Vert ^{2}\right\} 
\end{align*}
In all cases, the step sizes are updated as in Algorithm \ref{alg:adapeg}
(for the scalar algorithms), and Algorithm \ref{alg:adapeg-vector}
(for the vector algorithms).

AdaPEG-peg follows the Past Extra-Gradient scheme, and AdaPEG-optim
coincides in the unconstrained setting with the Optimistic Mirror
Descent scheme considered in \citep{DISZ18}. In the unconstrained
setting, the AdaPEG-optim update is equivalent to:
\[
x_{t+1}=x_{t}-\frac{2}{\gamma_{t}}\widehat{F(x_{t})}+\frac{1}{\gamma_{t-1}}\widehat{F(x_{t-1})}
\]

\subsection{WGAN experiment}

\begin{algorithm}

\caption{AdaPEG-Adam algorithm. The algorithm is closely related to the Optimistic-Adam
algorithm of \citet{DISZ18}, with the main difference being that
the step sizes are updated using the difference between the current
and previous gradients instead of the current gradient.}

\label{alg:adapeg-adam}

\textbf{Input:} initial parameters $\theta_{1}$, step size $\eta>0$,
exponential decay rates for moment estimates $\beta_{1},\beta_{2}\in[0,1)$,
$\epsilon>0$.

\textbf{Initialize:} $g_{0}=m_{0}=v_{0}=0$

For $t=1,\dots,T$, update:

$\quad$Compute the stochastic gradient: $g_{t}=\nabla_{\theta}\ell_{t}(\theta_{t})$

$\quad$Update biased estimate of first moment: $m_{t}=\beta_{1}m_{t-1}+\left(1-\beta_{1}\right)g_{t}$

$\quad$Update biased estimate of second moment: $v_{t}=\beta_{2}v_{t-1}+\left(1-\beta_{2}\right)\left(g_{t}-g_{t-1}\right)^{2}$

$\quad$Compute bias corrected first moment: $\widehat{m}_{t}=m_{t}/\left(1-\beta_{1}^{t}\right)$

$\quad$Compute bias corrected second moment: $\widehat{v}_{t}=v_{t}/\left(1-\beta_{2}^{t}\right)$

$\quad$Update the parameters: $\theta_{t+1}=\theta_{t}-2\eta\frac{\widehat{m}_{t}}{\sqrt{\widehat{v}_{t}}+\epsilon}+\eta\frac{\widehat{m}_{t-1}}{\sqrt{\widehat{v}_{t-1}}+\epsilon}$
\end{algorithm}

\begin{figure}
\centering

\includegraphics[width=0.5\linewidth]{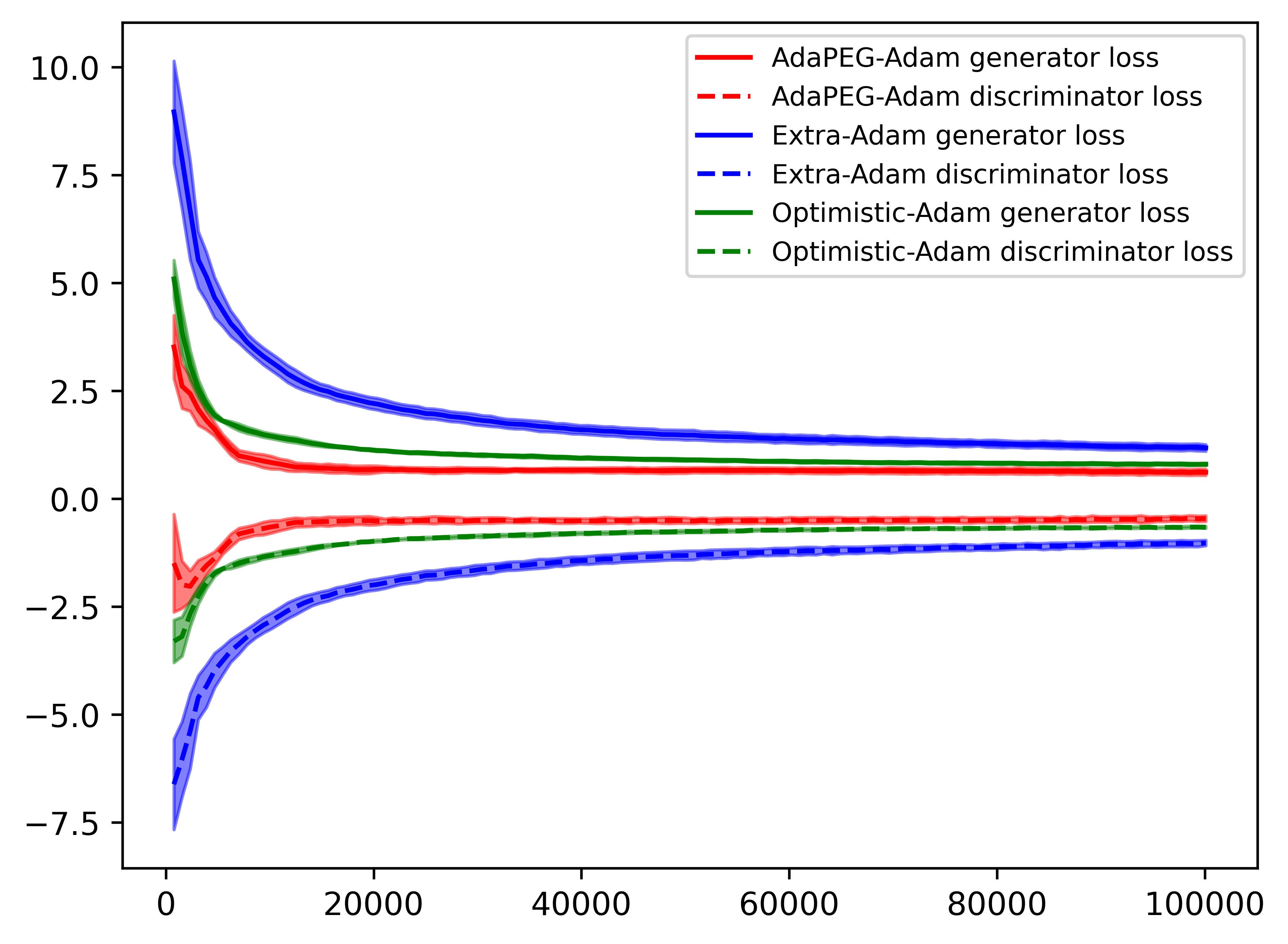}

\caption{WGAN experiment: training losses. We report the mean and standard
deviation over $5$ runs.}

\label{fig:wgan-losses}
\end{figure}

\begin{table}
\begin{centering}
\begin{tabular}{|c|c|c|}
\hline 
\textbf{Extra-Adam} & \textbf{Optimistic-Adam} & \textbf{AdaPEG-Adam}\tabularnewline
\hline 
\hline 
$33.65\pm2.05$ & $\mathbf{29.35\pm3.28}$ & $32.08\pm1.95$\tabularnewline
\hline 
\end{tabular}
\par\end{centering}
\caption{WGAN experiment: FID scores. We report the mean and standard deviation
over $5$ runs.}
\label{tb:wgan-fid}
\end{table}

\begin{figure}

\caption{Sample images generated by models trained with each of the algorithms
and the same random seed. Left: Extra-Adam. Middle: Optimistic-Adam.
Right: AdaPEG-Adam.}

\label{fig:wgan-samples}
\centering{}
\includegraphics[width=0.32\linewidth]{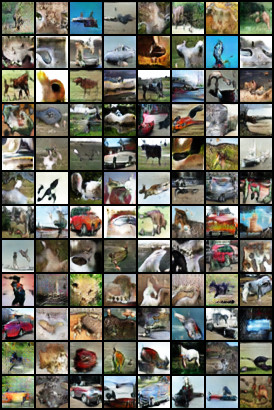}
\includegraphics[width=0.32\linewidth]{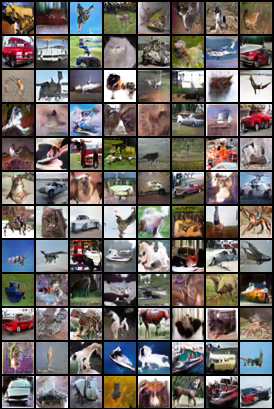}
\includegraphics[width=0.32\linewidth]{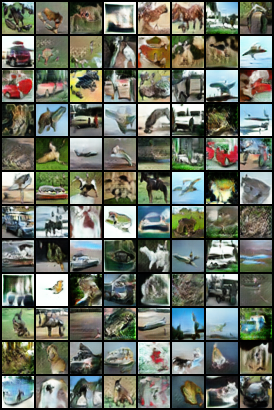}
\end{figure}

Following \citet{GidelBVVL19}, we evaluated a variant of our algorithm
for training generative adversarial networks (GANs). We followed the
experimental setup of \citet{GidelBVVL19}, and we built on their
implementation.\footnote{The code of \citet{GidelBVVL19} is available here: \url{https://github.com/GauthierGidel/Variational-Inequality-GAN}.}
We trained a ResNet model on the CIFAR10 dataset with the WGAN-GP
objective \citep{gulrajani2017improved}. We used the same ResNet
model and WGAN-GP objective as \citet{GidelBVVL19}. We evaluated
the models using the Frechet inception distance (FID) \citep{heusel2017gans}.
We computed the FID scores using 50000 samples using the code provided
by \citet{heusel2017gans}.\footnote{The code of \citet{heusel2017gans} is available here: \url{https://github.com/bioinf-jku/TTUR}.}

\textbf{Algorithms: }We experimentally compared the Extra-Adam algorithm
\citep{GidelBVVL19}, the Optimistic-Adam algorithm \citep{DISZ18},
and an Adam version of our algorithm shown in Algorithm \ref{alg:adapeg-adam},
which we refer to as AdaPEG-Adam. The AdaPEG-Adam algorithm is based
on the AdaPEG-optim variant described above, and we chose this variant
since it performed slightly better in the stochastic unconstrained
setting in the bilinear experiments, and this setting most closely
matches the WGAN-GP training. The AdaPEG-Adam algorithm is closely
related to the Optimistic-Adam algorithm of \citet{DISZ18}, with
the main difference being that the step sizes are updated using the
difference between the current and previous gradients instead of the
current gradient.

\textbf{Hyperparameters: }Due to computational constraints, we were
unable to perform an extensive hyperparameter search and instead we
used the hyperparameters suggested by \citet{GidelBVVL19} with the
following modifications. For all of the algorithms we used base learning
rate $5\cdot10^{-5}$ for the generator and $5\cdot10^{-4}$ for the
discriminator. These are the learning rates that \citet{GidelBVVL19}
used for Extra-Adam. For Optimistic-Adam, they used learning rate
$2\cdot10^{-5}$ for the generator and $2\cdot10^{-4}$ for the discriminator.
We tried both choices for the learning rates --- $(5\cdot10^{-5},5\cdot10^{-4})$
and $(2\cdot10^{-5},2\cdot10^{-4})$ --- for each of the three algorithms
and the first choice led to the best FID scores for all of the algorithms.
We performed $100,000$ stochastic operator evaluations for the generator.
This number of evaluations corresponds to $100,000$ generator updates
for Optimistic-Adam and AdaPEG-Adam, and $50,000$ generator updates
for Extra-Adam. 

\textbf{Results: }We report the FID scores in\textbf{ }Table \ref{tb:wgan-fid},
the WGAN-GP objective training losses in Figure \ref{fig:wgan-losses},
and sample images generated by the models in Figure \ref{fig:wgan-samples}.
We report the mean and standard deviation over $5$ runs with different
seeds. The FID scores were computed using the final model weights
constructed by each algorithm using $100,000$ stochastic operator
evaluations for the generator. We note that the methods have similar
FID scores within the standard deviation.

\textbf{Computing resources:} The experiments were performed on a
workstation with an Intel Core i7-9700k processor, 16GB of memory,
and a GPU card RTX 2080. Training a single model and a single FID
computation took approximately 3 hours each. We performed 5 runs,
and we trained 3 models in each run and computed the corresponding
FIDs.

\end{document}